%% file: Arxiv.tex
\documentclass{article} 
\usepackage{iclr2021_arxiv,times}

\input{math_commands.tex}

\newcommand{\states}{\mathcal{S}}
\newcommand{\state}{s}
\newcommand{\actions}{\mathcal{A}}
\newcommand{\action}{a}
\newcommand{\rewards}{R}
\newcommand{\reward}{r}

\newcommand{\transitions}{P}

\newcommand{\expectation}{\mathbb{E}}

\newcommand{\simudelay}{\tau_{\mathrm{sim}}}
\newcommand{\minsimudelay}{\underline{\tau}_{\mathrm{sim}}}
\newcommand{\maxsimudelay}{\overline{\tau}_{\mathrm{sim}}}

\newcommand{\cumulateregret}[1]{\mathrm{Regret}(#1)}
\newcommand{\syncrate}{\tau_{\mathrm{syn}}}

\newcommand{\indicatorfunc}[1]{\mathbbm{1} \left [ #1 \right ]}

\usepackage{url}
\usepackage{microtype}
\usepackage{graphicx}
\usepackage{subfigure}
\usepackage{booktabs} 
\usepackage{array}
\usepackage{hyperref}
\usepackage{bbm}
\usepackage{amsmath}
\usepackage{amsthm}
\usepackage{amssymb}
\usepackage[utf8]{inputenc}
\usepackage[english]{babel}
\usepackage{multicol}
\usepackage{natbib}
\usepackage{tikz}
\usepackage{algorithm}
\usepackage{algorithmic}
\usepackage{diagbox}
\usepackage{multirow}
\usepackage{bm}
\usepackage{footnote}
\usepackage{wrapfig}
\usepackage{makecell}
\usepackage{color}
\allowdisplaybreaks[4]

\newcommand\numberthis{\addtocounter{equation}{1}\tag{\theequation}}

\newtheorem{theorem}{Theorem}

\newtheorem{definition}{Definition}

\makeatletter
\newcommand{\eqnum}{\leavevmode\hfill\refstepcounter{equation}\textup{\tagform@{\theequation}}}
\makeatother

\title{On Effective Parallelization of \\ Monte Carlo Tree Search}

\author{Anji Liu$^\dagger$, Yitao Liang$^\dagger$, Ji Liu$^\ddagger$, Guy Van den Broeck$^\dagger$ \& Jianshu Chen$^\mathsection$ \\
$^\dagger$Department of Computer Science, University of California, Los Angeles \\
\texttt{\{liuanji, yliang, guyvdb\}@cs.ucla.edu} \\
$^\ddagger$Seattle AI Lab, Kwai Inc., Bellevue, WA 98004, USA \\
\texttt{jiliu@kuaishou.com} \\
$^\mathsection$Tencent AI Lab, Bellevue, WA 98004, USA \\
\texttt{jianshuchen@tencent.com}
}

\iclrfinalcopy 
\begin{document}

\maketitle

\begin{abstract}
Despite its groundbreaking success in Go and computer games, Monte Carlo Tree Search (MCTS) is computationally expensive as it requires a substantial number of rollouts to construct the search tree, which calls for effective parallelization. However, \emph{how to design effective parallel MCTS algorithms} has not been systematically studied and remains poorly understood. In this paper, we seek to lay its first theoretical foundation, by examining the potential performance loss caused by parallelization when achieving a desired speedup. In particular, we discover the necessary conditions of achieving a desirable parallelization performance, and highlight two of their practical benefits. First, by examining whether existing parallel MCTS algorithms satisfy these conditions, we identify key design principles that should be inherited by future algorithms, for example tracking the unobserved samples (used in WU-UCT \citep{liu2020watch}). We theoretically establish this essential design facilitates $\mathcal{O} ( \ln n + M / \sqrt{\ln n} )$ cumulative regret when the maximum tree depth is 2, where $n$ is the number of rollouts and $M$ is the number of workers. A regret of this form is highly desirable, as compared to $\mathcal{O} ( \ln n )$ regret incurred by a sequential counterpart, its excess part approaches zero as $n$ increases. Second, and more importantly, we demonstrate how the proposed necessary conditions can be adopted to design more effective parallel MCTS algorithms. To illustrate this, we propose a new parallel MCTS algorithm, called \emph{BU-UCT}, by following our theoretical guidelines. The newly proposed algorithm, albeit preliminary, out-performs four competitive baselines on 11 out of 15 Atari games. We hope our theoretical results could inspire future work of more effective parallel MCTS.
\end{abstract}

\section{Introduction}
\label{Introduction}

Monte Carlo Tree Search (MCTS) \citep{browne2012survey} algorithms have achieved unprecedented success in fields such as computer Go \citep{silver2016mastering}, card games \citep{powley2011determinization}, and video games \citep{schrittwieser2019mastering}. However, they generally require a large number of Monte Carlo rollouts to construct search trees, making themselves time-consuming. For this reason, parallel MCTS is highly appealing and has been successfully used in solving challenging tasks such as Go \citep{silver2017mastering,couetoux2017monte} and mobile games \citep{poromaa2017crushing,devlin2016combining}.

Despite their extensive usage, the performance of parallel MCTS algorithms \citep{chaslot2008parallel} is not systematically understood from a theoretical perspective. There are empirical studies on the advantages (e.g., \citet{yoshizoe2011scalable,gelly2006exploration}) and disadvantages (e.g., \citet{mirsoleimani2017analysis,soejima2010evaluating,bourki2010scalability}) of existing approaches. However, they are mainly algorithm-specific analysis, which provides less \emph{systematic} design principles on effective MCTS parallelization. As a consequence, practitioners still largely rely on the trial-and-error approach when designing a new parallel MCTS algorithm, which is time-wise costly.

In this paper, we seek to lay the first theoretical foundation for effective MCTS parallelization. Parallel MCTS algorithms generally exhibit different levels of performance loss compared to their sequential counterparts, especially when a large number of workers are employed to achieve high speedups \citep{segal2010scalability}. It is highly desirable for algorithm designers to minimize this loss while still achieving high speedup, especially in solving challenging large-scale tasks. Therefore, we focus on examining the potential performance loss caused by the parallelization when achieving a desired speedup. And we measure the performance loss by \emph{excess regret}, which is the extra cumulative regret of a parallel MCTS algorithm relative to its sequential counterpart. In particular, we will characterize the excess regret from a theoretical perspective and seek to answer the following key question: \emph{under what conditions would the excess regret vanish when the number of rollouts increases?}

To this end, with the help of a unified algorithm framework that covers all major existing parallel MCTS algorithms as its special cases, we derive two necessary conditions for any algorithm specified by the framework to achieve vanishing excess regret when the number of rollouts increases (Thm.~\ref{theorem: necessary condition for a good algorithm in the tree search setting}). We then highlight two practical benefits of the necessary conditions. First, the conditions allow us to identify \emph{key design wisdom proposed by existing algorithms}, for example tracking the unobserved samples, which is proposed in WU-UCT \citep{liu2020watch}. Second, and more importantly, we show that the necessary conditions can provide concrete guidelines for designing better (future) algorithms, which is demonstrated through an example workflow of algorithm design based on the necessary conditions. The resulting algorithm, \textbf{B}alance the \textbf{U}nobserved in UCT (BU-UCT), out-performs four competitive baselines on 11 out of 15 Atari games. We hope this encouraging result could inspire more future work to develop better parallel MCTS algorithms with our theory. 

\vspace{-0.5em}

\section{Preliminary: MCTS and its Parallelization}
\label{MCTS and Its Parallelization}

\vspace{-0.5em}

Consider a \emph{Markov Decision Process} (MDP) $\langle \states, \actions, \rewards, \transitions, \gamma \rangle$, where $\states$ denotes a \emph{finite} state space, $\actions$ is a \emph{finite} action space, $\rewards$ is a \emph{bounded} reward function, $\transitions$ defines a \emph{deterministic} state transition function, and $\gamma \in (0, 1]$ is the discount factor. At each time step $t$, the agent takes an action $\action_{t}$ when the environment is in a state $\state_{t}$, causing it to transit to the next state $\state_{t + 1}$ and emit a reward $\reward_{t}$. In the context of MCTS, $\transitions$ and $\rewards$ (or their approximations) are assumed to be known to the agent. By exploiting such knowledge, MCTS seeks to \emph{plan} the best action $\action$ at a given state $\state$ to achieve the highest expected cumulative reward $\expectation [ \sum_{t = 0}^{\infty} \gamma^{t} r_{t} \mid \state_{0} \!=\! \state]$. To this end, it constructs a search tree using a sequence of repeated Monte Carlo \emph{rollouts}, where a node corresponds to a state, and an edge from $\state_{t}$ to $\state_{t + 1}$ represents the action $\action_{t}$ that causes the transition from $\state_{t}$ to $\state_{t + 1}$. Each edge $(\state, \action)$ in the search tree also stores a set of statistics $\{ Q (\state, \action), N (\state, \action) \}$, where $Q (\state, \action)$ is the mean action value and $N (\state, \action)$ is the count of completed simulations. These statistics guide the construction of the search tree and are updated during the process. Specifically, during the \emph{selection} phase, the algorithm traverses over the current search tree by using a \emph{tree policy} (e.g., the Upper Confidence Bound (UCB) \cite{auer2002using}) to iteratively select an action $\action_{t}$ that leads to a child node $\state_{t + 1}$:
    { \setlength{\abovedisplayskip}{0.2em}
      \setlength{\belowdisplayskip}{-0.8em}
    \begin{align*}
        \action_t = \argmax_{\action \in \actions} \left \{ Q (\state_t, \action) \!+\! c \sqrt{\frac{2 \ln  \sum_{\action'} N (\state_t, \action') }{N (\state_t, \action)}} \right \}, \numberthis \label{eq: UCT tree policy}
    \end{align*}}
    
where the first term estimates the utility of executing $\action$ at $\state_{t}$, the second term represents the uncertainty of that estimate, and the hyperparameter $c$ controls the tradeoff between exploitation (term 1) and exploration (term 2). The selection process is performed iteratively until arriving at a node $\state_{T - 1}$ where some of its actions are not expanded. Then, the algorithm selects an unexpanded action $\action_{T - 1}$ at $\state_{T - 1}$ and adds a new leaf node $\state_{T}$ (corresponds to the next state) to the search tree at the \emph{expansion} phase, followed by querying its value $V (\state_{T})$ through \emph{simulation}, where a \emph{default policy} repeatedly interacts with the MDP starting from $\state_{T}$. Finally, in \emph{backpropagation}, the statistics along the selected path are recursively updated from $\state_{T - 1}$ to $\state_{0}$ (i.e., from $t = T - 1$ to $t = 0$) by
    { \setlength{\abovedisplayskip}{0.2em}
      \setlength{\belowdisplayskip}{-0.8em}
    \begin{gather*}
        N (\state_{t}, \action_{t}) \leftarrow N (\state_{t}, \action_{t}) + 1, \quad V (\state_{t}) = \rewards (\state_{t}, \action_{t}) + \gamma V (\state_{t + 1}), \numberthis \label{eq: recursive N and V hat update} \\
        Q (\state_{t}, \action_{t}) \leftarrow \frac{N (\state_{t}, \action_{t}) - 1}{N (\state_{t}, \action_{t})} Q (\state_{t}, \action_{t}) + \frac{V (\state_{t + 1})}{N (\state_{t}, \action_{t})}, \numberthis \label{eq: recursive Q update}
    \end{gather*}}
    
\noindent where the recursion starts from the simulation return value $V (\state_{T})$.

\noindent \emph{Parallel MCTS algorithms} seek to speedup their sequential counterparts by distributing workloads stemmed from the simulation steps to multiple workers, aiming to achieve the same performance with less computation time. Fig.~\ref{fig: typical parallel MCTS algs} presents five typical parallel MCTS algorithms. Among them, Leaf Parallelization (LeafP) \citep{cazenave2007parallelization} assigns multiple workers to simulate the same node simultaneously; Root Parallelization (RootP) \citep{cazenave2007parallelization} adopts the workers to independently maintain different search trees, and the statistics are aggregated after all workers complete their jobs; in Tree Parallelization (TreeP) \citep{chaslot2008parallel}, the workers independently perform rollouts on a shared search tree; TreeP with Virtual Loss (VL-UCT) \citep{segal2010scalability,silver2016mastering} and Watch the Unobserved in UCT (WU-UCT) \citep{liu2020watch} pre-adjust the node statistics with side information to achieve a better exploration-exploitation tradeoff. Please refer to Appendix~\ref{Existing Parallel MCTS Algorithms} for a more detailed and thorough discussion of existing parallel MCTS algorithms.

\begin{figure}[t]
    \centering
    \includegraphics[width=\columnwidth]{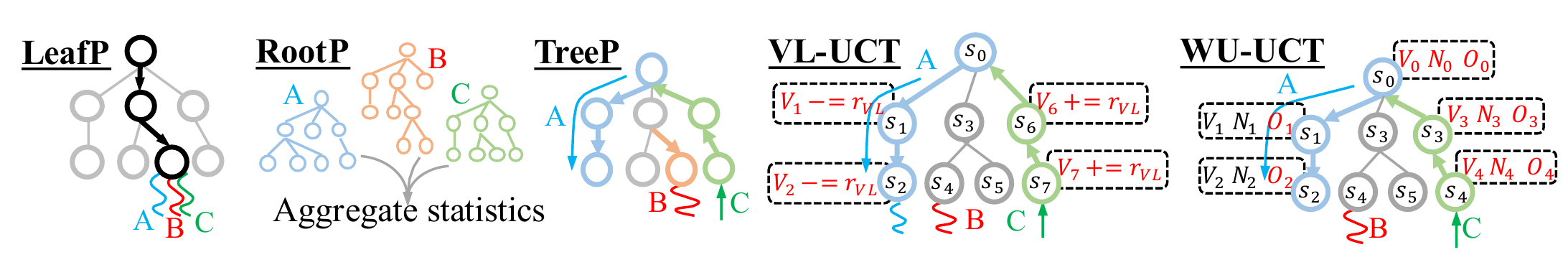}
    \vspace{-2.6em}
    \caption{Typical existing parallel MCTS algorithms. VL-UCT and WU-UCT use virtual loss (i.e., $r_{\mathrm{VL}}$) and number of on-going simulations (i.e., $O$) to pre-adjust node statistics, respectively.}
    \label{fig: typical parallel MCTS algs}
    \vspace{-1.2em}
\end{figure}

\paragraph{Main challenges} Since parallel MCTS algorithms have to initiate new rollouts before all assigned simulation tasks are completed, they are generally not able to incorporate the information from \emph{all} initiated simulations into its statistics (i.e., $Q$ and $N$). As demonstrated in previous studies (e.g., \citet{liu2020watch}), this could lead to significant performance loss compared to sequential MCTS algorithms since the tree policy (Eq.~(\ref{eq: UCT tree policy})) cannot properly balance exploration and exploitation when using such statistics. Therefore, most existing algorithms seek to improve their performance by augmenting the statistics $Q$ and $N$ used by the tree policy, which is done by either adjusting how statistics possessed by different workers are synchronized/aggregated (e.g., LeafP, RootP) or adding additional side information (e.g., VL-UCT, WU-UCT). Specifically, this can be formalized by introducing a set of \emph{modified statistics} (defined as $\overline{Q}$ and $\overline{N}$) in replacement of $Q$ and $N$ in the tree policy (Eq.~(\ref{eq: UCT tree policy})):
    { 
    \setlength{\abovedisplayskip}{0.4em}
    \setlength{\belowdisplayskip}{-1.2em}
    \begin{align}
        \overline{Q} (\state, \action) & := \alpha (\state, \action) \cdot Q (\state, \action) + \beta (\state, \action) \cdot \widetilde{Q} (\state, \action), \quad \; \overline{N} (\state, \action) := N (\state, \action) + \widetilde{N} (\state, \action) \footnotemark, \label{eq: modified value and visit count definition}
    \end{align}}
    
\footnotetext{Given the $N$-related are counts of simulations, which means $\overline{N}$ shall never be smaller than $N$, this equation is sufficient for the general purpose and there is no need for weights before $N$ and $\widetilde{N}$.}
    
\noindent where $\widetilde{Q}$ and $\widetilde{N}$ are a set of \emph{pseudo statistics} that incorporate additional side information; $\alpha$ and $\beta$ control the ratio between $Q$ and $\widetilde{Q}$. Common choices of the pseudo statistics include virtual loss (\citet{segal2010scalability}; for $\widetilde{Q}$) and incomplete visit count (\citet{liu2020watch}; for $\widetilde{N}$). Given this formulation, a natural question is \emph{how to design $\overline{Q}$ and $\overline{N}$ in order to achieve good parallel performance in MCTS?}

\vspace{-0.2em}

\section{Overview of Our Main Theoretical Results}
\label{Key Algorithmic Challenges And An Overview of Theoretical Findings}

\vspace{-0.2em}

The main objective of this paper is to answer the above question by identifying key necessary conditions of $\overline{Q}$ and $\overline{N}$ to achieve desirable performance\footnote{The notion of ``desirable performance'' will be formalized in Sec.~\ref{Evaluating Parallel MCTS Algorithms}.} in parallel MCTS algorithms. Throughout the paper, we highlight two benefits of our theoretical results: in hindsight, they help identify beneficial design principles used in existing algorithms (Sec.~\ref{Rethinking Existing Parallel MCTS Algorithms}); furthermore, they offer simple and effective guidelines for designing better (future) algorithms (Sec.~\ref{Towards Optimal Parallel Tree Search}).

The two necessary conditions are best illustrated in Fig.~\ref{fig: theoretical results demonstration}(a). Consider node $s$ in a search tree where we want to select one of its child nodes. Workers A and B are in their simulation steps, querying an offspring node of $\state_{1}$ and $\state_{2}$, respectively. To introduce the necessary condition of $\overline{N}$, we define the \emph{incomplete visit count} $O (\state, \action)$, which was introduced by \citet{liu2020watch} to track \emph{the number of simulation tasks that has been initiated but not yet completed}. For example, in Fig.~\ref{fig: theoretical results demonstration}(a), both the edges associated with $\state_{1}$ and $\state_{2}$ have incomplete visit counts of $1$ since workers A and B are still simulating their offspring nodes. The necessary condition regarding $\overline{N}$ is stated as follows:
    { \setlength{\abovedisplayskip}{0.4em}
      \setlength{\belowdisplayskip}{-1.2em}
    \begin{align}
        \forall (\state, \action) \in \big \{ \mathrm{edges~in~the~search~tree} \big \} \quad \overline{N} (\state, \action) \geq N (\state, \action) + O (\state, \action). \label{eq: necessary condition of N (abbrev)}
    \end{align}}
    
One potential benefit of adding incomplete visit count (i.e., $O$) into $\overline{N}$ is to improve the diversity of exploration \citep{liu2020watch}. Specifically, since increasing $O$ leads to a decrease of the exploration bonus (the second term) in the tree policy (Eq.~(\ref{eq: UCT tree policy})), nodes with high incomplete visit count will be less likely to be selected by other workers, which increase the diversity of exploration. In our example, the chance of selecting $\state_{3}$ is increased due to the introduction of $O$. 

The necessary condition of $\overline{Q}$ focuses on the similarity between the action value maintained by the parallel MCTS algorithm $\mathbb{A}$ and its sequential counterpart $\mathbb{A}_{\mathrm{seq}}$. Formally, it requires the following \emph{action value gap} $\overline{G}(\state, \action)$ to be \emph{zero} for each edge $(\state, \action)$ in the search tree:
    { \setlength{\abovedisplayskip}{0.4em}
      \setlength{\belowdisplayskip}{-1.0em}
    \begin{align}
        \overline{G} (\state, \action) := \big | \expectation \big [ \, \overline{Q} (\state, \action) \big ] - \expectation \big [ Q^{\mathbb{A}_\mathrm{seq}}_{m} (\state, \action) \big ] \big | \quad (m = N (\state, \action) + O (\state, \action)), \label{eq: the action value gap}
    \end{align}}
    
\noindent where $\overline{Q} (\state, \action)$ is generated by the parallel MCTS algorithm $\mathbb{A}$, $Q^{\mathbb{A}_{\mathrm{seq}}}_{m} (\state, \action)$ represents the action value of a sequential MCTS algorithm $\mathbb{A}_{\mathrm{seq}}$ that starts from the child node of $(\state, \action)$ and runs for $m$ rollouts, and $\expectation [ \cdot ]$ averages over the randomness in the simulation returns. Although seemingly nontrivial to satisfy, we will show that it indeed provides important insights for designing better algorithms.

Finally, we revisit both necessary conditions and give a preview of their two benefits: (i) identifying useful designs in existing algorithms that should be inherited by future algorithms (Section \ref{On Parallelization in Monte Carlo Tree Search}), and (ii) revealing new design principles for future algorithms (Section \ref{Towards Optimal Parallel Tree Search}). First, we identify key techniques used in existing algorithms that are \emph{aligned with our theoretical findings}. We found that none of them satisfy the necessary condition on $\overline{Q}$ and only WU-UCT satisfies the necessary condition of $\overline{N}$. In hindsight, this implies that that the design of $\overline{N}$ in WU-UCT is consistent with our theoretical guidelines. And we further confirm the benefit of this essential design by showing that it facilitates WU-UCT to achieve a cumulative regret of $\mathcal{O} (\ln n + M / \sqrt{\ln n})$ when the maximum tree depth is $2$ (Thm.~\ref{theorem: regret upper bound of WU-UCT in stochastic bandit setup}), where $n$ is the total number of rollouts and $M$ is the number of workers. Comparing to sequential UCT, whose cumulative regret is $\mathcal{O} (\ln n)$, WU-UCT merely incurs an excess regret of $\mathcal{O} (M / \sqrt{\ln n})$ that goes to zero as $n$ increase.
Second and more importantly, the necessary condition of $\overline{Q}$ can guide us in designing better (future) algorithms. Specifically, we show in Fig.~\ref{fig: theoretical results demonstration}(b) that the action value gap $\overline{G}$ is a strong performance indicator of parallel MCTS algorithms. The scatter plots obtained from two Atari games demonstrate that, regardless of algorithms and hyperparameters, there is a strong negative correlation between the action value gap $\overline{G}$ and the performance. In Sec.~\ref{Towards Optimal Parallel Tree Search}, we will demonstrate that finding a surrogate gap to approximate $\overline{G}$ and reducing its magnitude could lead to significant performance improvement across a large number of Atari games.

\begin{figure}[t]
    \centering
    \includegraphics[width=\columnwidth]{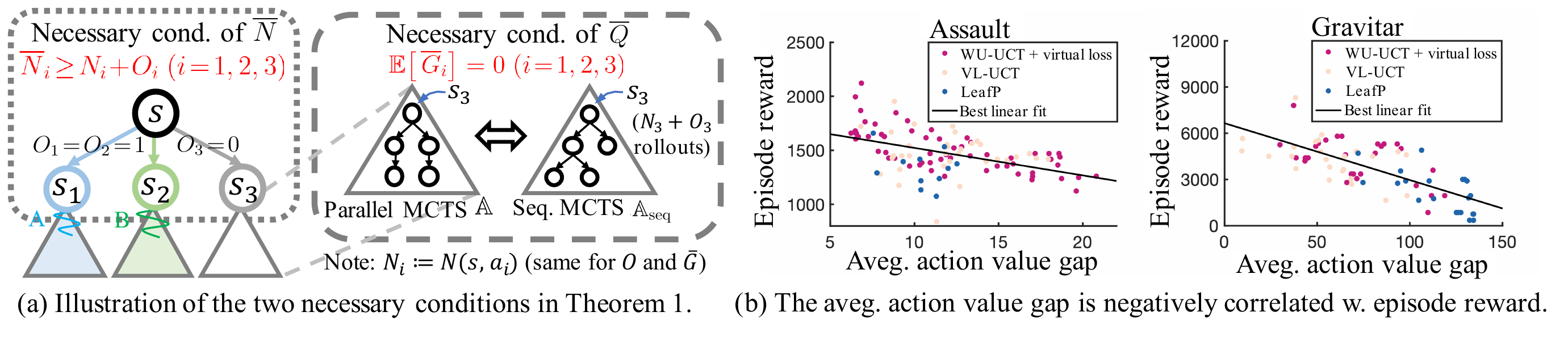}
    \vspace{-2.2em}
    \caption{The necessary conditions to achieve vanishing excess regret and their implications.}
    \label{fig: theoretical results demonstration}
    \vspace{-1.2em}
\end{figure}

\vspace{-0.3em}

\section{Parallel MCTS: Theory and Implications}
\label{On Parallelization in Monte Carlo Tree Search}

\vspace{-0.3em}

Following our aforementioned takeaways, we start this section with formalizing the evaluation criteria of MCTS parallelization before presenting the rigorous development of our theoretical results.

\vspace{-0.2em}

\subsection{What is Effective Parallel MCTS?}
\label{Evaluating Parallel MCTS Algorithms}

\vspace{-0.2em}

We analyze the performance of parallel MCTS algorithms by examining their \emph{performance loss} under a fixed \emph{speedup} requirement. To begin with, we define the following metrics.


\noindent \textbf{Speedup} $\;$ The speedup of a parallel MCTS algorithm $\mathbb{A}$ using $M$ workers\footnote{A worker refers to a computation unit in practical algorithms that performs simulation tasks \emph{sequentially}.} is defined as
    { \setlength{\abovedisplayskip}{0.3em}
      \setlength{\belowdisplayskip}{-0.7em}
    \begin{align*}
        \mathrm{speedup} = \frac{\mathrm{runtime~of~the~sequential~MCTS}}{\mathrm{runtime~of~algorithm}~\mathbb{A}~\mathrm{using}~M~\mathrm{workers}},
    \end{align*}}
    
\noindent where the runtime of both the sequential and the parallel algorithms is measured by the duration of performing the same fixed number of rollouts. Assuming simulation is much more time-consuming compared to other steps,\footnote{This holds in general since \emph{only} the simulation step requires massive interactions with the environment.} parallel MCTS algorithms have a speedup \emph{close to $M$} (see also Section \ref{Towards Optimal Parallel Tree Search}) since all $M$ workers will be occupied by simulation tasks most of the time.

\noindent \textbf{Performance loss} $\;$ We measure the performance of a parallel MCTS algorithm $\mathbb{A}$ by \emph{expected cumulative regret}, a common metric also used in related theoretical studies \citep{kocsis2006improved,auer2002finite,auer2002using}:
    { \setlength{\abovedisplayskip}{-0.8em}
      \setlength{\belowdisplayskip}{-0.8em}
    \begin{align}
        \mathrm{Regret}_{\mathbb{A}} (n) := \sum_{i = 1}^{n} \expectation \big [ V^{*}_{i} (\state_{0}) - V_{i} (\state_{0}) \big ],
        \label{eq: regret definition} 
    \end{align}}

\noindent where $\state_{0}$ is the root state of the search tree; $n$ is the number of rollouts; $V_{i} (\state_{0})$ is the value estimate of $\state_{0}$ obtained in the $i$th rollout of algorithm $\mathbb{A}$, which is computed according to Eq.~(\ref{eq: recursive N and V hat update}); similarly, $V^{*}_{i} (\state_{0})$ is the estimated value of $\state_{0}$ acquired in the $i$th rollout of an oracle algorithm that always select the highest-rewarded action; the expectation is performed to average over the randomness in the simulation returns. Intuitively, cumulative regret measures the expected regret of not having selected the optimal path. We measure the \emph{performance loss} of a parallel MCTS algorithm $\mathbb{A}$ by \emph{excess regret}, which is defined as the difference between the regret of $\mathbb{A}$ and its sequential counterpart $\mathbb{A}_{\mathrm{seq}}$ (i.e., $\mathrm{Regret}_{\mathbb{A}} (n) - \mathrm{Regret}_{\mathbb{A}_{\mathrm{seq}}} (n)$). We say algorithm $\mathbb{A}$ has \emph{vanishing excess regret} if and only if its excess regret converges to zero as $n$ goes to infinity. Roughly speaking, having vanishing excess regret means the parallel algorithm is almost as good as sequential MCTS under large $n$.\footnote{Note that with relatively small $n$, parallel MCTS is in general inferior to their sequential counterpart since they are not able to collect sufficient information for effective exploration-exploitation tradeoff during selection.}

Beside cumulative regret, simple regret is also widely used in related studies. While it is generally agreed that simple regret is preferable in the Multi-Armed Bandit (MAB) setting given only the final recommendation affect the performance, it is still debatable whether MCTS should seek to minimize simple or cumulative regret \citep{pepels2014minimizing}. Specifically, nodes in the search tree need both good final performance (corr. to simple regret) to make good recommendations and good any-time performance (corr. to cumulative regret) to back-propagate well-estimated values $V (\state)$. In fact, recently proposed (sequential) MCTS algorithms largely use hybrid approaches that seek to minimize \emph{both} simple regret and cumulative regret \citep{feldman2014simple,feldman2014mabs,kaufmann2017monte,liu2015regulation}.
Due to this lack of consensus, we in this paper are in particular focused on excess cumulative regret, and leave analysis revolving around simple regret to future work.

\vspace{-0.2em}

\subsection{When Will Excess Regret Vanish?}
\label{When Will the Excess Regret Vanish?}

\vspace{-0.2em}

This section examines \emph{what conditions should be satisfied for a parallel MCTS algorithm to achieve vanishing excess regret}. To perform a unified theoretical analysis of existing parallel MCTS algorithms, we introduce a general algorithm framework (formally introduced in Appendix~\ref{A General Framework for Parallel MCTS Algorithms}) that covers most existing parallel MCTS algorithms and their variants as its special cases. Specifically, Appendix~\ref{Specialization of Existing Parallel MCTS Algorithms into the General Framework} provides a rigorous justification of how the general framework can be specialized to LeafP, RootP, TreeP, VL-UCT, and WU-UCT. The following theorem gives two necessary conditions for any algorithm specialized from the general framework to achieve vanishing excess regret.

\begin{theorem}
\label{theorem: necessary condition for a good algorithm in the tree search setting}
Consider an algorithm $\mathbb{A}$ that is specified from the general parallel MCTS framework formally introduced in Appendix~\ref{A General Framework for Parallel MCTS Algorithms}. Choose $\widetilde{N} (\state, \action)$ as a function of $O (\state, \action)$. If there exists an edge $(\state, \action)$ in the search tree such that $\mathbb{A}$ violates any of the following conditions: \vspace{0.1em} \newline
$\bullet \;$ \textbf{Necessary cond. of $\, \overline{Q}$: $\,$} $\overline{G} (\state, \action) \!:=\! \big | \expectation [ \overline{Q} (\state, \action) ] \!-\! \expectation [ Q^{\mathbb{A}_{\mathrm{seq}}}_{m} (\state, \action) ] \big | \!=\! 0 \; \big ( m \!=\! N (\state, \action) \!+\! O (\state, \action) \big )$, \eqnum\label{eq: necessary condition 1} 

\vspace{-0.6em}

$\bullet \;$ \textbf{Necessary cond. of $\, \overline{N}$: $\,$} $\overline{N} (\state, \action) \geq N (\state, \action) + O (\state, \action)$, \eqnum\label{eq: necessary condition 2}

\vspace{-0.3em}

\noindent then there exists an MDP $\mathcal{M}$ such that the excess regret of running $\mathbb{A}$ on MDP $\mathcal{M}$ does not vanish.
\end{theorem}

\vspace{-0.4em}

While the necessary condition of $\overline{N}$ is rather straightforward, suggesting that the modified visit count $\overline{N} (\state, \action)$ should be no less than the total number of simulations \emph{initiated} (regardless of completed or not) from offspring nodes of $(\state, \action)$ (i.e., $N (\state, \action) \!+\! O (\state, \action)$), the necessary condition of $\overline{Q}$ needs further elaboration. Intuitively, the action value gap $\overline{G} (\state, \action)$ measures \emph{how well the modified action value $\overline{Q} (\state, \action)$ of $\mathbb{A}$ approximates the action value computed by its sequential counterpart $\mathbb{A}_{\mathrm{seq}}$} (i.e., $Q^{\mathbb{A}_{\mathrm{seq}}}_{m} (\state, \action)$). There are two main obstacles toward lowering the action value gap and satisfy its necessary condition (i.e., Eq.~(\ref{eq: necessary condition 1})). First, as demonstrated in Sec.~\ref{Key Algorithmic Challenges And An Overview of Theoretical Findings} as well as previous studies \citep{chaslot2008parallel,liu2020watch}, the statistics used by the tree policy (Eq.~(\ref{eq: UCT tree policy})) in parallel MCTS algorithms tend to harm the effectiveness of the tree policy, which leads to suboptimal node selections and hence biases the simulation outcomes compared to that of the sequential algorithm. Second, as hinted by our notation, while the modified action value $\overline{Q} (\state, \action)$ incorporates information from $N (\state, \action)$ simulation returns, $Q^{\mathbb{A}_{\mathrm{seq}}}_{m} (\state, \action)$ is the average of $N (\state, \action) \!+\! O (\state, \action)$ simulation outcomes. This requires the modified action value $\overline{Q}$ to incorporate additional information that helps ``anticipate'' the outcomes of incomplete simulations through the pseudo action value $\widetilde{Q}$.

\vspace{-0.2em}

\subsection{Rethinking Existing Parallel MCTS Algorithms}
\label{Rethinking Existing Parallel MCTS Algorithms}

\vspace{-0.2em}

In retrospect, we examine which techniques proposed in existing algorithms should be retained in future parallel MCTS algorithms by inspecting whether they satisfy the two necessary conditions. First, regarding $\overline{Q}$, existing algorithms either modify how simulation returns from different workers aggregate to generate action values (e.g., LeafP and RootP) or use virtual loss \citep{segal2010scalability} to penalize action value of nodes with high incomplete visit count (e.g., VL-UCT), which not necessarily minimize the action value gap $\overline{G}$. Hence, based on our knowledge, the necessary condition of $\overline{Q}$ is not satisfied by any existing algorithms. Next, regarding $\overline{N}$, we found that WU-UCT satisfies its necessary condition by using the sum of complete and incomplete visit count as its modified visit count (i.e., $\overline{N} (\state, \action) \!:=\! N (\state, \action) \!+\! O (\state, \action)$). Now a natural question to ask is \emph{whether satisfying the necessary condition of $\overline{N}$ offers noticeable gain in WU-UCT's performance}, which can be answered in the affirmative. Specifically, besides its empirical success reported in the original paper, we demonstrate the superiority of WU-UCT from a theoretical perspective through the following theorem.

\begin{theorem}
\label{theorem: regret upper bound of WU-UCT in stochastic bandit setup}
Consider a tree search task $\mathbb{T}$ with maximum depth $D \!=\! 2$ (abbreviate as the depth-2 tree search task): it contains a root node $\state$ and $K$ feasible actions $\{ \action_{i} \}_{i = 1}^{K}$ at $\state$, which lead to terminal states $\{ \state_{i} \}_{i = 1}^{K}$, respectively. Let $\mu_{i} \!:=\! \expectation [ V (\state_{i}) ]$, $\mu^{*} \!:=\! \max_{i} \mu_{i}$ and $\Delta_{k} \!:=\! \mu^{*} \!-\! \mu_{k}$, and further assume: $\forall i, V (\state_{i}) \!-\! \mu_{i}$ is 1-subgaussian \citep{buldygin1980sub}. The cumulative regret of running WU-UCT \citep{liu2020watch} with $n$ rollouts on $\mathbb{T}$ is upper bounded by:
    { \setlength{\abovedisplayskip}{0.1em}
      \setlength{\belowdisplayskip}{-0.9em}
    \begin{align*}
        \underbrace{\sum_{k: \mu_{k} < \mu^{*}} \Big ( \frac{8}{\Delta_{k}} + 2 \Delta_{k} \Big ) \ln n + \Delta_{k}}_{\rewards_{\mathrm{UCT}} (n)} + \underbrace{4 M \sum_{k: \mu_{k} < \mu^{*}} \frac{\Delta_{k}^{2}}{\sqrt{\ln n}}}_{\mathrm{excess~regret}},
    \end{align*}}
    
\noindent where $\rewards_{\mathrm{UCT}} (n)$ is the cumulative regret of running the (sequential) UCT for $n$ steps on $\mathbb{T}$.
\end{theorem}

\vspace{-0.6em}

Before interpreting the theorem, we emphasize that this result only apply to tasks where the maximum depth of the search tree is 2, which closely resembles the Multi-Armed Bandit \citep{auer2002finite,auer2002using} setup. Therefore, although WU-UCT has some desirable properties that other existing algorithms do not, it is still far from optimal when considering MCTS tasks in general. 

Thm.~\ref{theorem: regret upper bound of WU-UCT in stochastic bandit setup} indicates that the regret upper bound of WU-UCT in the depth-2 tree search task consists of two terms: the cumulative regret of the sequential UCT algorithm (i.e., $\rewards_{\mathrm{UCT}}$) and an excess regret term that \emph{converges to zero as $n$ increases}. Apart from showing a desirable theoretical property of WU-UCT, this result suggests that designing algorithms that satisfy the necessary conditions in Thm.~\ref{theorem: necessary condition for a good algorithm in the tree search setting} can potentially offers empirical as well as theoretical benefits.

In conclusion, by looking back at existing parallel MCTS algorithms, the necessary conditions suggest that retaining WU-UCT's approach to augment $\overline{N}$ would be beneficial. This left us with the question \emph{how to make use of the necessary condition of $\overline{Q}$ to further improve existing parallel MCTS algorithms}, which will be addressed in the following section. 

\section{Theory in Practice: A Promising Study}
\label{Towards Optimal Parallel Tree Search}
\vspace{-0.2em}

In this section, we demonstrate that exploiting the proposed necessary conditions in Theorem \ref{theorem: necessary condition for a good algorithm in the tree search setting} can immediately lead to a more effective parallel MCTS algorithms: \textbf{B}alance the \textbf{U}nobserved in UCT (BU-UCT). The newly proposed BU-UCT, albeit preliminary, is shown to outperform strong baselines (including WU-UCT, the current state-of-the-art) on 11 out of 15 Atari games. We want to highlight that BU-UCT is only used as an illustrative example about how to use our theoretical results in practice, and we hope this encouraging result could inspire more future work to develop better parallel MCTS algorithms with our theory.

\textbf{Algorithm Design} $\;$ Thm.~\ref{theorem: necessary condition for a good algorithm in the tree search setting} suggests that parallel-MCTS algorithms should be designed to satisfy both necessary conditions. First, it is relatively easier to construct $\overline{N}$ to satisfy its necessary condition. For example, we can borrow wisdom from WU-UCT to choose $\overline{N} (\state, \action) \!:=\! N (\state, \action) \!+\! O (\state, \action)$. On the other hand, however, the necessary condition on $\overline{Q}$ (i.e., $\overline{G} (\state, \action) \!=\! 0$) is more difficult to satisfy strictly. Nevertheless, we find that the magnitude of the average action value gap $\overline{G}(\state, \action)$ has a strong negative correlation with the actual performance (i.e., episode reward) --- see Fig.~\ref{fig: theoretical results demonstration}(b).\footnote{Note that each game step requires a new search tree. Hence the action value gap is averaged w.r.t. (i) search trees built at different game steps and (ii) different nodes in a search tree. See Appendix~\ref{Details for the Action Value Gap vs. Performance Experiments} for more detail.} And the behavior holds true regardless of the algorithms (points with different colors represent different algorithms) as well as the hyperparameters (points of the same color denote results obtained from different hyperparameters). The phenomenon suggests that \emph{designing a parallel-MCTS algorithm that reduces $\overline{G}$ could lead to better performance in practice}. However, according to Eq.~(\ref{eq: necessary condition 1}), directly using the original gap $\overline{G}(s,a)$ for algorithm design is not practical because it requires running the sequential UCT algorithm to compute $Q^{\mathbb{A}_{\mathrm{seq}}}_{m}$. Therefore, a more realistic approach is to construct a \emph{surrogate gap} to approximate $\overline{G}(s,a)$ based on the available statistics. In the following, we give one example to show how to construct such a surrogate gap for designing a better parallel MCTS algorithms. Please refer to Appendix~\ref{Alternative Surrogate Statistics} for more potential options for the surrogate gap. 

Let $O_{i} (\state, \action)$ be the number of on-going simulations associated with the edge $(\state, \action)$ at the $i$th rollout step. We consider using the following statistics $\overline{G}^{*} (\state, \action)$ as a surrogate gap to approximate $\overline{G}(\state, \action)$:
    { \setlength{\abovedisplayskip}{0.1em}
      \setlength{\belowdisplayskip}{-0.9em}
    \begin{align*}
        \overline{G}^{*} \! (\state, \action) \!:=\! \max_{\action' \in \actions} \overline{O}(s',a')
        \!=\! \max_{\action' \in \actions}
        \Big\{ \frac{1}{n} \sum_{i = 1}^{n} \! O_{i} (\state', \action')
        \Big\}
        \; \text{(}\state'\text{~is~the~next~state~following~}(\state, \action)\text{)}, \numberthis \label{eq: surrogate statistics}
    \end{align*}}

\noindent where $n$ is the number of rollouts. Before discussing its key insights, we first examine the correlation between $\overline{G}^{*}$ and the action value gap $\overline{G}$. As shown in Fig.~\ref{fig: action value gap approximation}(a), except for a few outliers, $\overline{G}^{*} (\state, \action)$ and $\overline{G} (\state, \action)$ have a strong positive correlation.\footnote{Note that there are a few data points with $\overline{G} (\state, \action) \!>\! 10$ that the surrogate statistics cannot fit properly, which indicates that there could exist better surrogate gap that potentially leads to better parallel MCTS algorithms.} Motivated by this observation, we seek to design a better parallel MCTS algorithm by reducing the surrogate gap $\overline{G}^{*} (\state, \action)$. In the following, we introduce the proposed algorithm BU-UCT, and highlight how it lowers the surrogate gap $\overline{G}^{*} (\state, \action)$.

\begin{figure}[t]
    \centering
    \includegraphics[width=\columnwidth]{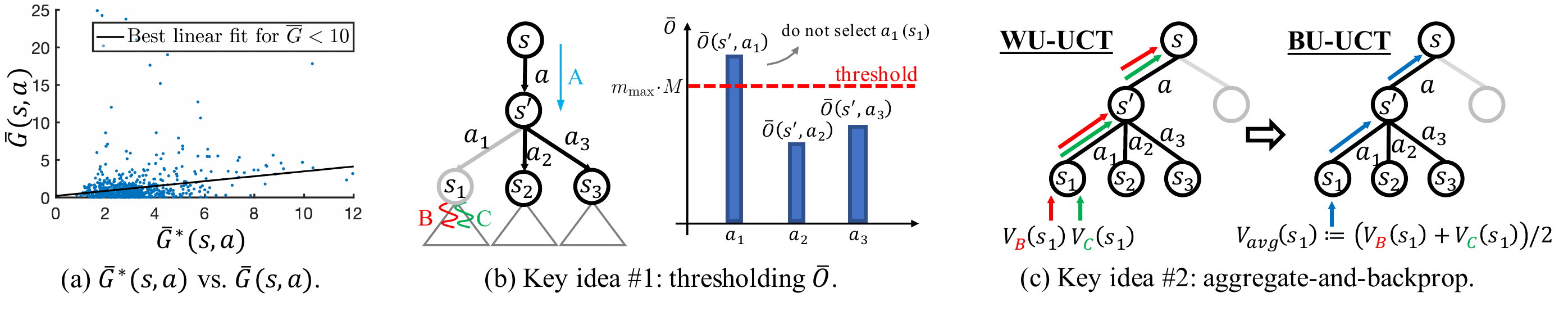}
    \vspace{-2.4em}
    \caption{The BU-UCT algorithm. (a) Motivation: The statistics $\overline{G}^{*}$ is strongly correlated with the original gap $\overline{G}$ , suggesting that it can be used as a surrogate gap to guide algorithm design. (b) Key idea \#1: reducing $\overline{G}^{*}$ by thresholding $\overline{O}$ --- only query nodes whose $\overline{O}$ is smaller than a threshold. (c) Key idea \#2: aggregate the simulation returns on a same node (e.g., $\state_{1}$) and then backpropagate.}
    \label{fig: action value gap approximation}
    \vspace{-1.0em}
\end{figure}

\textbf{Algorithm Details} $\;$ Built on top of WU-UCT, BU-UCT proposes to lower $\overline{G}^{*}$ through (i) thresholding $\overline{O}$, and (ii) aggregating-and-backpropagating simulation returns. The first idea, thresholding $\overline{O}$, seek to explicitly set an upper limit to $\overline{O}$ (and hence $\overline{G}^{*}$). Specifically, BU-UCT keeps record of the $\overline{O}$ values on all edges and assure edges whose $\overline{O}$ is above a threshold will not be selected by the tree policy. Concretely, this is achieved with the following modified action value $\overline{Q}$:
    { \setlength{\abovedisplayskip}{0.4em}
      \setlength{\belowdisplayskip}{-1.2em}
    \begin{align}
        \overline{Q} (\state, \action) := Q (\state, \action) + \mathbb{I} \big [ \, \overline{O} (\state, \action) \!<\! m_{\mathrm{max}} \!\cdot\! M \big ], \label{eq: BU-UCT modified action value}
    \end{align}}
    
\noindent where $m_{\mathrm{max}} \!\in\! (0, 1)$ is a hyperparameter and $M$ is the number of workers; the indicator function $\mathbb{I} [ \cdot ]$ is defined to be zero when the condition holds and $-\infty$ otherwise. Consider the example given in Fig.~\ref{fig: action value gap approximation}(b). Since $\overline{O} (\state', \action_{1})$ is above the threshold $m_{\mathrm{max}} \!\cdot\! M$, its corresponding $\overline{Q}(\state', \action_{1})$ becomes $-\infty$ due to the second term of Eq.~(\ref{eq: BU-UCT modified action value}) and hence the tree policy will not allow the new worker A to select $\action_{1}$, which will eventually decrease $\overline{O} (\state', \action_{1})$ (and hence lower $\overline{G}^{*} (\state, \action)$).

The second key idea, aggregating-and-backpropagating simulation returns, decreases $\overline{G}^{*}$ by reducing the maximum value in $\{ \overline{O} (\state', \action') \mid \action' \!\in\! \actions \}$. Intuitively, $\overline{G}^{*} (\state, \action)$ will be large only if some child nodes of $\state'$ are \emph{constantly} (reflected by the ``average'' operator in the definition of $\overline{O}$) selected by multiple workers. However, it is undesirable for such nodes to be extensively queried in earlier stages as it prevents the algorithm from exploring other nodes. Therefore, BU-UCT decrease the maximum $\overline{O}$ (and hence $\overline{G}^{*}$) by lowering $\overline{N}$ in earlier stages to encourage exploration of other nodes. Specifically, as shown in Fig.~\ref{fig: action value gap approximation}(c), BU-UCT aggregates the simulation returns originated from the same node ($V_{B} (\state_{1})$ and $V_{C} (\state_{1})$) into their mean value ($V_{\mathrm{avg}} (\state_{1})$) and then backpropagate it. Compared to backpropagating all simulation returns individually, backpropagate aggregated statistics lowers $\overline{N}$ at all children of $\state'$, which encourage exploration in earlier stages and hence lowers $\overline{G}^{*}$.

\textbf{Experiment setup} $\;$ We compare BU-UCT with four baselines (i.e., LeafP, RootP, VL-UCT, and WU-UCT) on 15 Atari games. We use a pretrained PPO policy as the default policy during simulation. All experiments are performed with 128 rollouts and 16 workers. See Appendix~\ref{Details of the Experiments} for more details.

\textbf{Experiment results} $\;$ First, we verify speedup. Across 15 Atari games, BU-UCT achieves an average per-step speedup of 14.33 using 16 workers, suggesting that BU-UCT achieves (approximately) linear speedup even with a large number of workers. Next, we compare the performance, measured by average episode reward, between BU-UCT and four baselines. On Each task, we repeat 5 times with the mean and standard deviation reported in Table~\ref{table:atari_benchmark}. Thanks to its efforts to lower the action value gap, BU-UCT outperforms all considered parallel alternatives in 11 out of 15 tasks. Pairwise student t-tests further show that BU-UCT performs significantly better ($p\text{-value} \!<\! 0.05$) than WU-UCT, TreeP, LeafP, and RootP in 2, 8, 12, and 12 tasks, respectively; note except in RoadRunner where WU-UCT tops the chart, in all other tasks BU-UCT performs statistically comparably to the baselines, which promisingly renders it as a potential default choice, if one wants to try one parallel MCTS algorithm.

\newcolumntype{R}{>{$}r<{$}}
\newcolumntype{L}{>{$}l<{$}}
\newcolumntype{M}{R@{}L}
\newcolumntype{C}{R@{}L@{$\;$}L@{}L@{}L@{}L|}

\begin{table*}[t]
\renewcommand\arraystretch{0.8}
\caption{Performance on 15 Atari games. Average episode return ($\pm$ standard deviation) over 5 trials are reported. The best average scores are highlighted in boldface.  According to t-tests, BU-UCT significantly outperforms or is comparable with the existing alternative on 14 games, except RoadRunner where WU-UCT is better. ``\textbf{*}'', ``$\dagger$'', ``$\ddagger$'', and ``$\mathsection$'' denote BU-UCT's large-margin superiority (p-value $<$ 0.05) over WU-UCT, VL-UCT, LeafP, and RootP, respectively.}
\label{table:atari_benchmark}
\vspace{0.2em}

\centering
{\fontsize{8}{8}\selectfont

\setlength{\tabcolsep}{0.8mm}{
\begin{tabular}{c|CMMMM}
\toprule
Environment & \multicolumn{6}{c|}{BU-UCT (ours)} & \multicolumn{2}{c}{WU-UCT} & \multicolumn{2}{c}{VL-UCT} & \multicolumn{2}{c}{LeafP} & \multicolumn{2}{c}{RootP} \\
\midrule

Alien & 5320&\pm231 & &\dagger&\ddagger& & \textbf{5938}&\pm1839 & 4200&\pm1086 & 4280&\pm1016& 5206&\pm282 \\
Boxing & \textbf{100}&\pm0 & &\dagger&\ddagger&\mathsection & \textbf{100}&\pm0 & 99&\pm0 & 95&\pm4 & 98&\pm1 \\
Breakout & \textbf{425}&\pm30 & & &\ddagger&\mathsection & 408&\pm21 & 390&\pm33& 331&\pm45& 281&\pm27 \\
Centipede & \textbf{1610419}&\pm338295 & &\dagger&\ddagger&\mathsection & 1163034&\pm403910 & 439433&\pm207601 & 162333&\pm69575 & 184265&\pm104405 \\
Freeway & \textbf{32}&\pm0& &&& & \textbf{32}&\pm0 & \textbf{32}&\pm0 & 31&\pm1 & \textbf{32}&\pm0 \\
Gravitar & \textbf{5130}&\pm499 & &&\ddagger& & 5060&\pm568 & 4880&\pm1162 & 3385&\pm155 & 4160&\pm1811 \\
MsPacman & 17279&\pm6136 & &&\ddagger&\mathsection & \textbf{19804}&\pm2232 & 14000&\pm2807 & 5378&\pm685 & 7156&\pm583 \\
NameThisGame & \textbf{47066}&\pm5911 &\textbf{*}&\dagger&\ddagger&\mathsection & 29991&\pm1608 & 23326&\pm2585 & 25390&\pm3659 & 27440&\pm9533 \\
RoadRunner & 44920&\pm1478 & &\dagger&\ddagger&\mathsection & \textbf{46720}&\pm1359 & 24680&\pm3316 & 25452&\pm2977 & 38300&\pm1191 \\
Robotank & \textbf{121}&\pm18 & &\dagger&\ddagger&\mathsection & 101&\pm19 & 86&\pm13 & 80&\pm11 & 78&\pm13 \\
Qbert & \textbf{15995}&\pm2635 & &&&\mathsection & 13992&\pm5596 & 14620&\pm5738 & 11655&\pm5373 & 9465&\pm3196 \\
SpaceInvaders & \textbf{3428}&\pm525 & &&&\mathsection & 3393&\pm292 & 2651&\pm828 & 2435&\pm1159 & 2543&\pm809 \\
Tennis & 3&\pm1 & &\dagger&\ddagger&\mathsection & \textbf{4}&\pm1 & -1&\pm0 & -1&\pm0 & 0&\pm1 \\
TimePilot & \textbf{111100}&\pm58919 &\textbf{*}&\dagger&\ddagger&\mathsection & 55130&\pm12474 & 32600&\pm2165 & 38075&\pm2307 & 45100&\pm7421 \\
Zaxxon & \textbf{42500}&\pm4725 & &&\ddagger&\mathsection & 39085&\pm6838 & 39579&\pm3942 & 12300&\pm821 & 13380&\pm769 \\

\bottomrule
\end{tabular}}
}
\vspace{-1.3em}
\end{table*}

\vspace{-0.8em}

\section{Related Works}
\label{Related Works}

\vspace{-0.2em}

MCTS has a profound track record of being adopted to achieve optimal planning and decision making in complex environments \citep{schafer2008uct,browne2012survey,silver2016mastering}. Recently, it has also been combined with learning methods to bring mutual improvements \citep{guo2014deep,shen2018m,silver2017mastering}. 
To maximize the power of MCTS and enable its usage in time-sensitive tasks, effective parallel algorithms are imperative \citep{bourki2010scalability,segal2010scalability}. 
Specifically, leaf parallelization \citep{cazenave2007parallelization,kato2010parallel} manages to collect better statistics by assigning multiple workers to query the same node, at the expense of reducing the tree search diversity. 
In root parallelization, multiple trees are built and statistics are periodically synchronized. 
It promises better performance in some real-world tasks \citep{bourki2010scalability}, while being inferior on Go \citep{soejima2010evaluating}.
In contrast, tree parallelization assigns workers to traverse the same tree. 
To increase search diversity, \citet{chaslot2008parallel} proposes a virtual loss. 
Though having been adopted in some high-profile applications \citep{powley2011determinization}, virtual loss punishes performance under even four workers \citep{mirsoleimani2017analysis}.
So far, WU-UCT \citep{liu2020watch} achieves the best tradeoff (i.e., linear speedup with small performance loss) by introducing statistics to track on-going simulations.
Another related line of works focus on \emph{distributed} multi-armed bandits (MAB) \citep{liu2010distributed,hillel2013distributed,lai1985asymptotically,martinez2019decentralized}, which is similar to parallel MCTS; in both multiple workers collaborate to improve the planning performance. 
Though inspiring, this line shares an overarching theme that highlights inter-agent communication, making their results not directly adaptable to our setting. 

\vspace{-0.8em}

\section{Conclusion}
\label{Conclusion}

\vspace{-0.2em}

In this paper, we established the first theoretical foundation for parallel MCTS algorithm. In particular, we derived two necessary conditions for the algorithms to achieve a desired performance. The conditions can be used to diagnose existing algorithms and guide future algorithm design. We justify the first benefit (i.e., diagnosing existing algorithms) by identifying the key design wisdom inherent in existing algorithms. The second benefit (i.e., inspiring future algorithms) is demonstrated by constructing a new parallel MCTS algorithm, BU-UCT, based on our theoretical guidelines.

\bibliography{references}
\bibliographystyle{iclr2021_conference}

\appendix
\input{supplementary}

\end{document}

%% file: math_commands.tex
\usepackage{amsmath,amsfonts,bm}

\def\eqref#1{equation~\ref{#1}}

\def\1{\bm{1}}

\DeclareMathAlphabet{\mathsfit}{\encodingdefault}{\sfdefault}{m}{sl}
\SetMathAlphabet{\mathsfit}{bold}{\encodingdefault}{\sfdefault}{bx}{n}

\DeclareMathOperator*{\argmax}{arg\,max}

%% file: supplementary.tex
\onecolumn
\section*{\centering \huge \bf Supplementary Material}

In this supplementary material, we first give a more detailed review of existing parallel MCTS algorithms in Appendix~\ref{Existing Parallel MCTS Algorithms}. We then formally introduce the general algorithm framework for parallel MCTS algorithms (see Figure~\ref{fig: general parallel MCTS framework} as well as Algorithm~\ref{alg: general framework of parallel MCTS algs}) in Appendix~\ref{A General Framework for Parallel MCTS Algorithms}, especially showing how existing algorithms fall into our general framework (Appendix~\ref{Specialization of Existing Parallel MCTS Algorithms into the General Framework}). Then we provide the detailed proofs of the parallel algorithms in Appendix~\ref{Appendix:MCTS}. The supplementary ends up with more details on the proposed algorithm BU-UCT (Appendix~\ref{Additional Details for BU-UCT}), the surrogate statistics introduced in Section~\ref{Towards Optimal Parallel Tree Search} (Appendix~\ref{Alternative Surrogate Statistics}), and additional details for the Atari experiments (Appendix~\ref{Details of the Experiments}).

\section{Existing parallel MCTS algorithms}
\label{Existing Parallel MCTS Algorithms}

Leaf Parallelization (LeafP) \citep{cazenave2007parallelization}, Root Parallelization (RootP) \citep{cazenave2007parallelization}, and Tree Parallelization (TreeP) \citep{chaslot2008parallel} develop different ways to cooperate among the workers.\footnote{LeafP and RootP are originally called ``single-run'' parallelization \citep{cazenave2007parallelization}.} 
As shown in Figure~\ref{fig: typical parallel MCTS algs}, \emph{LeafP} and \emph{RootP} parallelize MCTS from the leaf nodes and the root node, respectively. Specifically, in LeafP, only a main process performs sequential rollouts. However, during the simulation step, $M$ workers simultaneously query the same node choosed in the selection and expansion steps, and after all simulations complete, the simulation returns are backpropagated to update node statistics along the selected path. In RootP, $M$ workers independently run sequential MCTS and maintain different search trees, each with a predefined rollout budget. After all workers complete their jobs, the statistics are aggregated to make the final decision (i.e. which action to take). On the other hand, in TreeP, the workers independently perform sequential rollouts on a shared search tree. Node statistics updated by any worker are immediately observable by other workers.
    
TreeP with Virtual Loss (VL-UCT) \citep{segal2010scalability,silver2016mastering} and Watch the Unobserved in UCT (WU-UCT) \citep{liu2020watch} pre-adjust the node statistics with side information before the respective simulation tasks are initiated. As shown in Figure~\ref{fig: typical parallel MCTS algs}, to encourage different workers to explore different nodes, VL-UCT penalizes the action value (i.e., $\overline{Q}$) of nodes that are currently being simulated by some workers so that other workers tend not to query this same set of nodes. Specifically, it has the following two variants. The \emph{hard penalty} version \citep{chaslot2008parallel} adds fixed virtual rewards $r_{\mathrm{VL}}$ directly to the average return and uses the following expression in the tree policy:
    \begin{align*}
        \overline{Q} (\state, \action) := Q (\state, \action) - O (\state, \action) \cdot r_{\mathrm{VL}}. \quad \text{(VL-UCT hard penalty)}
    \end{align*}
Instead of directly penalizing $\overline{Q}$, when a node is being simulated by a worker, the \emph{soft penalty} version \citep{silver2016mastering} adds $n_{\mathrm{VL}}$ virtual simulation returns each with reward $ - r_{\mathrm{VL}}$: 
    \begin{align*}
        \overline{Q} (\state, \action) &:= \frac{Q (\state, \action) \cdot N (\state, \action) - r_{\mathrm{VL}} \cdot n_{\mathrm{VL}} \cdot O (\state, \action)}{N (\state, \action) + n_{\mathrm{VL}} \cdot O (\state, \action)}, \\
        \overline{N} (\state, \action) &:= N (\state, \action) + n_{\mathrm{VL}} \cdot O (\state, \action). \quad \text{(VL-UCT soft penalty)}
    \end{align*}
Intuitively, the hard version of VL-UCT aggressively encourages different workers to explore different nodes, while the soft version has diminishing effect as the visit count grows to infinity. 

In anticipation that the confidence in our estimates of $Q (\state, \action)$ will eventually increase if some child nodes of $(\state, \action)$ are currently being simulated, \cite{liu2020watch} proposes to only adjust the visit count $\overline{N}$ by
    \begin{align*}
        \overline{N} (\state, \action) := N (\state, \action) + O (\state, \action). \quad {(\text{WU-UCT})}
    \end{align*}

\section{A General Framework for Parallel MCTS Algorithms}
\label{A General Framework for Parallel MCTS Algorithms}

This section formally introduce a general algorithm framework for parallel MCTS algorithms, which is a critical component of Theorem~\ref{theorem: necessary condition for a good algorithm in the tree search setting}. Specifically, since the necessary conditions stated in Theorem~\ref{theorem: necessary condition for a good algorithm in the tree search setting} only apply to algorithms that are specialized from the general algorithm framework, it is important that the framework covers all major existing parallel MCTS algorithms and their variants. In the following, we first introduce the general framework for parallel MCTS (Appendix~\ref{Formal Introduction of The General Algorithm Framework}) and provide additional details (Appendix~\ref{Additional Details of the General Algorithm Framework}). Appendix~\ref{Specialization of Existing Parallel MCTS Algorithms into the General Framework} then explains how existing approaches fit in the general algorithm framework.

\subsection{Formal Introduction of The General Algorithm Framework}
\label{Formal Introduction of The General Algorithm Framework}

In the following, we first provide an overview of the general framework for parallel MCTS algorithms, highlighting its two key modules, \emph{statistics collection} and \emph{statistics augmentation}, which allow it to represent various existing methods. We then discuss both modules in detail.

\noindent \textbf{Overview} $\;$ The general framework consists of a master process and $M$ simulator processes. Simulators perform simulations and return the outcomes (i.e. $V (\state)$) back to the master. All simulators communicate \emph{only} with the master and perform one simulation at a time. $M$ search trees $\{\mathcal{T}_{m}\}_{m = 1}^{M}$ are maintained to mimic $M$ distinct sets of statistics stored in existing algorithms. For example, in RootP (Figure~\ref{fig: typical parallel MCTS algs}), each of the $M$ workers maintain a search tree locally with different statistics, which can be represented by the $M$ search trees in the general algorithm framework, respectively. As illustrated by the block diagram in Figure~\ref{fig: general parallel MCTS framework}, the master performs \emph{rollouts} repeatedly to gradually build the $M$ search trees and the statistics in them.\footnote{A \emph{rollout} represents the process of executing \emph{all} steps in the block diagram illustrated in Figure~\ref{fig: general parallel MCTS framework} once, while a \emph{simulation} refers to a step in the rollout process that queries a node's value (i.e. $V (\state)$).} During this process, \emph{statistics collection} and \emph{statistics augmentation} are two crucial modules in the rollout process that make the general framework flexible enough to represent various algorithms. Specifically, statistics collection consists of the \emph{tree selection}, \emph{simulation}, and \emph{tree sync} steps, which characterize how the master employs the simulators to obtain simulation results and use them to update the $M$ search trees. Statistics augmentation includes the \emph{pseudo statistics pre-update} and \emph{backpropagation} steps, both aiming to improve node statistics in individual search trees with additional side information to achieve better exploration-exploitation tradeoff during \emph{node selection}.

We briefly go through the rollout process, where the important steps will be further discussed later. In Figure~\ref{fig: general parallel MCTS framework}, at the beginning of each rollout, a search tree $\mathcal{T}_{m}$ is selected using the function $f_{\mathrm{sel}}$ in the \emph{tree selection} step. Then, during \emph{node selection}, $\mathcal{T}_{m}$ is traversed using a modified tree policy, where a set of modified statistics ($\overline{Q}_{m}$ and $\overline{N}_{m}$) are adopted. The modified statistics are defined as follows:
    \begin{align}
        \overline{Q}_{m} (\state, \action) & := \alpha_{m} (\state, \action) Q_{m} (\state, \action) + \beta_{m} (\state, \action) \widetilde{Q}_{m} (\state, \action), \label{eq: modified value definition --} \\
        \overline{N}_{m} (\state, \action) & := N_{m} (\state, \action) + \widetilde{N}_{m} (\state, \action), \label{eq: modified visit count definition --}
    \end{align}
\noindent where $Q_{m}$ and $N_{m}$ are the original statistics used in the sequential MCTS algorithm (Eq.~(\ref{eq: UCT tree policy})); $\widetilde{Q}_{m}$ and $\widetilde{N}_{m}$ are a set of pseudo statistics that incorporate additional side information; $\alpha_{m}$ and $\beta_{m}$ controls the ratio between $Q_{m}$ and $\widetilde{Q}_{m}$. Note that Eqs.~(\ref{eq: modified value definition --}) and (\ref{eq: modified visit count definition --}) resemble Eq.~(\ref{eq: modified value and visit count definition}) in the main text. Then, after \emph{expanding} a new node in a similar manner to the sequential algorithm, the \emph{pseudo statistics pre-update} step\footnote{The statistics $O_{m} (\state, \action)$ in Figure~\ref{fig: general parallel MCTS framework} will be introduced in the ``statistics augmentation'' paragraph.} adjusts the pseudo statistics using the functions $f_{\widetilde{Q}}$ and $f_{\widetilde{N}}$. Afterwards, it assigns the \emph{simulation} task to an idle simulator. Rollouts are started over again here unless all simulators are occupied or have completed task not yet processed by the master. Otherwise, the master \emph{waits} for a completed simulation result and performs \emph{backpropagation}, which consists of the traditional update (i.e. Eqs.~(\ref{eq: recursive N and V hat update})-(\ref{eq: recursive Q update})) and a \emph{pseudo statistics post-update} step. In the post-update step, pseudo statistics are adjusted with $g_{\widetilde{Q}}$ and $g_{\widetilde{N}}$. Finally, information from different search trees are synchronized every $\tau_{\mathrm{syn}}$ rollouts, where $\tau_{\mathrm{syn}}$ is defined as the \emph{synchronization interval}.

\noindent \textbf{Statistics collection} $\;$ By choosing different $f_{\mathrm{sel}}$ and $\tau_{\mathrm{syn}}$ (in the tree selection and tree sync step, respectively), the simulators cooperate in different collaboration models that appear in various existing algorithms. Specifically, if tree sync happens in all rollout steps (i.e. $\tau_{\mathrm{syn}} \!=\! 1$), then the $M$ search trees are always identical during node selection, which can be regarded as $M$ workers performing sequential rollouts (Section~\ref{MCTS and Its Parallelization}) on a shared search tree, representing TreeP. On the other hand, when having no communication between the search trees until finishing the last rollout (i.e. $\tau_{\mathrm{syn}} \!=\! n_{\mathrm{max}}$, the total number of rollouts), and letting $f_{\mathrm{sel}}$ choose the search tree that is updated in the backpropagation step of the previous rollout (i.e. $\mathcal{T}_{\hat{m}}$), then the $M$ search trees can be regarded as search trees maintained by $M$ independent sequential MCTS algorithms, which resembles RootP.\footnotemark[1]

\begin{figure}[t]
    \centering
    \includegraphics[width=\columnwidth]{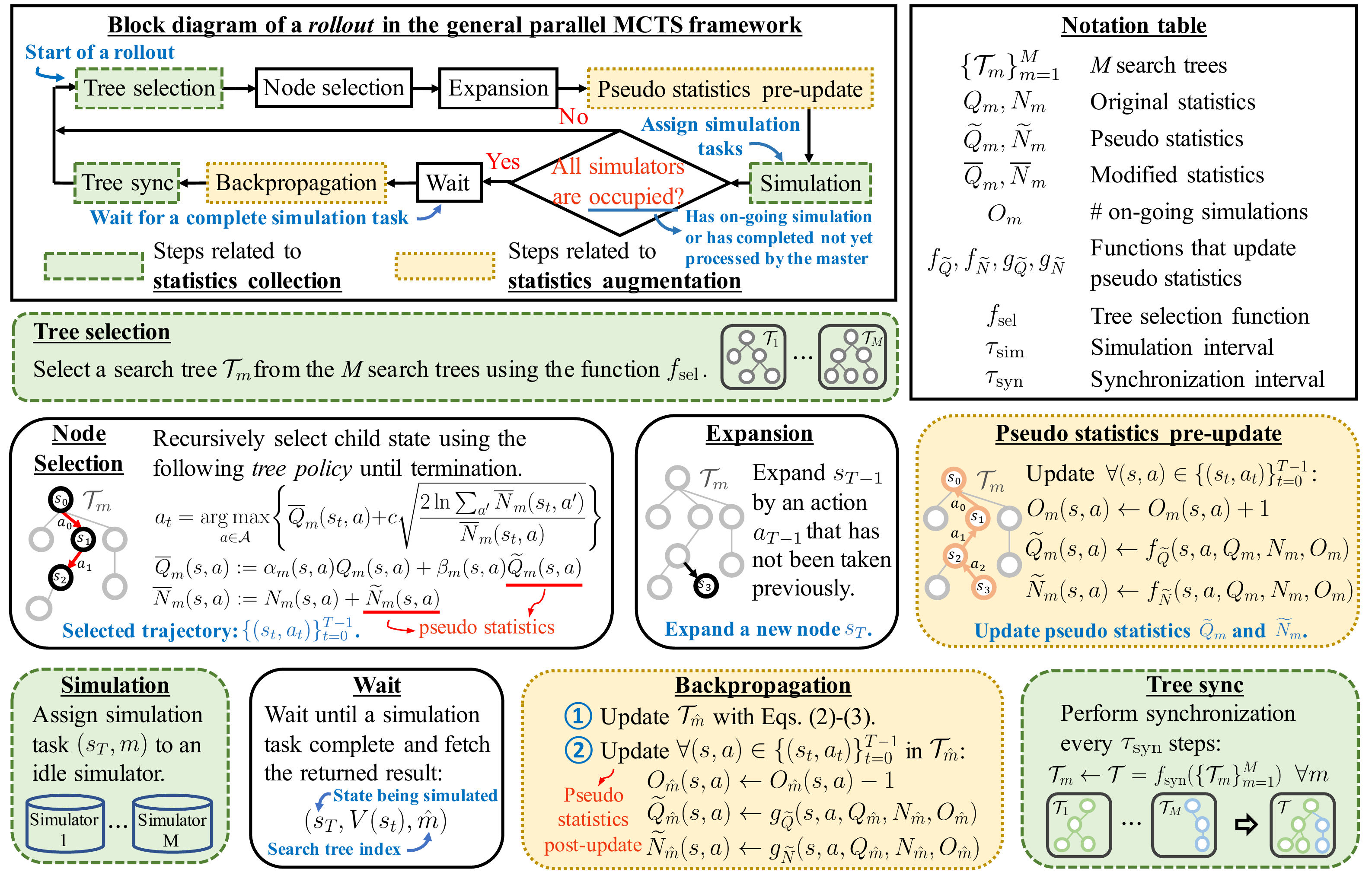}
    \vspace{-2em}
    \caption{The proposed general framework that covers existing parallel MCTS algorithms. The overall diagram is on the top-left, the notation is on the top-right, and the details are in other boxes.}
    \label{fig: general parallel MCTS framework}
    \vspace{-2.0mm}
\end{figure}

\noindent \textbf{Statistics augmentation} $\;$ Statistics augmentation focuses on using extra side information besides the simulation returns to improve the tree policy. Specifically, besides the statistics extracted from completed simulations (i.e. $Q_{m}$ and $N_{m}$), the general framework also uses pseudo statistics (i.e. $\widetilde{Q}_{m}$ and $\widetilde{N}_{m}$) to incorporate information from on-going simulations into its tree policy. Central to the pseudo statistics is the \emph{incomplete sample count} $O_{m}$ that keeps track of the number of initiated but not yet completed simulations for each node \cite{liu2020watch}. It is used to construct pseudo statistics through the \emph{pseudo statistics pre-/post-update} steps. For example, by choosing $\alpha_{m} (\state, \action) \!=\! 1$, $\beta_{m} (\state, \action) \!=\! 0$, $f_{\widetilde{N}} (\state, \action) \!=\! g_{\widetilde{N}} (\state, \action) \!=\! O_{m} (\state, \action)$, the general framework is specialized into WU-UCT.\footnotemark[1]

Finally, Table~\ref{table: parallel MCTS algorithms specification extended} summarizes how different hyperparameter choices in the general parallel MCTS framework specialize the general algorithm to various existing parallel MCTS algorithms. Please refer to Appendix~\ref{Specialization of Existing Parallel MCTS Algorithms into the General Framework} for rigorous justifications for such specializations.

\begin{table}[t]
\setlength{\textfloatsep}{1.2em}
\caption{Different choices of the parameters in the general parallel MCTS algorithm framework correspond to various existing parallel MCTS algorithms. $N_m$ and $O_m$ are abbreviation of $N_m (\state, \action)$ and $O_m (\state, \action)$, respectively. $n_{\mathrm{max}}$ is the total number of rollouts. $r_{\mathrm{VL}}$ and $n_{\mathrm{VL}}$ are hyperparameters specific to VL-UCT. $m'$ and $\hat{m}$ are the index of the search tree selected in the previous tree selection step and updated in the previous backpropagation step, respectively.}
\label{table: parallel MCTS algorithms specification extended}

\centering
{\fontsize{8}{8}\selectfont

\setlength{\tabcolsep}{2.6mm}{
\begin{tabular}{c|cc|cccc}
     \toprule
     Algorithm & $f_{\mathrm{sel}} (m', \hat{m})$ & $\syncrate$ & $\alpha_m (\state, \action)$ & $\beta_m (\state, \action)$ & $\widetilde{Q}_m (\state, \action)$ & $\widetilde{N}_m (\state, \action)$ \\
     \midrule
     UCT & $1$ & $1$ & 1 & $0$ & $0$ & $0$ \\
     \midrule
     LeafP & $(m' + 1) \% M$ & $M$ & $1$ & $0$ & $0$ & $0$ \\
     RootP & $\hat{m}$ & $n_{\mathrm{max}}$ & $1$ & $0$ & $0$ & $0$ \\
     TreeP & $\mathrm{randint} (M)$ & $1$ & $1$ & $0$ & $0$ & $0$ \\
     WU-UCT & $\mathrm{randint} (M)$ & $1$ & $1$ & $0$ & $0$ & $O_m$ \\
     VL-UCT (hard) & $\mathrm{randint} (M)$ & $1$ & $1$ & $O_m$ & $-r_{\mathrm{VL}}$ & $0$ \\
     VL-UCT (soft) & $\mathrm{randint} (M)$ & $1$ & $\frac{N_m}{N_m + n_{\mathrm{VL}} \!\cdot O_m}$ & $\frac{n_{\mathrm{VL}} \!\cdot O_m}{N_m + n_{\mathrm{VL}} \!\cdot O_m}$ & $-r_{\mathrm{VL}}$ & $n_{\mathrm{VL}} \!\cdot O_m$ \\
     \bottomrule
\end{tabular}
}
}
\vspace{-0.4em}
\end{table}

\begin{figure}[t]
\begin{algorithm}[H]
\caption{A general framework of parallel MCTS algorithms.}
\label{alg: general framework of parallel MCTS algs}
{\fontsize{10}{10} \selectfont
\begin{algorithmic}[1]

\STATE{\textbf{input:} Environment $\mathcal{M}$; number of simulator processes $M$; number of rollouts $N$; functions $\alpha_{m} (\forall m)$, $\beta_{m} (\forall m)$, $f_{\mathrm{sel}}$, $f_{\widetilde{Q}}$, $f_{\widetilde{N}}$, $g_{\widetilde{Q}}$, $g_{\widetilde{N}}$; synchronization interval $\syncrate$; initial state $\state_{0}$; maximum depth $d_{\mathrm{max}}$.\footnotemark}

\STATE{\textbf{initialize:} number of completed simulations $n_{\mathrm{complete}} \!\leftarrow\! 0$; search tree No. $m' \!\leftarrow\! M$ and $\hat{m} \!\leftarrow\! 1$; $M$ search trees $\{ \mathcal{T}_{m} \}_{m = 1}^{M}$, each with node set $\mathcal{V}_{m} \!\leftarrow\! \{ \state_{0} \}$ and edge set $\mathcal{E}_{m} \!\leftarrow\! \emptyset$:
    { \setlength{\abovedisplayskip}{0.2em}
      \setlength{\belowdisplayskip}{-0.4em}
    \begin{align*}
        \mathcal{T}_{m} := \langle (\mathcal{V}_{m}, \mathcal{E}_{m}), \{Q_{m},  \widetilde{Q}_{m}, N_{m}, \widetilde{N}_{m}, O_{m} \} \rangle,
    \end{align*}}
    
\noindent where the statistics $\{Q_{m},  \widetilde{Q}_{m}, N_{m}, \widetilde{N}_{m}, O_{m} \}$ are initialized to zero.
\vspace{1.0em}
}

\WHILE{$n_{\mathrm{complete}} < N$}

\STATE{{\bf (Tree selection)} Select a search tree $\mathcal{T}_{m}$ where $m = f_{\mathrm{sel}} (m', \hat{m}) \in \{1,\ldots, M\}$.}

\vspace{0.2em}

\STATE{{\bf (Node selection)} Traverse over $\mathcal{T}_{m}$ according to the following tree policy and collect a sequence of traversed state action pair $\{ (\state_{t}, \action_{t}) \}_{t = 0}^{T - 1}$, where $\state_{0}$ is the root node and $\state_{T - 1}$ is the state that satisfies one of the following conditions: (i) it contains unexpanded child nodes, (ii) its depth exceed $d_{\mathrm{max}}$:
    { \setlength{\abovedisplayskip}{-0.2em}
      \setlength{\belowdisplayskip}{-0.0em}
    \begin{align}
        \action_{t} = \argmax_{\action \in \actions} \left \{ \overline{Q}_{m} (\state_{t}, \action) + c \sqrt{\frac{2 \ln \sum_{\action'} \overline{N}_{m} (\state_{t}, \action')}{\overline{N}_{m} (\state_{t}, \action)}} \right \},
        \label{eq: modified UCT}
    \end{align}}
    
\noindent where the adjusted statistics $\overline{Q}_{m}$ and $\overline{N}_{m}$ are given by
    { \setlength{\abovedisplayskip}{0.2em}
      \setlength{\belowdisplayskip}{-0.4em}
    \begin{align}
        \overline{Q}_m (\state, \action) & := \alpha_m (\state, \action) Q_m (\state, \action) + \beta_m (\state, \action) \widetilde{Q}_m (\state, \action), \label{eq: adjusted Q} \\
        \overline{N}_m (\state, \action) & := N_m (\state, \action) + \widetilde{N}_m (\state, \action). \label{eq: adjusted N}
    \end{align}}
}
\STATE{{\bf (Expansion)}
Pick an expandable action $\action_{T - 1}$ at $\state_{T - 1}$ and add node $\state_{T}$ (the next state following $(\state_{T - 1}, \action_{T - 1})$) to tree $\mathcal{T}_{m}$.
}

\vspace{0.2em}

\STATE{{\bf (Pseudo statistics pre-update)} \emph{Pre-update} pseudo statistics for all $(\state, \action) \in \{ ( \state_{t}, \action_{t} ) \}_{t = 0}^{T - 1}$: 
    { \setlength{\abovedisplayskip}{0.2em}
      \setlength{\belowdisplayskip}{0.2em}
    \begin{align*}
        O_{m} (\state, \action) & \leftarrow O_{m} (\state, \action) + 1, \\
        \widetilde{Q}_{m} (\state, \action) & \leftarrow f_{\widetilde{Q}} (\state, \action, Q_{m}, N_{m}, O_{m}), \\ \widetilde{N}_{m} (\state, \action) & \leftarrow f_{\widetilde{N}} (\state, \action, Q_{m}, N_{m}, O_{m}).
    \end{align*}}
\vspace{-1.0em}
}

\STATE{{\bf (Simulation)} Assign simulation task $(\state_{T}, m)$ to a simulator process. \vspace{0.5em}}

\STATE{\textbf{if} there exist simulators without an assigned task \textbf{then continue} \vspace{0.5em}}

\STATE{{\bf (Wait)} Wait until a simulation task completes and fetch the simulation return $(\state_{T}, V (\state_{T}), \hat{m})$.}

\vspace{0.2em}

\STATE{{\bf (Backpropagation)} Update $Q_{\hat{m}}$ and $N_{\hat{m}}$ in the search tree $\mathcal{T}_{\hat{m}}$ using the same rule as Eqs.~(\ref{eq: recursive N and V hat update}) and (\ref{eq: recursive Q update}); perform \emph{pseudo-statistics post-update} on the search tree $\mathcal{T}_{\hat{m}}$ for all $(\state, \action) \in \{ ( \state_{t}, \action_{t} ) \}_{t = 0}^{T - 1}$: 
    { \setlength{\abovedisplayskip}{0.2em}
      \setlength{\belowdisplayskip}{0.2em}
    \begin{align*}
        O_{\hat{m}} (\state, \action) & \leftarrow O_{\hat{m}} (\state, \action) - 1, \\
        \widetilde{Q}_{\hat{m}} (\state, \action) & \leftarrow g_{\widetilde{Q}} (\state, \action, Q_{\hat{m}}, N_{\hat{m}}, O_{\hat{m}}), \\
        \widetilde{N}_{\hat{m}} (\state, \action) & \leftarrow g_{\widetilde{N}} (\state, \action, Q_{\hat{m}}, N_{\hat{m}}, O_{\hat{m}}).
    \end{align*}}
\vspace{-1.0em}
}

\IF{ $n_{\mathrm{complete}} \equiv \syncrate - 1 \; (\mathrm{mod} \; \syncrate)$ \vspace{0.2em}}

\STATE{{\bf (Tree sync)} Synchronize the statistics in different search trees such that: 
\vspace{-0.5em}
\begin{align*}
    \mathcal{T}_m \leftarrow \mathcal{T} = f_{\mathrm{syn}} ( \{ \mathcal{T}_{m} \}_{m = 1}^{M}) \quad m = 1, \dots, M.
\end{align*}
\vspace{-1.5em}
}

\ENDIF

\STATE{$n_{\mathrm{complete}} \leftarrow n_{\mathrm{complete}} + 1$; $\; m' \leftarrow m$ \vspace{0.4em}}

\ENDWHILE

\vspace{0.2em}

\STATE{\textbf{return $\mathcal{T} = f_{\mathrm{syn}} ( \{ \mathcal{T}_{m} \}_{m = 1}^{M} )$} (or return the ``best'' action for the initial state $\state_{0}$)}

\end{algorithmic}
}
\end{algorithm}
\vspace{-2em}
\end{figure}

\footnotetext{The functions $\alpha$, $\beta$, $f_{\mathrm{sel}}$, $f_{\widetilde{Q}}$, $f_{\widetilde{N}}$, $g_{\widetilde{Q}}$, $g_{\widetilde{N}}$ and the synchronization interval $\tau_{\mathrm{syn}}$ are hyperparameters. When set differently, the algorithm can be specialized to different parallel MCTS algorithms.}

\subsection{Additional Details of the General Algorithm Framework}
\label{Additional Details of the General Algorithm Framework}

In this subsection, we provide additional details for the general framework of parallel MCTS algorithm. Specifically, we introduce the general framework using Algorithm~\ref{alg: general framework of parallel MCTS algs}, highlighting details that are not stated clearly enough in the main text. We proceed by introducing each of the steps shown in the block diagram in Figure~\ref{fig: general parallel MCTS framework}.

\noindent \textbf{Tree selection} $\;$ The tree selection function $f_{\mathrm{sel}}$ takes $m'$ and $\hat{m}$ as input. According to Line 15, $m'$ denotes the index of the search tree selected in the previous rollout. $\hat{m}$ is the index of the search tree being updated in the backpropagation step during the previous rollout (see Lines 10 and 11).

\noindent \textbf{Node selection} $\;$ Note that the terminal conditions can be customized. Here we adopt a widely used set of terminal conditions: either the node contains unexpanded child nodes or its depth exceed $d_{\mathrm{max}}$.

\noindent \textbf{Expansion} $\;$ Identical to the expansion step in sequential MCTS.

\noindent \textbf{Pseudo statistics pre-update} $\;$ Although explicitly written here, $\widetilde{Q}_{m}$ and $\widetilde{N}_{m}$ may not need to be explicitly stored during implementation since this computation may be done during node selection.

\noindent \textbf{Simulation} $\;$ The search tree index $m$ is passed to the simulator as record. Recall that the $M$ search trees maintained by the master respectively mimic the ``search trees'' maintained by the $M$ workers in practical algorithms, the index $m$ helps Algorithm~\ref{alg: general framework of parallel MCTS algs} to mimic the activation of different ``workers''.

\noindent \textbf{Wait} $\;$ Similarly, the search tree index $\hat{m}$ is returned so that the algorithm knows which search tree to update the statistics.

\noindent \textbf{Backpropagation} $\;$ Additional to the updation of $Q_{m}$ and $N_{m}$, pseudo-statistics are also updated.

\noindent \textbf{Tree sync} $\;$ We provide a formal definition of the synchronization function $f_{\mathrm{syn}}$. Note that the following descriptions are only for rigorous purpose, practical algorithms do not need to actually implement the following algorithm.

The input of $f_{\mathrm{syn}}$ is a set of $M$ search trees $\{ \mathcal{T}_{m} \}_{m = 1}^{M}$ and the output is a synchronized search tree $\mathcal{T}$. Intuitively, $f_{\mathrm{syn}}$ performs union of the $M$ individual search trees and aggregate their newly acquired statistics after the previous synchronization (see Algorithm \ref{alg: f_syn}). Therefore, it can be divided into two steps: \emph{topology construction phase} and the \emph{statistics aggregation phase}. The topology construction phase generates a new tree topology for $\mathcal{T}$ by taking the union of the topologies from $\{\mathcal{T}_m\}_{m=1}^M$.
It can be implemented by the following steps. We begin with a search tree $\mathcal{T}$ with only one root node representing the initial state (i.e., the input $\state_{0}$ in Algorithm~\ref{alg: general framework of parallel MCTS algs}). In addition, we initialize a set of node, $\mathcal{V}_{\mathrm{syn}}$, with the root node $\state_{0}$. We then repeat the following steps until $\mathcal{V}_{\mathrm{syn}}$ is empty: (i) (randomly) take out an element $\state$ from $\mathcal{V}_{\mathrm{syn}}$ and delete it from $\mathcal{V}_{\mathrm{syn}}$, (ii) for all $\action \in \actions$, if the edge $(\state, \action)$ exists in at least one of the $M$ search trees $\{ \mathcal{T}_{m} \}_{m = 1}^{M}$, we grow the tree $\mathcal{T}$ by attaching this edge $(\state, \action)$ along with its next state $\state'$ to node $s$, and (iii) add node $\state'$ to the set $\mathcal{V}_{\mathrm{syn}}$.

To explain the \emph{statistics aggregation phase} of $f_{\mathrm{syn}}$ (i.e., the second phase), we have to introduce two sets of additional statistics associated with each edge $(\state, \action)$ at the $M$ input search trees $\{\mathcal{T}_m\}_{m=1}^M$. Specifically, for each edge $(\state, \action)$ in the search tree $\mathcal{T}_{m}$, let $\mathcal{R}_{m} (\state, \action)$ be a set that consists of elements in the following form and is constructed in a recursive manner (to be explained later):
    \begin{align}
        \mathcal{R}_{m} (\state, \action) &= \{ (V_{\state, \action}, \xi_{\state, \action}): \; V_{\state, \action} := V (\state'),\; \xi_{\state, \action} \in \{0, 1\} \}, \label{eq: the additional statistics R}
    \end{align}
\noindent where $\state'$ is the next state of $(\state, \action)$, and $V (\state')$ is recursively defined (over $\mathcal{T}_m$) according to Eq.~(\ref{eq: recursive N and V hat update}) starting from the simulation return $V (\state_{T})$.\footnote{We drop the dependency on $m$ in the above set of $\mathcal{R}_{m} (\state, \action)$ for simplicity of notation.} $\xi_{\state, \action}=1$ means that $V_{\state, \action}$ at this edge $(s,a)$ has been synchronized in the previous synchronization cycles and 0 otherwise. When an edge $(\state, \action)$ is initialized (e.g., expanded), an empty set $\mathcal{R}_{m} (\state, \action)$ will be initialized accordingly. During the \emph{backpropagation} phase of Algorithm~\ref{alg: general framework of parallel MCTS algs}, for each traversed edge corresponding to the complete simulation with return $(\state_{T}, V (\state_{T}), \hat{m})$ (assume the traversed edges are $\{ (\state_{t}, \action_{t}) \}_{t = 1}^{T - 1}$), we update the sets $\mathcal{R}_{\hat{m}} (\state_{t}, \action_{t}) \; (t = 0, \dots, T)$ by recursively computing $V (\state_{t})$ using Eq.~(\ref{eq: recursive N and V hat update}) and add the element $(V (\state_{t + 1}), 0)$ into the set $\mathcal{R}_{\hat{m}} (\state_{t}, \action_{t})$.

During the statistics aggregation phase, for each edge $(\state, \action) \in \mathcal{T}$, we perform the following steps to construct the set $\mathcal{R} (\state, \action)$: (i) initialize an empty set $\mathcal{R} (\state, \action)$, (ii) traverse all elements $( V_{\state, \action}, \xi_{\state, \action} ) \in \mathcal{R}_{1} (\state, \action)$ and add it to $\mathcal{R} (\state, \action)$ if $\xi_{\state, \action} \!=\! 1$, (iii) traverse all elements $( V_{\state, \action}, \xi_{\state, \action} ) \in \displaystyle \cup_{m = 1}^{M} \mathcal{R}_{m} (\state, \action)$\footnote{$\cup$ refers to the set union.} and add $(V_{\state, \action}, 1)$ to $\mathcal{R} (\state, \action)$ if $\xi_{\state, \action} \!=\! 0$. The intuition of the above procedure is that both the synchronized elements ($\xi_{s,a} = 1$) and elements that have not been synchronized yet ($\xi_{s,a} \!=\! 0$) are added to $\mathcal{R} (\state, \action)$ only once. We then calculate the statistics $Q$ and $N$ at the output search tree $\mathcal{T}$ as follows:
    \begin{align}
        Q (\state, \action) & := \frac{1}{| \mathcal{R} (\state, \action) |} \sum_{\langle V, \xi \rangle \in \mathcal{R} (\state, \action)} V, \label{eq: update during synchronization f_syn 1} \\
        N (\state, \action) & := | \mathcal{R} (\state, \action) |, \label{eq: update during synchronization f_syn 2}
    \end{align}
\noindent where $| \mathcal{R} (\state, \action) |$ denotes the cardinality of the set $\mathcal{R} (\state, \action)$. Finally, the synchronization of the on-going simulation count $O (\state, \action)$ is performed in the following manner: for each edge $(\state, \action) \in \mathcal{T}$,
    \begin{align}
        O (\state, \action) \leftarrow \sum_{m = 1}^{M} O_{m} (\state, \action), \label{eq: update during synchronization f_syn 3}
    \end{align}
\noindent where $O_{m} (\state, \action)$ is set to zero if this particular edge $(s,a)$ does not appear in $\mathcal{T}_m$. The details for the implementation of $f_{\mathrm{syn}}$ are summarized in Algorithm \ref{alg: f_syn}.

\begin{algorithm}[t]
\caption{The synchronization function $f_{\mathrm{syn}}$}
\label{alg: f_syn}

{\fontsize{10}{10} \selectfont
\begin{algorithmic}[1]

\STATE{\textbf{input:} $M$ search trees $\{ \mathcal{T}_{m} \}_{m = 1}^{M}$.}

\STATE{\textbf{initialize:} a search trees $\mathcal{T} := \langle (\mathcal{V}, \mathcal{E}), \{Q, N\} \rangle$, where $\mathcal{V} \leftarrow \{ \state_{0} \}$ is the set of nodes, and $\mathcal{E} \leftarrow \emptyset$ is the set of edges ($\state_{0}$ is the root node of $\mathcal{T}$).}

\vspace{1.0em}

\STATE\texttt{\# Phase 1:~Topology construction}

\STATE{ Initialize $\mathcal{V}_{\mathrm{syn}} := \{ \state_{0} \}$.}

\WHILE{$\mathcal{V}_{\mathrm{syn}}$ not empty}

\STATE{$\state \leftarrow \mathrm{pop} (\mathcal{V}_{\mathrm{syn}})$}

\FOR{$\action \in \actions$}

\IF{$(\state, \action)$ exists in at least one of the $M$ search trees $\{ \mathcal{T}_{m} \}_{m = 1}^{M}$}

\STATE{$\state' \leftarrow \mathrm{the~next~state~following~} (\state, \action)$ }

\STATE{Add edge $(\state, \action)$ and node $\state'$ to $\mathcal{T}$}

\STATE{Add $\state'$ to the set $\mathcal{V}_{\mathrm{syn}}$}

\ENDIF
    
\ENDFOR

\ENDWHILE

\vspace{1.0em}

\STATE \texttt{\# Phase 2:~Statistics aggregation}
\STATE{ For all trees $\mathcal{T}_{m}$ and edges $(\state, \action)$, define $\mathcal{R}_{m} (\state, \action)$ according to equation (\ref{eq: the additional statistics R}). $\mathcal{R}_{m} (\state, \action)$ is maintained during rollouts as described in Section~\ref{Additional Details of the General Algorithm Framework}.}


\FORALL{edges $(\state, \action)$ in $\mathcal{T}$}

\STATE{$\mathcal{R} (\state, \action) := \emptyset$}


\FORALL{$(V_{\state, \action}, \xi_{\state, \action}) \in \mathcal{R}_{1} (\state, \action)$}

\IF{$\xi_{\state, \action} = 1$}

\STATE{Add $(V_{\state, \action}, \xi_{\state, \action})$ to $\mathcal{R} (\state, \action)$}

\ENDIF

\ENDFOR


\FOR{$m = 1, \dots, M$}

\FORALL{$(V_{\state, \action}, \xi_{\state, \action}) \in \mathcal{R}_{m} (\state, \action)$}

\IF{$\xi_{\state, \action} = 0$}

\STATE{Add $(V_{\state, \action}, 1)$ to $\mathcal{R} (\state, \action)$}

\ENDIF

\ENDFOR

\ENDFOR

\STATE{Update $Q (\state, \action)$, $N (\state, \action)$, and $O (\state, \action)$ according to equations (\ref{eq: update during synchronization f_syn 1})-(\ref{eq: update during synchronization f_syn 3}), respectively.}

\ENDFOR

\vspace{1.0em}

\STATE{\textbf{return} search tree $\mathcal{T}$}

\end{algorithmic}
}

\end{algorithm}

\subsection{Specialization of the General Framework into Existing Parallel MCTS Algorithms}
\label{Specialization of Existing Parallel MCTS Algorithms into the General Framework}

In this subsection, we show how the existing algorithms introduced in Appendix~\ref{Existing Parallel MCTS Algorithms} could be viewed as special cases of Algorithm~\ref{alg: general framework of parallel MCTS algs}. Table~\ref{table: parallel MCTS algorithms specification extended} demonstrates how different choices of the hyperparameters in Algorithm~\ref{alg: general framework of parallel MCTS algs} could lead to different parallel algorithms. The functions $f_{\widetilde{Q}}$, $f_{\widetilde{N}}$, $g_{\widetilde{Q}}$, and $g_{\widetilde{N}}$ are omitted in Table~\ref{table: parallel MCTS algorithms specification extended} since they can be inferred from $\widetilde{Q}_{m}$ and $\widetilde{N}_{m}$. Note that for some methods the equivalence exists only when the simulation phase takes much more time than the other phases. Nevertheless, this holds in general \citep{liu2020watch,chaslot2008parallel} and therefore does not affect our analysis.

\noindent \textbf{LeafP} $\;$ Consider the following identification in Algorithm~\ref{alg: general framework of parallel MCTS algs}: $f_{\mathrm{sel}} (m', \hat{m}) := (m' + 1) \% M$, where $\%$ denotes the modulo operator, $\alpha_{m} (\state, \action) = 1$, and $\beta_{m} (\state, \action) = \widetilde{Q}_{m} (\state, \action) = \widetilde{N}_{m} (\state, \action) = 0$. If we further choose $\syncrate = M$, Algorithm~\ref{alg: general framework of parallel MCTS algs} will be equivalent to LeafP for the following reasons. First, since synchronization happens at time steps $\syncrate, 2 \syncrate, \dots$ (i.e., $M, 2M, \dots$), the search trees $\{ \mathcal{T}_{m} \}_{m = 1}^{M}$ are identical at the end of these time steps. We now show that the algorithm status at the ends of the rollouts $M, 2M, \dots$ in Algorithm~\ref{alg: general framework of parallel MCTS algs} is equivalent to the algorithm status of LeafP at the ends of the rollouts $1, 2, \dots$, respectively (note that in each rollout of LeafP, $M$ simulation returns of the same node is acquired).
Specifically, during the $M$ rollouts in the general framework (i.e., Algorithm~\ref{alg: general framework of parallel MCTS algs}), each search tree is selected only once due to the specific setting of $f_{\mathrm{sel}}$ (i.e., sequentially select all search trees). Since the $M$ trees are identical and the tree policy (Eq.~(\ref{eq: modified UCT})) is deterministic, each of the $M$ rollouts will independently expand and simulate one unique search tree among the $M$ trees at the \emph{same leaf node position}, which keeps all the $M$ trees having an identical topology. Finally, the synchronization step aggregates the $M$ simulation returns into a single search tree. As a result, it becomes equivalent to having $M$ workers to simulate the same node in the simulation step of LeafP. Figure~\ref{fig: general framework LeafP} illustrates the above equivalence between LeafP and the general framework under this identification.

\begin{figure}[t]
    \centering
    \subfigure[LeafP.]{
        \includegraphics[height=2.7in]{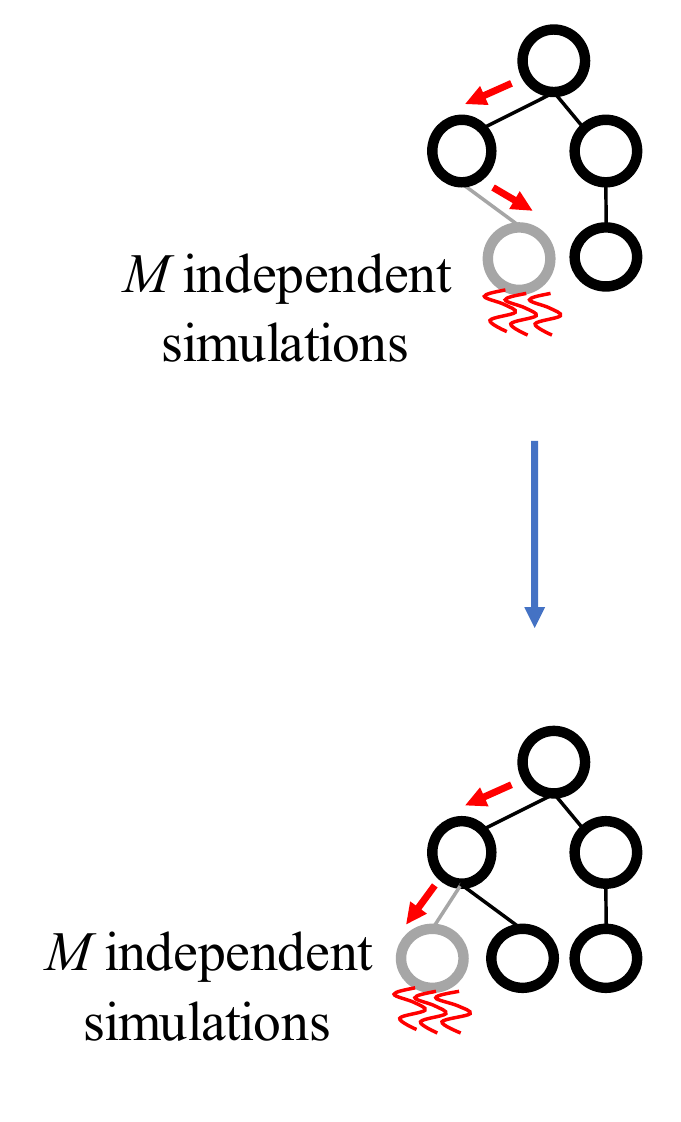}
        \label{fig: general framework LeafP A}
    }
    \hfil
    \subfigure[Equivalent LeafP by Algorithm~\ref{alg: general framework of parallel MCTS algs}.]{
        \includegraphics[height=2.7in]{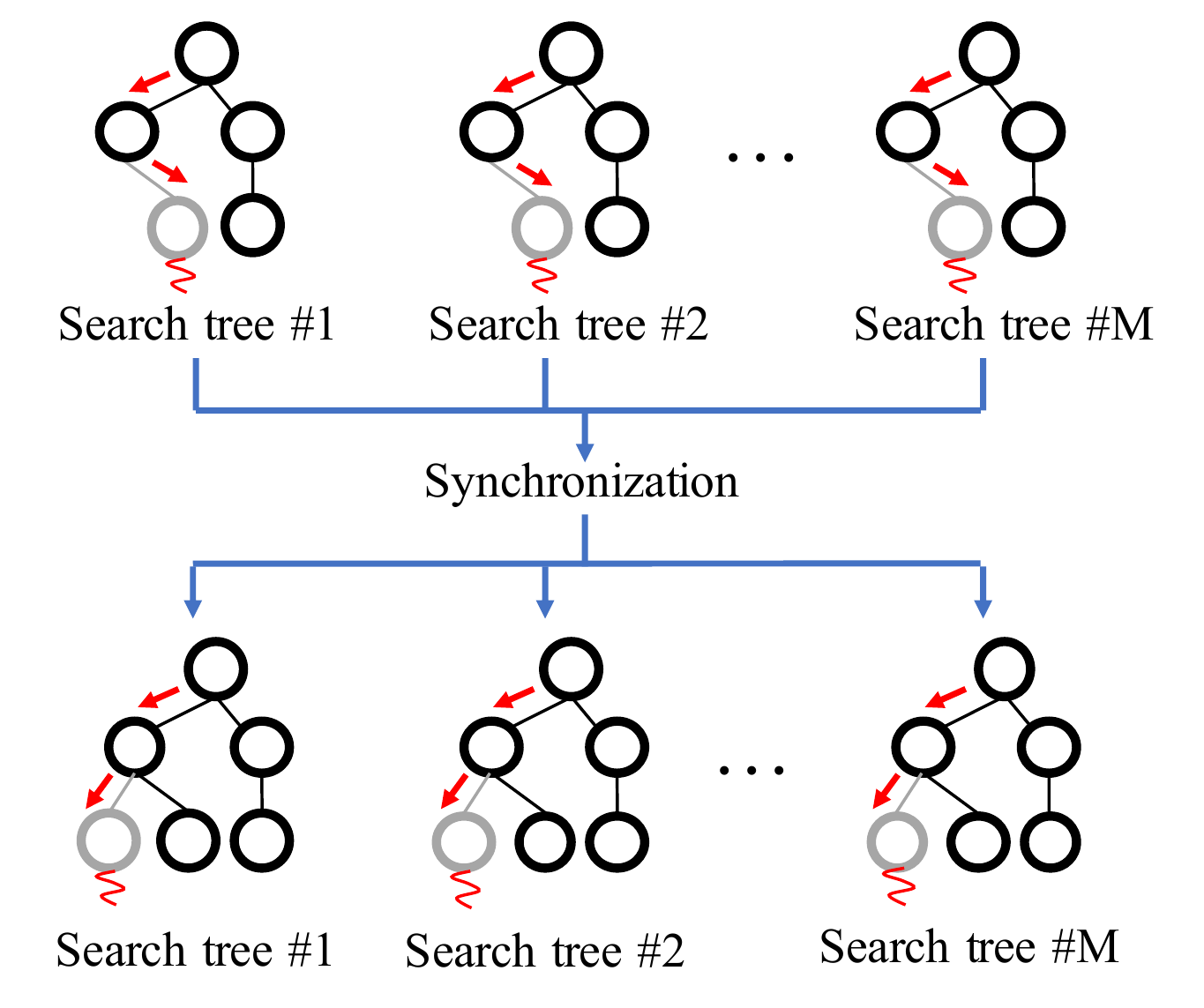}
        \label{fig: general framework LeafP B}
    }
    \vspace{-1em}
    \caption{Illustration of how LeafP can be viewed as a special case of Algorithm \ref{alg: general framework of parallel MCTS algs}. In (b), each of the $M$ trees initializes an identical simulation task to the simulator processes and synchronization happens after all $M$ simulation tasks are completed. This is analogous to (a), where $M$ workers are assigned to simulate a same node independently.}
    \label{fig: general framework LeafP}
\end{figure}

\noindent \textbf{TreeP} $\;$ Consider the choice of $\alpha_{m} (\state, \action) \!=\! 1$ and $\beta_{m} (\state, \action) \!=\! \widetilde{Q}_{m} (\state, \action) \!=\! \widetilde{N}_{m} (\state, \action) \!=\! 0$, and let the synchronization be excuted at each rollout cycle in Algorithm~\ref{alg: general framework of parallel MCTS algs} (i.e., $\syncrate = 1$, also see Table~\ref{table: parallel MCTS algorithms specification extended}).
We now show that this resembles the TreeP algorithm. First, since synchronization happens at every rollout cycle, the $M$ search trees are identical at the beginning of each rollout cycle in Algorithm~\ref{alg: general framework of parallel MCTS algs}, and can be regarded as a global search tree since all simulation returns are gathered immediately at the end of each rollout cycle (according to the definition of $f_{\mathrm{syn}}$). 
Second, the simulator processes are independent, and whenever a simulator completes, its simulation return will be updated to the global search tree (in the backpropagation phase) by the synchronization step performed at every time step, which resembles TreeP. Finally, whenever a worker is idle, the algorithm will traverse the global search tree to assign a new simulation task to it, which mimics the setting in TreeP that each worker individually perform rollouts and update the global statistics. See Figure~\ref{fig: general framework TreeP} for an illustration of the intuition for this equivalence.

\begin{figure}[t]
    \centering
    \subfigure[TreeP.]{
        \includegraphics[height=1.125in]{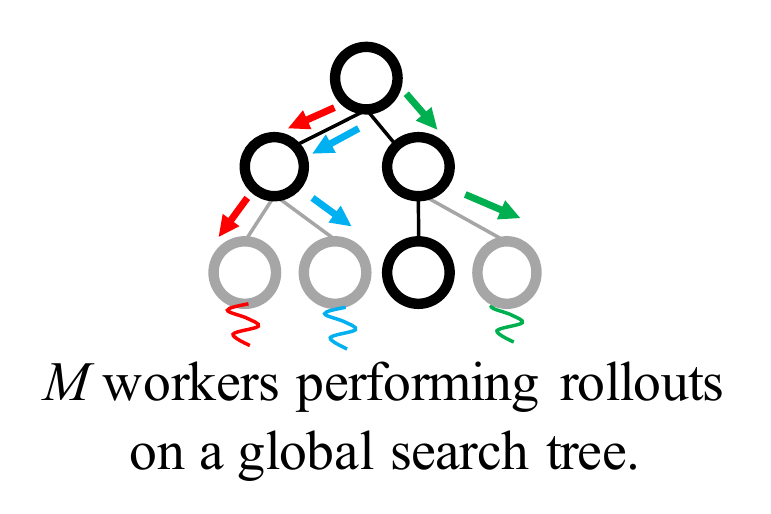}
        \label{fig: general framework TreeP A}
    }
    \hfil
    \subfigure[Equivalent TreeP by Algorithm~\ref{alg: general framework of parallel MCTS algs}.]{
        \includegraphics[height=1.125in]{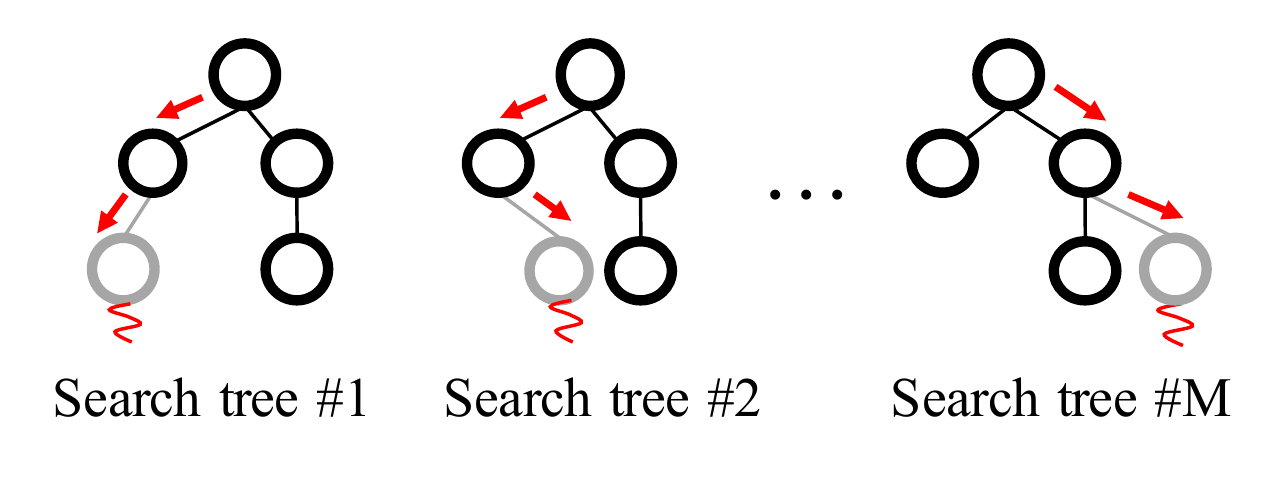}
        \label{fig: general framework TreeP B}
    }
    \caption{Illustration of how TreeP can be viewed as a special case of Algorithm \ref{alg: general framework of parallel MCTS algs}. Performing $M$ independent rollouts on $M$ search trees and then synchronizing the statistics per rollout cycle ($\syncrate = 1$) is equivalent to having $M$ workers independently performing rollouts and updating the statistics on one global search tree in TreeP. This equivalence holds in general, regardless of whether virtual loss or pseudo-statistics are used. However, without them, the vanilla TreeP normally will quicly collapse into a mode that is similar to LeafP.}
    \label{fig: general framework TreeP}
\end{figure}

\noindent \textbf{RootP} $\;$ Consider the following choice of hyperparameters: $f_{\mathrm{sel}} (m', \hat{m}) := \hat{m}$, (i.e., always select the search tree updated in the backpropagation step in the most recently completed rollout),  $\alpha_{m} (\state, \action) = 1$, $\beta_{m} (\state, \action) = \widetilde{Q}_{m} (\state, \action) = \widetilde{N}_{m} (\state, \action) = 0$, and $\syncrate = N_{\mathrm{max}}$ (i.e., synchronize after all the jobs at all the workers are totally completed). This setting is equivalent to RootP for the following reasons. First, since $\syncrate = N_{\mathrm{max}}$, all the $M$ search trees act independently (i.e., building their own search trees) and will not be aggregated by $f_{\mathrm{syn}}$ until all rollouts are completed. Second, we can show that the rollout cycles in Algorithm~\ref{alg: general framework of parallel MCTS algs} will preserve the independence of the operations at these $M$ search trees under the above identification. To see this, note that, by $f_{\mathrm{sel}} (m, \hat{m}) \!:=\! \hat{m}$, Algorithm \ref{alg: general framework of parallel MCTS algs} at the current rollout cycle will always select the search tree $\mathcal{T}_{\hat{m}}$ that has returned its simulation in the previous rollout cycle. This means that, in the current rollout cycle, Algorithm \ref{alg: general framework of parallel MCTS algs} will continue to perform rollouts and employ another worker to simulate this same search tree $\mathcal{T}_{\hat{m}}$. For this reason, it can be viewed as if we have $M$ virtual ``designated'' workers to perform rollouts and simulations for these $M$ search trees independently, which is exactly what RootP does. Since we assume other phases consume much less time than the simulation phase, these $M$ virtual ``designated'' workers are almost bound to continuously performe rollouts and simulation process without long waits. Finally, different variants of RootP (e.g., certain workers only operate on some child nodes of the search tree) can also be modeled by Algorithm \ref{alg: general framework of parallel MCTS algs} by setting $\widetilde{Q}_{m}$ at these nodes. For instance, $\widetilde{Q}_{m}$ can be chosen to be big enough such that at the root node the algorithm will always choose these same child nodes. Figure~\ref{fig: general framework RootP} illustrates the equivalence between RootP and Algorithm~\ref{alg: general framework of parallel MCTS algs} under the above identification.

\begin{figure}[t]
    \centering
    \subfigure[RootP.]{
        \includegraphics[height=0.975in]{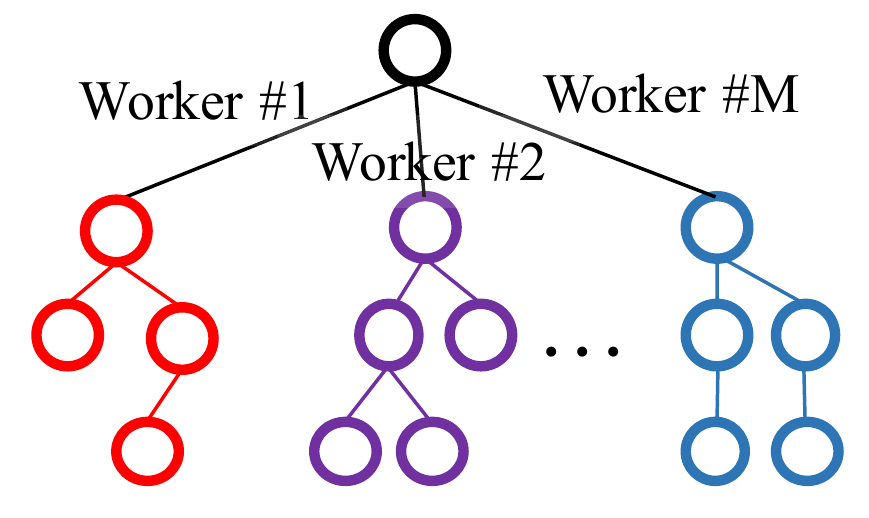}
        \label{fig: general framework RootP A}
    }
    \hfil
    \subfigure[Equivalent RootP by Algorithm~\ref{alg: general framework of parallel MCTS algs}.]{
        \includegraphics[height=1.17in]{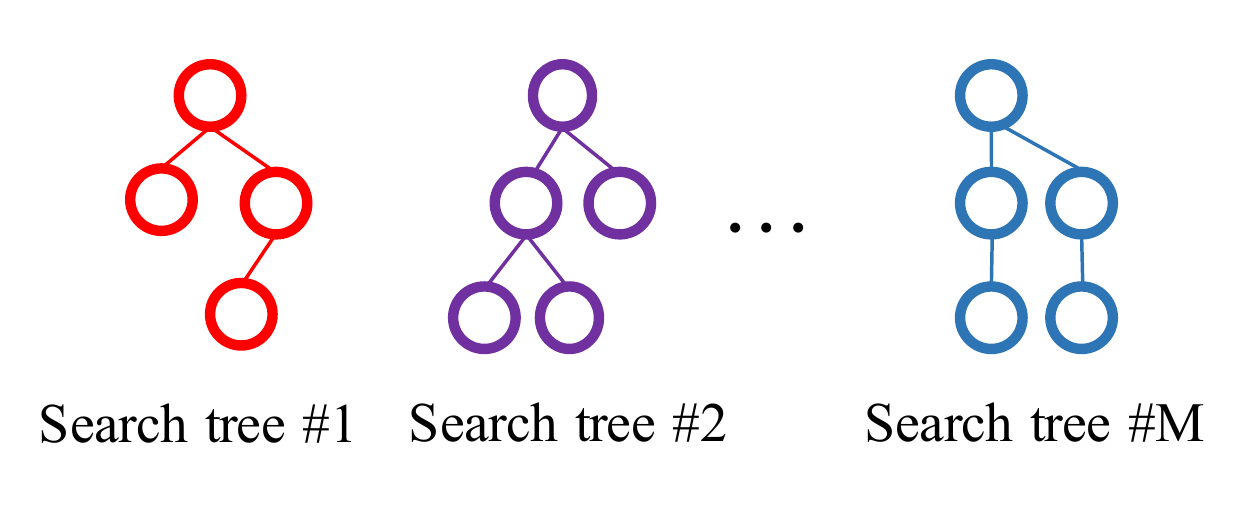}
        \label{fig: general framework RootP B}
    }
    \caption{Illustration of how RootP could be viewed as a special case of Algorithm \ref{alg: general framework of parallel MCTS algs}. Each subtree in RooP corresponds to one of the $M$ search trees in Algorithm~\ref{alg: general framework of parallel MCTS algs}. Under a particular identification, Algorithm \ref{alg: general framework of parallel MCTS algs} can be viewed as having $M$ virtual ``designated'' workers that operate independently on these $M$ search trees, which is equivalent to what RootP does.}
    \label{fig: general framework RootP}
\end{figure}

\noindent \textbf{VL-UCT} $\;$ Since it is a variant of TreeP, the workers' collaboration model in VL-UCT is identical to that of TreeP. Therefore, we can follow the same setting in $\syncrate \!=\! 1$ and $f_{\mathrm{sel}}$. On the other hand, we choose the pseudo statistics as shown in Table~\ref{table: parallel MCTS algorithms specification extended}. Specifically, for VL-UCT with hard penalty, we select (also see Table~\ref{table: parallel MCTS algorithms specification extended})
    \begin{gather*}
        \alpha_{m} (\state, \action) = 1, \quad \beta_{m} (\state, \action) = O_{m} (\state, \action), \\
        \widetilde{Q}_m (\state, \action) = - r_{\mathrm{VL}}, \quad \widetilde{N}_m (\state, \action) = 0.
    \end{gather*}
And for VL-UCT with soft penalty, we choose 
    \begin{gather*}
        \alpha_{m} (\state, \action) = \frac{N_{m} (\state, \action)}{N_{m} (\state, \action) + n_{\mathrm{VL}} \cdot O_{m} (\state, \action)}, \\
        \beta_{m} (\state, \action) = \frac{n_{\mathrm{VL}} \cdot O_{m} (\state, \action)}{N_{m} (\state, \action) + n_{\mathrm{VL}} \cdot O_{m} (\state, \action)}, \\
        \widetilde{Q}_{m} (\state, \action) = - r_{\mathrm{VL}}, \\
        \widetilde{N}_{m} (\state, \action) = n_{\mathrm{VL}} \cdot O_{m} (\state, \action).
    \end{gather*}

\noindent \textbf{WU-UCT} $\;$ Although not exactly based on TreeP, WU-UCT follows the same master-worker architecture as in Algorithm~\ref{alg: general framework of parallel MCTS algs}. We now show that WU-UCT can also be viewed as a special case of Algorithm \ref{alg: general framework of parallel MCTS algs} under the identification to be explained below. Similar to TreeP, we set $\syncrate=1$, i.e., the statistics from the $M$ search trees are synchronized at the end of each rollout cycle. Likewise, we set $f_{\mathrm{sel}}(m',\hat{m})=\mathrm{randint}(M)$; that is, it selects a random search tree in the selection phase.\footnote{In the original paper, WU-UCT also parallelizes the expansion step. However, since we assume the simulation phase is much more time-consuming then other phases, we ignore this detail.} In addition, we make the following choices (see Table~\ref{table: parallel MCTS algorithms specification extended})
    \begin{gather*}
        \alpha_{m} (\state, \action) = 1, \quad \beta_{m} (\state, \action) = \widetilde{Q}_{m} (\state, \action) = 0, \\
        \widetilde{N}_{m} (\state, \action) = O_{m} (\state, \action).
    \end{gather*}

\section{Proofs: Parallel Algorithms for Monte Carlo Tree Search}
\label{Appendix:MCTS}

This section provides proofs for Theorems~\ref{theorem: necessary condition for a good algorithm in the tree search setting}, and \ref{theorem: regret upper bound of WU-UCT in stochastic bandit setup}, which locate in Sections~\ref{The Necessary Conditions} and \ref{Theoretical Justification of WU-UCT}, respectively.

\subsection{The Necessary Conditions}
\label{The Necessary Conditions}

To help elaboration, we first introduce the following additional definitions. Define $\mathbb{A}_{\mathrm{seq}}$ as the sequential MCTS algorithm introduced in Section~\ref{MCTS and Its Parallelization} \citep{kocsis2006improved}. $\mathcal{T}^{\mathbb{A}}_{\state, n}$ is defined as the search tree with root node $\state$ and is constructed by a (parallel) MCTS algorithm $\mathbb{A}$ with $n$ rollouts. Whenever it is clear from context, we omit the subscript $n$ for notation simplicity. Let $V^{\mathbb{A}}_{\state, n} (\state')$ be the cumulative reward $V (\state')$ obtained in the backpropagation step (i.e. computed by Eq.~(\ref{eq: recursive N and V hat update})) when performing a rollout using algorithm $\mathbb{A}$ on the search tree $\mathcal{T}^{\mathbb{A}}_{\state, n}$ (if $\state'$ is not selected during the rollout, $V^{\mathbb{A}}_{\state, n} (\state') \!:=\! 0$). Note that $V^{\mathbb{A}}_{\state, n} (\state')$ is indeed a random variable due to the stochasticity in the simulation returns.

Following the above definitions as well as the terminology in the general algorithm framework (Appendix~\ref{Formal Introduction of The General Algorithm Framework}), we give a formal version of Theorem~\ref{theorem: necessary condition for a good algorithm in the tree search setting}.

\begin{theorem}[A formal version of Theorem~\ref{theorem: necessary condition for a good algorithm in the tree search setting}]
\label{theorem: necessary condition for a good algorithm in the tree search setting (formal)}
Consider an algorithm $\mathbb{A}$ that is specified from the general parallel MCTS framework by choosing $\widetilde{N}_m (\state, \action) \!=\! f( O_{m} (\state, \action) )$ ($m=1,\ldots,M$), where $f (\cdot): \mathbb{Z}^{+}_{0} \!\rightarrow\! \mathbb{R}$ is a function. If there exists an edge $(\state, \action)$ in any of the $M$ search trees $\{\mathcal{T}_{m}\}$ such that algorithm $\mathbb{A}$ violates any of the following conditions (with $\state'$ defined as the next state following $(\state, \action)$): \vspace{0.1em} \newline
$\bullet \;$ \textbf{Necessary cond. of $\overline{Q}$:} $\expectation [ \overline{Q}_{m} (\state, \action) ] \!=\! \frac{1}{n}\sum_{n' = 1}^{n} \expectation [ V_{\state', n'}^{\mathbb{A}_{\mathrm{seq}}} (\state') ]$ ($n \!=\! N_{m} (\state, \action) \!+\! O_{m} (\state, \action)$), \eqnum\label{eq: necessary condition 1 --} 

\vspace{-0.6em}

$\bullet \;$ \textbf{Necessary cond. of $\overline{N}$:} $f(x) \geq x \; (\forall x \in \mathbb{Z}^{+}_{0})$, \eqnum\label{eq: necessary condition 2 --}

\vspace{-0.3em}

\noindent then there exists an MDP $\mathcal{M}$ such that the excess regret of running $\mathbb{A}$ on MDP $\mathcal{M}$ does not vanish.
\end{theorem}

In the following, we provide the formal proof of Theorem~\ref{theorem: necessary condition for a good algorithm in the tree search setting (formal)}, which states two necessary conditions for having vanishing excess regret in parallel MCTS algorithms. Before delving into the proof, we use Figure~\ref{fig: mini MAB} to introduce the concept of \emph{mini-MAB}. Specifically, in a search tree, each node and its child nodes represent a two-layer search tree that resembles a MAB with the same number of children. We define this two-layer search tree as mini-MAB. Note that one core difference between mini-MABs and MABs is that the reward acquired by a child node of mini-MABs are rewards obtained from a sub-tree rooted at the child node (Figure~\ref{fig: mini MAB (B)}), while for MAB all child nodes produce i.i.d. rewards following a pre-defined distribution.

\begin{figure}
    \centering
    
    \subfigure[MAB with $K = 3$.]{
        \includegraphics[height=1.0in]{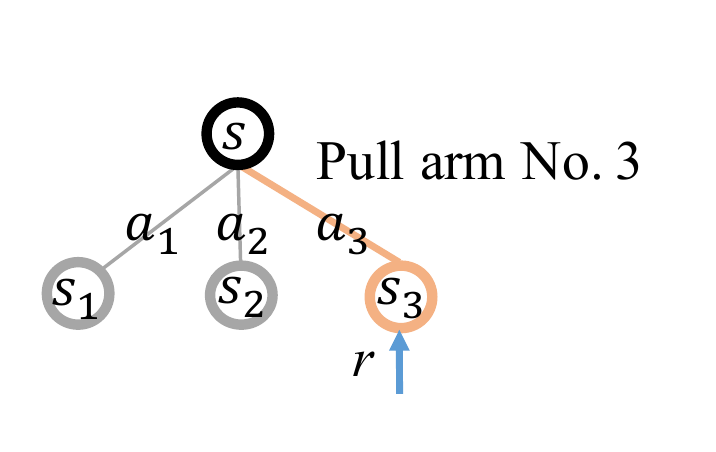}
        \label{fig: mini MAB (A)}
    }
    \subfigure[Equivalent mini-MAB in a search tree.]{
        \includegraphics[height=1.0in]{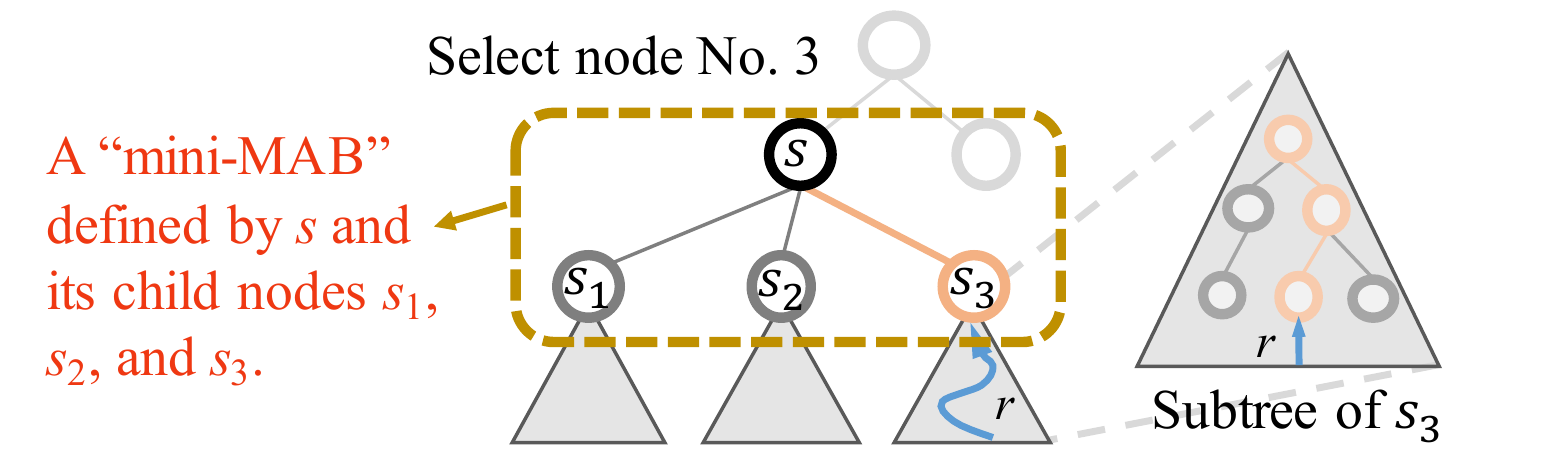}
        \label{fig: mini MAB (B)}
    }
    
    \caption{Demonstration of the mini-MABs in MCTS search trees that resembles a multi-armed bandit (MAB). (a): a MAB with three arms. (b): $\state$, $\state_{1}$, $\state_{2}$, and $\state_{3}$ define a \emph{mini-MAB} that resembles the MAB in (a).}
    \label{fig: mini MAB}
\end{figure}

\begin{proof}[Proof of Theorem~\ref{theorem: necessary condition for a good algorithm in the tree search setting} (Theorem~\ref{theorem: necessary condition for a good algorithm in the tree search setting (formal)}]
To obtain vanishing excess regret, it is necessary to show the following: \emph{the excess regret of the mini-MABs that represent nodes on the optimal path in the search tree should decrease as $t$ increases.} This necessary condition holds since all nodes on the optimal path will be visited $\Omega ( t )$ times when $t$ is sufficiently large (see \citet{kocsis2006improved}), and if any of the nodes have nonvanishing excess regret, the tree search algorithm will suffer from nonvanishing regret. In the following, we derive the necessary conditions for algorithms that have vanishing excess regret in \emph{mini-MABs}.

Consider a \emph{mini-MAB} whose root node is $\state$ (assume it is on the optimal path). The actions are defined as $\{ \action_{k} \}_{k = 1}^{K}$ and the next state following $(\state, \action_{k})$ is defined as $\state_{k}$. 
In order to achieve vanishing excess regret during parallel, it is necessary to have vanishing excess regret when this mini-MAB is parallelized while rollouts its child nodes are performed sequentially. That is, assume we use the sequential algorithm $\mathbb{A}_{\mathrm{seq}}$ to produce reward for all child nodes of the mini-MAB rooted at $\state$. Correspondingly, we define $\mu_{k, n}$ as the expected reward obtained by executing action $\action_{k}$ for the $n$th time. That is, 
    \begin{align}
        \mu_{k, n} := \frac{1}{n} \sum_{n' = 1}^{n} \expectation [ V^{\mathbb{A}_{\mathrm{seq}}}_{\tau_{\state, \action_{k}} (n')} (\state') ]. \label{eq: mu_k,n}
    \end{align}
Similarly, define $\mu^{*}_{n} := \max_{k} \mu_{k, n}$.

Define $\overline{Q}_{k, n, o}$ as the estimated value of the $k$th child node of the \emph{mini-MAB} when there are $n$ initialized simulations and $o$ on-going simulations (which means that there are $n - o$ completed simulations). Formally, $\overline{Q}_{k, n, o}$ can be written as (reflects Eq.~(\ref{eq: modified visit count definition --})/Eq.~(\ref{eq: modified value and visit count definition}) in the general framework)
    \begin{align*}
        \overline{Q}_{k, n, o} := \alpha_{k} Q_{k, n - o} + \beta_{k} \widetilde{Q}_{k, n, o}, \numberthis \label{eq: definition of the new pseudo value}
    \end{align*}
\noindent where $\alpha_{k} := \alpha (\state, \action_{k})$ and $\beta_{k} := \beta (\state, \action_{k})$ (assume $\state$ as the root node of the \emph{mini-MAB}); $\widetilde{Q}_{k, n, o} := \widetilde{Q} (\state, \action_{k})$ is the pseudo value. Note that $\alpha_{k}$ and $\beta_{k}$ might also depend on $n$ and $o$.

We define $T_{k} (t)$ as the number of times action $\action_{k}$ is selected in the first $t$ rollouts. According to the regret decomposition identity \cite{mazumdar2017multi}, $\cumulateregret{t}$ can be decomposed with respect to different arms:
    \begin{align*}
        \cumulateregret{t} = \sum_{k \in \{1, \dots, K\}, k \neq k^{*}} \Delta_{k} \expectation \left [ T_{k} (t) \right ]
    \end{align*}
\noindent where $\Delta_{k} := \max_{k'} \expectation [ Q_{t} (\state, \action_{k'}) ] - \expectation [ Q_{t} (\state, \action_{k}) ]$ is the expected regret of selecting action $\action_{k}$ instead of the best action in the mini-MAB. Therefore, to achieve vanishing excess regret, it is necessary to show that the number of times a suboptimal action is chosen (i.e. $\expectation [T_{k} (t)]$) for a parallel MCTS algorithm should be the number of times such action is taken in sequential MCTS plus a term that vanishes as $t$ goes to infinity.

Define $\overline{e}_{t, n, o} := \sqrt{(2 \ln t) / (n + f (o))}$, where $f(\cdot)$ is defined in Theorem~\ref{theorem: necessary condition for a good algorithm in the tree search setting (formal)}. We lower bound $T_{k} (t)$ by ($k^{*}$ is the index of the optimal action)
    \begin{align*}
        T_{k} (t) & = \sum_{\tau = 1}^{t} [ \text{Action}_{\tau} = \action_{k} ] \\
        & \geq \sum_{\tau = 1}^{t} \min_{n, n' \in [0, \tau]; o, o' \in [0, M - 1]}  \mathbbm{1} \Bigg [ \overline{Q}_{k^{*}, n, o} \!\!+\! \overline{e}_{\tau, n, o} \!\!\leq\! \overline{Q}_{k, n', o'} \!\!+\! \overline{e}_{\tau, n', o'} \Bigg ]
    \end{align*}
We lower bound the probability of the event $\overline{Q}_{k^{*}, n, o} \!\!+\! \overline{e}_{\tau, n, o} \!\!\leq\! \overline{Q}_{k, n', o'} \!\!+\! \overline{e}_{\tau, n', o'}$ using one minus the sum of the probability of the two following events:
    \begin{gather}
        \overline{Q}_{k^{*}, n, o} \geq \mu^{*}_{n} - \overline{e}_{\tau, n, o}, \label{eq: tree with pseudo statistics term 1} \\
        \overline{Q}_{k, n', o'} \leq \mu_{k, n'} + \overline{e}_{\tau, n', o'}, \label{eq: tree with pseudo statistics term 2}
    \end{gather}
\noindent where $\mu_{k, n} := \frac{1}{n} \expectation [\sum_{t = 1}^{n} Q_{t} (\state, \action_{k})]$ is defined as the average value return of the first $n$ times $\action_{k}$ is taken; $\mu^{*}_{n}$ denotes the same quantity defined for the optimal action $\action^{*} := \argmax_{\action'} \expectation [\sum_{t = 1}^{n} Q_{t} (\state, \action')]$. Note that this bound (i.e. Eqs.~(\ref{eq: tree with pseudo statistics term 1}) and (\ref{eq: tree with pseudo statistics term 2})) holds since by definition $\mu^{*}_{n - o} > \mu_{k, n - o}$. 

We first give an outline of the proof regarding the necessary condition of $\overline{Q}$. We shall first show that if $\expectation [ \overline{Q}_{k^{*}, n, O_{k^{*}, t}} ] \geq \mu_{n}^{*}$ and $\expectation [ \overline{Q}_{k, n', O_{k, t}} ] \leq \mu_{k, n'}$ are not satisfied, there exists a mini-MAB task, $n_{0}$, and $p_{\epsilon} \in (0, 1)$ such that for any $n > n_{0}$, the probability of both Eq.~(\ref{eq: tree with pseudo statistics term 1}) and Eq.~(\ref{eq: tree with pseudo statistics term 2}) are smaller than $p_{\epsilon}$. Hence $T_{k} (t)$ will be lower bounded by $(1 - 2 p_{\epsilon}) \cdot t$, meaning that the suboptimal arm $k$ will be pulled $\Omega (t)$ times. Therefore, the algorithm cannot achieve vanishing excess regret. Next, given that (i) $\expectation [ \overline{Q}_{k^{*}, n, O_{k^{*}, t}} ] \geq \mu_{n}^{*}$ and $\expectation [ \overline{Q}_{k, n', O_{k, t}} ] \leq \mu_{k, n'}$ should be satisfied, and (ii) the algorithm does not know which arm is optimal (i.e., it cannot distinguish between $k$ and $k^{*}$), the algorithm has to satisfy $\expectation [ \overline{Q}_{k^{*}, n, O_{k^{*}, t}} ] = \mu_{n}^{*}$ and $\expectation [ \overline{Q}_{k, n', O_{k, t}} ] = \mu_{k, n'}$, which gives the necessary condition of $\overline{Q}$. Details are provided as follows.

Define $\mu_{k, n_{1}, n_{2}}$ as the average reward of the $k$th arm of the \emph{mini-MAB} from its $n_{1}$th rollout to its $n_{2}$th rollout ($n_{1} \leq n_{2}$). Similarly $\mu^{*}_{n_{1}, n_{2}}$ defines the same quantity for the optimal arm $k^{*}$. We have the following results (for any integer $0 \leq o \leq n$):
    \begin{align*}
        \mu^{*}_{n} & = \frac{n - o}{n} \mu^{*}_{n - o} + \frac{o}{n} \mu^{*}_{n - o + 1, n}, \numberthis \label{eq: decomposition of the expected values 1} \\
        \mu_{k, n} & = \frac{n - o}{n} \mu_{k, n - o} + \frac{o}{n} \mu_{k, n - o + 1, n}. \numberthis \label{eq: decomposition of the expected values 2}
    \end{align*}
Using the above results, Eqs.~(\ref{eq: tree with pseudo statistics term 1}) and (\ref{eq: tree with pseudo statistics term 2}) can be equivalently written as
    \begin{gather*}
        \alpha_{k^{*}} \left ( Q_{k^{*}, n - o} - \mu^{*}_{n - o} \right ) + \beta_{k^{*}} \widetilde{Q}_{k^{*}, n, o} + \alpha_{k^{*}} \mu^{*}_{n - o} - \mu^{*}_{n} \geq - \overline{e}_{\tau, n, o}, \numberthis \label{eq: tree with pseudo statistics term 4} \\
        \alpha_{k} \left ( Q_{k, n' - o'} - \mu_{k, n' - o'} \right ) + \beta_{k} \widetilde{Q}_{k, n', o'} + \alpha_{k} \mu_{k, n' - o'} - \mu_{k, n'} \leq \overline{e}_{\tau, n', o'}, \numberthis \label{eq: tree with pseudo statistics term 5}
    \end{gather*}
\noindent where $\phi_{k}$ is a variable that depend on $k$, $n$, and $o$. By definition, we have $\expectation [ Q_{k^{*}, n - o} ] = \mu^{*}_{n - o}$ and $\expectation [ Q_{k, n' - o'} ] = \mu_{k, n' - o'}$. We then focus on the following terms in the above equations:
    \begin{gather*}
        \beta_{k^{*}} \widetilde{Q}_{k^{*}, n, o} \!+\! \alpha_{k^{*}} \mu^{*}_{n \!-\! o} \!\!-\!\! \mu^{*}_{n}, \numberthis \label{eq: tree necessary condition result term 1} \\
        \beta_{k} \widetilde{Q}_{k, n'\!, o'} \!+\! \alpha_{k} \mu_{k, n' \!-\! o'} \!-\! \mu_{k, n'}. \numberthis \label{eq: tree necessary condition result term 2}
    \end{gather*}
We show that vanishing excess regret cannot be achieved unless Eqs.~(\ref{eq: tree necessary condition result term 1}) and (\ref{eq: tree necessary condition result term 2}) have $\leq 0$ and $\geq 0$ expectation value, respectively. Otherwise, there exists $n_{0}$ such that for any $n > n_{0}$ and $o < M$ (by definition), $\overline{e}_{\tau, n, o}$ and $\overline{e}_{\tau, n', o'}$ will have smaller absolute value than Eqs.~(\ref{eq: tree necessary condition result term 1}) and (\ref{eq: tree necessary condition result term 2}), respectively. For Eq.~(\ref{eq: tree with pseudo statistics term 5}), this means that when $n > n_{0}$ its left-hand side has higher expectation value than its right-hand side, which means there exists $p_{\epsilon} \in (0, 1)$ such that the probability of Eq.~(\ref{eq: tree with pseudo statistics term 5}) is smaller than $p_{\epsilon}$. This argument similarly applies to Eq.~(\ref{eq: tree with pseudo statistics term 4}). As mentioned before, Eqs.~(\ref{eq: tree with pseudo statistics term 4}) and (\ref{eq: tree with pseudo statistics term 5}) have probability upper bound means the suboptimal arm $k$ will be pulled $\Omega (t)$ times, which makes the parallel MCTS algorithm fail to achieve vanishing excess regret.

Given that the parallel MCTS algorithm belongs to the general framework (Algorithm~\ref{alg: general framework of parallel MCTS algs}), there are three types of pseudo statistics that can be added to $\widetilde{Q}$, which are (i) statistics related to all complete simulations, (ii) statistics related to all incomplete simulations, and (iii) statistics non-related to simulation returns. Given this, we decompose the pseudo value $\widetilde{Q}_{k, n, o}$ into three terms:
    \begin{align*}
        \widetilde{Q}_{k, n, o} := \widetilde{Q}^{\mu_{k, n - o}}_{k, n, o} + \phi_{k} \cdot \widetilde{Q}^{\mu_{k, n - o + 1, n}}_{k, n, o} + \widetilde{Q}^{R}_{k, n, o},
    \end{align*}
\noindent where $\expectation [ \widetilde{Q}^{\mu_{k, n - o}}_{k, n, o} ] = \mu_{k, n - o}$, $\expectation [ \widetilde{Q}^{\mu_{k, n - o + 1, n}}_{k, n, o} ] = \mu_{k, n - o + 1, n}$, and $\widetilde{Q}^{R}_{k, n, o}$ is independent of both $\mu_{k, n - o}$ and $\mu_{k, n - o + 1, n}$. For the optimal arm $k^{*}$ (the following holds for other arms as well), we have
    \begin{align*}
        & \beta_{k^{*}} \widetilde{Q}_{k^{*}, n, o} + \alpha_{k^{*}} \mu^{*}_{n - o} - \mu^{*}_{n} \\
        \overset{(a)}{=} & \beta_{k^{*}} ( \widetilde{Q}^{\mu_{k^{*}, n - o}}_{k^{*}, n, o} \!\!+\! \phi_{k} \cdot \widetilde{Q}^{\mu_{k^{*}, n - o + 1, n}}_{k^{*}, n, o} + \widetilde{Q}^{R}_{k^{*}, n, o}) \!+\! (\alpha_{k^{*}} - \frac{n - o}{n}) \mu^{*}_{n - o} - \frac{o}{n} \mu^{*}_{n - o + 1, n} \\
        =\! & \left ( \beta_{k^{*}} \widetilde{Q}^{\mu_{k^{*}, n - o}}_{k^{*}, n, o} \!+\! (\alpha_{k^{*}} \!-\! \frac{n - o}{n}) \mu^{*}_{n - o} \right ) \!+\! \left ( \beta_{k^{*}} \phi_{k^{*}} \widetilde{Q}^{\mu_{k^{*}, n - o + 1, n}}_{k^{*}, n, o} \!\!-\! \frac{o}{n} \mu^{*}_{n - o + 1, n} \right ) \!+\! \widetilde{Q}^{R}_{k^{*}, n, o}, \numberthis \label{eq: three terms for proving the necessary conditions on tree search}
    \end{align*}
\noindent where $(a)$ uses the result of Eq.~(\ref{eq: decomposition of the expected values 1}). The expectation value of Eq.~(\ref{eq: three terms for proving the necessary conditions on tree search}) is
    \begin{align*}
        & \left ( \beta_{k^{*}} + \alpha_{k^{*}} - \frac{n - o}{n} \right) \mu^{*}_{n - o} + \left ( \beta_{k^{*}} \phi_{k^{*}} - \frac{o}{n} \right ) \mu^{*}_{n - o + 1, n} + \expectation [ \widetilde{Q}^{R}_{k^{*}, n, o} ]. \numberthis \label{eq: three terms for proving the necessary conditions on tree search expectation}
    \end{align*}
Similarly, for arm $k$ we have
    \begin{align*}
        & \beta_{k} \widetilde{Q}_{k, n', o'} + \alpha_{k} \mu_{k, n' - o'} - \mu_{k, n'} \\
        \overset{(a)}{=} & \beta_{k} ( \widetilde{Q}^{\mu_{k, n' - o'}}_{k, n', o'} \!\!+\! \phi_{k} \cdot \widetilde{Q}^{\mu_{k, n' - o' + 1, n'}}_{k, n', o'} + \widetilde{Q}^{R}_{k, n', o'}) \!+\! (\alpha_{k} - \frac{n' - o'}{n'}) \mu_{k, n' - o'} - \frac{o'}{n'} \mu_{k, n' - o' + 1, n'} \\
        =\! & \left ( \beta_{k} \widetilde{Q}^{\mu_{k, n' - o'}}_{k, n', o'} \!+\! (\alpha_{k} \!-\! \frac{n' - o'}{n'}) \mu_{k, n' - o'} \right ) \!+\! \left ( \beta_{k} \phi_{k} \widetilde{Q}^{\mu_{k, n' - o' + 1, n'}}_{k, n', o'} \!\!-\! \frac{o'}{n'} \mu_{k, n' - o' + 1, n'} \right ) \!+\! \widetilde{Q}^{R}_{k, n', o'}, \numberthis \label{eq: three terms for proving the necessary conditions on tree search 2}
    \end{align*}
and its expectation value is 
    \begin{align*}
        & \left ( \beta_{k} + \alpha_{k} - \frac{n - o}{n} \right) \mu_{k, n - o} + \left ( \beta_{k} \phi_{k} - \frac{o}{n} \right ) \mu_{k, n - o + 1, n} + \expectation [ \widetilde{Q}^{R}_{k, n, o} ]. \numberthis \label{eq: three terms for proving the necessary conditions on tree search expectation 2}
    \end{align*}

We now argue that in order for the \emph{mini-MAB} to achieve vanishing excess regret, the expectation of the three terms in the above equation should all be 0. Specifically, previously we have shown that Eqs.~(\ref{eq: tree necessary condition result term 1}) and (\ref{eq: tree necessary condition result term 2}) should have $\leq 0$ and $\geq 0$ expected values. However, as suggested by Eqs.~(\ref{eq: three terms for proving the necessary conditions on tree search expectation}) and (\ref{eq: three terms for proving the necessary conditions on tree search expectation 2}), since the algorithm does not know which arm is optimal, and both equations have the same form, it is impossible to have Eq.~(\ref{eq: three terms for proving the necessary conditions on tree search expectation}) $< 0$ while Eq.~(\ref{eq: three terms for proving the necessary conditions on tree search expectation 2}) $> 0$. Hence, both equations should be equal to zero. Since $\mu^{*}_{n - o}$ and $\mu^{*}_{n - o + 1, n}$ ($\mu_{k, n - o}$ and $\mu_{k, n - o + 1, n}$) are task-specific, to make Eqs.~(\ref{eq: three terms for proving the necessary conditions on tree search expectation}) and (\ref{eq: three terms for proving the necessary conditions on tree search expectation 2}) equals to zero, we should have ($\forall k$):
    \begin{gather*}
        \beta_{k} + \alpha_{k} - \frac{n - o}{n} = 0, \\
        \beta_{k} \phi_{k} - \frac{o}{n} = 0, \\
        \expectation [ \widetilde{Q}^{R}_{k, n, o} ] = 0.
    \end{gather*}
Plug in the above results into Eq.~(\ref{eq: definition of the new pseudo value}), we conclude that one necessary condition for having vanishing excess regret in the \emph{mini-MAB} is ($\forall k$)
    \begin{align*}
        \expectation [ \overline{Q}_{k, n, o} ] \!&=\! \alpha_{k} \mu_{k, n - o} + \beta_{k} \Big ( \widetilde{Q}^{\mu_{k, n - o}}_{k, n, o} + \phi_{k} \widetilde{Q}^{\mu_{k, n - o + 1, n}}_{k, n, o} + \widetilde{Q}^{R}_{k, n, o} \Big ) \\
        & = (\alpha_{k} + \beta_{k}) \mu_{k, n \!-\! o} + \beta_{k} \phi_{k} \mu_{k, n - o + 1, n} + \expectation [ \widetilde{Q}^{R}_{k, n, o} ] \\
        & = \frac{n - o}{n} \mu_{k, n \!-\! o} + \frac{o}{n} \mu_{k, n - o + 1, n} \\
        & \overset{(a)}{=} \mu_{k, n},
    \end{align*}
where $(a)$ uses the result in Eq.~(\ref{eq: decomposition of the expected values 1}).

According to the definition of $\mu_{k, n}$, the necessary condition for having vanishing excess regret is
    \begin{align*}
        \expectation [ \overline{Q}_{k, n, o'} ] = \mu_{k, n} = \frac{1}{n} \sum_{n' = 1}^{n} \expectation [ V^{\mathbb{A}_{\mathrm{seq}}}_{\tau_{\state, \action_{k}} (n')} (\state') ].
    \end{align*}
Equivalently, it can be written as 
    \begin{align*}
        \expectation [ \overline{Q} (\state, \action_{k}) ] = \mu_{k, n'} = \frac{1}{n} \sum_{n' = 1}^{n} \expectation [ V^{\mathbb{A}_{\mathrm{seq}}}_{\tau_{\state, \action_{k}} (n')} (\state') ],
    \end{align*}
\noindent where $n = N (\state, \action_{k}) + O (\state, \action_{k})$. This completes the proof of the first necessary condition in Theorem~\ref{theorem: necessary condition for a good algorithm in the tree search setting}.

Assuming the first necessary condition is satisfied, we proceed to prove the second necessary condition. Note that according to the assumption on $\widetilde{N}$ made in the theorem, we have:
    \begin{align*}
        \overline{N} (\state_{0}, \action_{k}) = N (\state_{0}, \action_{k}) + f ( O (\state_{0}, \action_{k}) ),
    \end{align*}
\noindent where $f$ can be any function whose domain is $\{ x \mid 0 \leq x \leq M - 1, x \in \mathbb{Z} \}$ and whose range is $[ 0, + \infty )$. 

Suppose at time step $\tau_{0}$ ($\tau_{0} < \simudelay$), arm $k$ has been visited $n_{0} + 1$ times (one of them is the done at the initialization phase of the corresponding edge $(\state, \action_{k})$). We consider the quantity $\Pr ( \overline{Q}_{k, n_{0} + 1, n_{0}} \leq \mu_{k} + \overline{e}_{\tau_{0}, n_{0} + 1, n_{0}} )$, which represents the probability of Eq.~(\ref{eq: tree with pseudo statistics term 2}) in the circumstance specified by $k$, $\tau_{0}$, and $n_{0}$:
    \begin{align*}
        & \Pr ( \overline{Q}_{k, n_{0} + 1, n_{0}} \leq \mu_{k, n_{0} + 1} + \overline{e}_{\tau_{0}, n_{0} + 1, n_{0}} ) \\
        \overset{(a)}{=} & \Pr ( \overline{Q}_{k, n_{0} + 1, 0} \geq \mu_{k, n_{0} + 1} + \overline{e}_{\tau_{0}, n_{0} + 1, n_{0}} ) \\
        \overset{(b)}{\leq} & \exp ( - \frac{n_{0} + 1}{2} \overline{e}_{\tau_{0}, n_{0} + 1, n_{0}}^{2} ) \\
        \overset{(c)}{=} & \exp ( - \frac{n_{0} + 1}{2} \frac{2 \log g(\tau_{0})}{f(n_{0}) + 1} ) \\
        = & \frac{1}{g(\tau_{0})^{\frac{n_{0} + 1}{f(n_{0}) + 1}}}, \numberthis \label{eq: another partial result from the necessary condition proof 1}
    \end{align*}
\noindent where $(a)$ uses the assumption that the first necessary condition is satisfied (i.e. $\overline{Q}_{k, n_{0} + 1, n_{0}} = Q_{k, n_{0} + 1}$), $(b)$ follows the Chernoff-Hoeffding bound on the $\frac{1}{n_{0} + 1}$-subgaussian random variable $Q_{k, n_{0} + 1}$, and $(c)$ expands the definition of $\overline{e}_{\tau_{0}, n_{0} + 1, n_{0}}$ ($g (\cdot)$ is defined as follows). Suppose we have $f(n_{0}) = n'_{0} < n_{0}$ (without loss of generality assume $f(x) = x \; (x \neq n_{0})$). We first show that in this case $g (\tau_{0}) = \tau_{0} - n_{0} + n'_{0}$. Specifically, according to the tree policy (Eq.~(\ref{eq: modified UCT})), we have
    \begin{align*}
        g(\tau) & = \sum_{k} \overline{N} (\state, \action_{k}) \\
        & = \sum_{k} N (\state, \action_{k}) + f ( O (\state, \action_{k}) ) \\
        & = \sum_{k} N (\state, \action_{k}) + O (\state, \action_{k}) + \Big ( f( O(\state, \action_{k}) ) - O(\state, \action_{k}) \Big ).
    \end{align*}
First, notice that at rollout step $\tau_{0}$, we have $\sum_{k} N (\state, \action_{k}) + O (\state, \action_{k}) = \tau_{0}$. The reason is that at rollout step $\tau_{0}$, Algorithm~\ref{alg: general framework of parallel MCTS algs} has initialized $\tau_{0}$ simulations in total, and each simulation is either observed (will be counted by $N (\state, \action_{k})$) or unobserved (will be counted by $O (\state, \action_{k})$). Next, we look at the last term $f( O(\state, \action_{k}) ) - O(\state, \action_{k})$. By assumption, it equals to $n'_{0} - n_{0}$ if and only if $O (\state, \action_{k}) = n_{0}$, and is otherwise zero. Therefore, at rollout step $\tau_{0}$, as long as no other edges have $n_{0}$ on-going simulations, we can conclude that $g(\tau_{0}) = \tau_{0} - n_{0} + n'_{0}$. Therefore, Eq.~(\ref{eq: another partial result from the necessary condition proof 1}) can be further simplified as
    \begin{align*}
        \Pr ( \overline{Q}_{k, n_{0} + 1, n_{0}} \leq \mu_{k, n_{0} + 1} + \overline{e}_{\tau_{0}, n_{0} + 1, n_{0}} ) 
        \leq & (\tau_{0} - n_{0} + n'_{0})^{-\frac{n_{0} + 1}{n'_{0} + 1}}. \numberthis \label{eq: another partial result from the necessary condition proof 2}
    \end{align*}
We focus on the condition of having nonvanishing excess regret. Specifically, we focus on the condition of Eq.~(\ref{eq: another partial result from the necessary condition proof 2}) being greater than the upper bound in the sequential case (i.e. when all $n_{0}$ unobserved samples are observed), which is $\frac{1}{\tau_{0}}$ Chernoff-Hoeffding inequality of subgaussian variables:
    \begin{align*}
        (\tau_{0} - n_{0} + n'_{0})^{-\frac{n_{0} + 1}{n'_{0} + 1}} > \frac{1}{\tau_{0}} \Leftrightarrow \tau_{0} > (\tau_{0} - n_{0} + n'_{0})^{\frac{n_{0} + 1}{n'_{0} + 1}}.
    \end{align*}
For any $n_{0}$ and $n'_{0}$, there exists $t_{0}$ such that when $\tau_{0} > t_{0}$, we have
    \begin{align*}
        \tau_{0} - (\tau_{0} - n_{0} + n'_{0})^{\frac{n_{0} + 1}{n'_{0} + 1}} > n_{0} - n'_{0}, \numberthis \label{eq: a somehow important result}
    \end{align*}
\noindent where $n_{0} - n'_{0}$ is nonvanishing as $\tau_{0}$ increases. Therefore, it will incur a nonvanishing term in the probability of Eq.~(\ref{eq: tree with pseudo statistics term 2}), which will result in a nonvanishing regret term. Therefore, to have vanishing cumulative regret, we should not have $f(n_{0}) = n'_{0} < n_{0}$. This confirms the necessary condition $f(x) > x$. 
\end{proof}

\subsection{Theoretical Justification of WU-UCT}
\label{Theoretical Justification of WU-UCT}

This section provides formal proof of Theorem~\ref{theorem: regret upper bound of WU-UCT in stochastic bandit setup}, which indicates WU-UCT achieves vanishing excess regret under the depth-2 setup. In the following, we first justify the statement ``$R_{UCT} (n)$ is the cumulative regret of running the (sequential) UCT for $n$ steps on $\mathbb{T}$'', i.e., the expected cumulative regret of the UCT algorithm under the depth-2 setup.

\noindent \textbf{Cumulative regret upper bound of UCT in the depth-2 case} $\;$ Define $\action_{k^{*}}$ as the optimal action that leads to the highest expected reward. According to the regret decomposition identity \cite{mazumdar2017multi}, $\cumulateregret{t}$ can be decomposed with respect to different arms:
    \begin{align*}
        \cumulateregret{t} = \sum_{k \in \{1, \dots, K\}, k \neq k^{*}} \Delta_{k} \expectation \left [ T_{k} (t) \right ], \numberthis \label{eq: re-decomposed regret}
    \end{align*}
\noindent where $\Delta_{k} := \mu^{*} - \mu_{k}$, $\mu^{*} := \max_{k} \mu_{k}$, $k^{*} := \argmax_{k} \mu_{k}$, and $T_{k} (t)$ is defined as the number of times arm $k$ is selected in the first $t$ rollouts.
This suggests that we only need to bound the expected visit counts of all suboptimal arms (i.e., $\expectation \left [ T_{k} (t) \right ] \; (k \neq k^{*})$). $Q_{t} (\state_{0}, \action_{k})$ is defined as the reward estimate for arm $k$ at the end of the $t$th rollout, and $N_{t} (\state_{0}, \action_{k})$ denotes the visit count of arm $k$ at the end of rollout step $t$. To simplify notation, we additionally define $Q_{k, n}$ as the (empirical) average reward of arm $k$ after the $n$th observation of that arm (i.e., $n$ simulation returns have been obtained).

The event $\text{Arm}_{\tau} = k$ means the $k$th arm is pulled at time $t$. According to the definition of $T_{k} (t)$, we have (define $l$ as an arbitrary positive integer; $e_{t, n} := \sqrt{(2 \ln t) / n}$)
    \begin{align*}
        T_{k} (t) & = 1 + \sum_{\tau = K + 1}^{t} \indicatorfunc{\text{Arm}_{\tau} = k} \\
        & \leq l + \sum_{\tau = K + 1}^{t} \indicatorfunc{\text{Arm}_{\tau} = k, T_{k} (\tau - 1) \geq l} \\
        & \overset{(a)}{\leq} l + \sum_{\tau = K + 1}^{t} \mathbbm{1} \Bigg [Q_{\tau - 1} (\state_{0}, \action_{k^{*}}) + e_{\tau, N_{\tau - 1} (\state_{0}, \action_{k^{*}})} \leq Q_{\tau - 1} (\state_{0}, \action_{k}) + e_{\tau, N_{\tau - 1} (\state_{0}, \action_{k})} \Bigg ] \\
        & \leq l + \sum_{\tau = K + 1}^{t} \mathbbm{1} \Bigg [ \min_{0 < n < \tau} Q_{k^{*}, n} + e_{\tau, n} \leq \max_{l \leq n' < \tau} Q_{k, n'} + e_{\tau, n'} \Bigg ] \\
        & \leq l + \sum_{\tau = 1}^{t} \max_{n \in [1, \tau - 1]} \max_{n' \in [1, \tau - 1]} \mathbbm{1} \Bigg [ Q_{k^{*}, n} + e_{\tau, n} \leq Q_{k, n'} + e_{\tau, n'} \Bigg ]. \numberthis \label{eq: bound for sub-optimal arm visit count}
    \end{align*}
\noindent where $(a)$ uses the fact that the necessary condition of choosing arm $k$ at rollout step $\tau$ is that the upper confidence bound of the $k$th arm is greater than or equal to that of the optimal arm $k^{*}$.
    
We bound the probability of the event $Q_{k^{*}, n} + e_{\tau, n} \leq Q_{k, n'} + e_{\tau, n'}$ using the sum of the following three events' probability:
    \begin{align}
        Q_{k^{*}, n} \leq \mu^{*} - e_{\tau, n}, \label{eq: UCT bound condition 1} \\
        Q_{k, n'} \geq \mu_{k} + e_{\tau, n'}, \label{eq: UCT bound condition 2} \\
        \mu^{*} < \mu_{k} + 2 e_{\tau, n'}. \label{eq: UCT bound condition 3}
    \end{align}
Since the rewards received from arm $k$ minus its expectation (i.e., $R_{k} - \mu_{k}$) are independent 1-subgaussian random variables (by the assumption made in Theorem~\ref{theorem: regret upper bound of WU-UCT in stochastic bandit setup}), we can show that $Q_{k, n}$ is $1/n$-subgaussian (since it is the average of $n$ 1-subgaussian random variables \cite{buldygin1980sub}). The Chernoff-Hoeffding bound for subgaussian random variables state that if random variable $X$ is $\sigma^{2}$-subgaussian, we have $\Pr \left ( X \geq \epsilon \right ) \leq \exp \left ( - \epsilon^{2} / (2 \sigma^{2}) \right )$. Plug in Eqs.~(\ref{eq: UCT bound condition 1}) and (\ref{eq: UCT bound condition 2}), we have 
    \begin{align}
        & \Pr \left ( Q_{k^{*}, n} \leq \mu^{*} - e_{\tau, n} \right ) \leq 1 / \tau, \label{eq: UCT bound chernoff 1} \\
        & \Pr \left ( Q_{k, n'} \geq \mu_{k} + e_{\tau, n'} \right ) \leq 1 / \tau. \label{eq: UCT bound chernoff 2}
    \end{align}
Next, we focus on Eq.~(\ref{eq: UCT bound condition 3}):
    \begin{align*}
        \mu_{k} + 2 e_{t, n} > \mu^{*} \quad \Leftrightarrow \quad \mu_{k} + 2 \sqrt{\frac{2 \ln t}{s_{k}}} > \mu^{*} \quad \Leftrightarrow \quad \sqrt{\frac{2 \ln t}{s_{k}}} > \frac{\Delta_{k}}{2} \quad \Leftrightarrow \quad s_{k} < \frac{8 \ln t}{\Delta_{k}^{2}}.
    \end{align*}
Therefore, when $n' \geq \left \lceil \frac{8 \ln t}{\Delta_{k}^{2}} \right \rceil$, Eq.~(\ref{eq: UCT bound condition 3}) is guaranteed to be false. So we have
    \begin{align*}
        \expectation \left [ T_{k} (t) \right ] & \leq \left \lceil \frac{8 \ln t}{\Delta_{k}^{2}} \right \rceil + \sum_{\tau = 1}^{t} \left ( \Pr \left ( Q_{k^{*}, n} \leq \mu^{*} - e_{\tau, n} \right ) + \Pr \left ( Q_{k, n'} \geq \mu_{k} + e_{\tau, n'} \right ) \right ) \\
        & \leq \left \lceil \frac{8 \ln t}{\Delta_{k}^{2}} \right \rceil + \sum_{\tau = 1}^{t} \frac{2}{\tau} \\
        & \leq \left ( \frac{8}{\Delta_{k}^{2}} + 2 \right ) \ln t + 1.
    \end{align*}
Plugging this result in Eq.~(\ref{eq: re-decomposed regret}) gives the regret upper bound
    \begin{align*}
        \cumulateregret{t} \leq R_{UCT} := \sum_{k \in \{1, \dots, K\}, k \neq k^{*}} \left [ \left ( \frac{8}{\Delta_{k}} + 2 \Delta_{k} \right ) \ln t + \Delta_{k} \right ].
    \end{align*}
    
Next, we justify the cumulative regret upper bound of WU-UCT.

\noindent \textbf{Formal proof of Theorem~\ref{theorem: regret upper bound of WU-UCT in stochastic bandit setup}}

\vspace{-0.8em}

\begin{proof}[Proof of Theorem~\ref{theorem: regret upper bound of WU-UCT in stochastic bandit setup}]
Before delving into the proof, we briefly review WU-UCT \cite{liu2020watch}. WU-UCT constructs a global search tree that is operated only by the main/master process. The master process repeatedly perform rollouts and assign simulation and expansion tasks to the workers and collect results from them. Specifically, the main process performs selection with the modified tree policy (\ref{eq: modified UCT}) (with the hyperparameter specified according to Table~\ref{table: parallel MCTS algorithms specification extended}), where an \emph{incomplete update} process increments the \emph{incomplete visit count} $O (\state_{0}, \action_{k})$ of the traversed nodes by one. Expansions and simulations are done in parallel by the workers, and we refer readers interested in the details to Liu et al. \cite{liu2020watch}. During backpropagation, an additional \emph{complete update} process decrements $O (\state, \action)$ of the traversed nodes by one.

On the high level, WU-UCT has a parallel architecture similar to TreeP, where all statistics are globally available (thus $\syncrate = 1$). We start the proof by a high-level demonstration, and then follow the key intuitions to formalize it.

\begin{figure}[t]
    \centering
    \includegraphics[width=0.9\columnwidth]{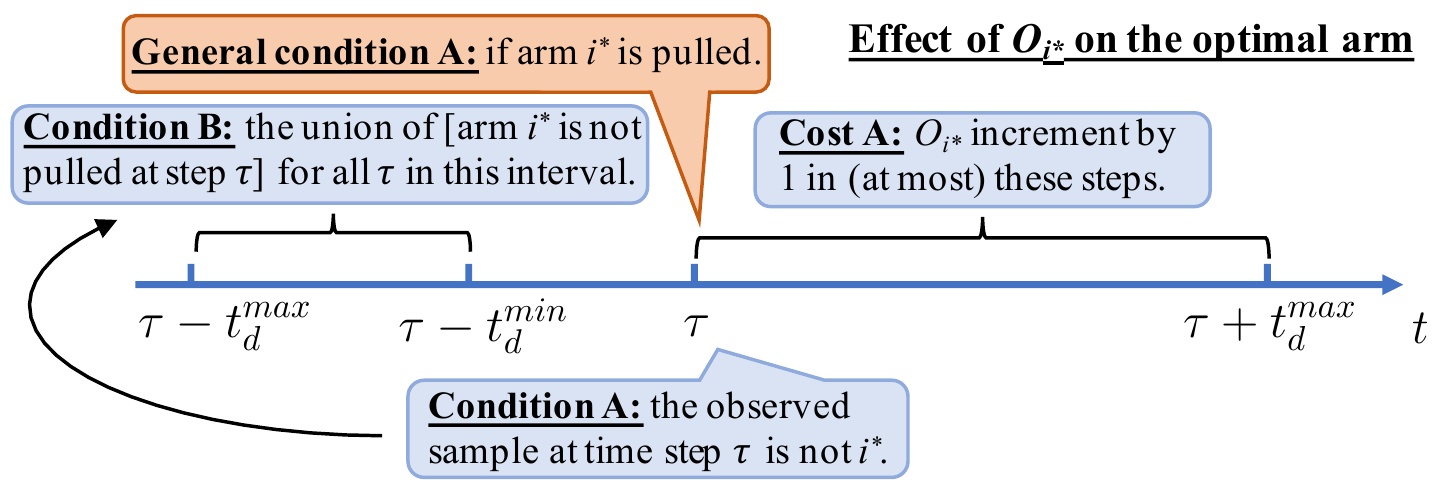}
    
    \vspace{-0.5em}
    \caption{The influence of WU-UCT on the expected cumulative regret (comparing to the sequential case) by pre-updating $O_{i^{*}}$.}
    \label{fig: WU-UCT asymptotic proof}
\end{figure}

We argue that when dealing with the MAB problem, WU-UCT can be treated as a sequential UCT where some of the \emph{observed samples} are replaced by \emph{unobserved samples} without actual simulation return. First, note that with the help of the adjustment on the visit count (i.e., $\overline{N} (\state_{0}, \action_{k}) := N (\state_{0}, \action_{k}) + O (\state_{0}, \action_{k})$), at time step $\tau$, we can upper bound $T_{k} (t)$ by 
    \begin{align*}
        T_{k} (t) & \leq l + \sum_{\tau = K + 1}^{t} \mathbbm{1} \Bigg [Q_{\tau - 1} (\state_{0}, \action_{k^{*}}) + e_{\rho (\tau), \overline{N}_{\tau - 1} (\state_{0}, \action_{k^{*}})} \leq Q_{\tau - 1} (\state_{0}, \action_{k}) + e_{\rho (\tau), \overline{N}_{\tau - 1} (\state_{0}, \action_{k})} \Bigg ], \numberthis \label{eq: WU-UCT proof partial result 1}
    \end{align*}
\noindent where $\rho (\tau) := \sum_{k = 1}^{K} \overline{N} (\state_{0}, \action_{k})$ according to the tree policy defined by Eq.~(\ref{eq: modified UCT}). By the definition $\overline{N} (\state_{0}, \action_{k}) := N (\state_{0}, \action_{k}) + O (\state_{0}, \action_{k})$ we can easily verify that $\rho (\tau) = \tau$: note that at the end of the $\tau$th rollout, there are $\tau$ assigned simulation tasks, and each task is either observed (is recorded in $N$) or unobserved (is recorded in $O$). Therefore, we can rewrite Eq.~(\ref{eq: WU-UCT proof partial result 1}) as 
    \begin{align*}
        T_{k} (t) & \leq l + \sum_{\tau = K + 1}^{t} \mathbbm{1} \Bigg [\overline{Q}_{\tau - 1} (\state_{0}, \action_{k^{*}}) + e_{\tau, \overline{N}_{\tau - 1} (\state_{0}, \action_{k^{*}})} \leq \overline{Q}_{\tau - 1} (\state_{0}, \action_{k}) + e_{\tau, \overline{N}_{\tau - 1} (\state_{0}, \action_{k})} \Bigg ]. \numberthis \label{eq: WU-UCT proof partial result 2}
    \end{align*}
The key observation we want to emphasize here is that with the help of the adjustment on the visit count (i.e., $\overline{N} (\state_{0}, \action_{k}) := N (\state_{0}, \action_{k}) + O (\state_{0}, \action_{k})$), the time step represented by $\rho (\tau)$ has been calibrated to be the same with the sequential case, i.e., $\rho (\tau) := \tau$. 
In this way, as we shall proceed to show, though according to Table~\ref{table: parallel MCTS algorithms specification extended}, WU-UCT has $\simudelay = M$, it has vanishing regret. Under this observation, the main difference between WU-UCT and the sequential UCT is that its value estimates $\{ Q_{\tau} (\state_{0}, \action_{k}) \}_{k = 1}^{K}$ are less informative compared to UCT. Specifically, for the \emph{incomplete simulations}, though $\overline{N} (\state_{0}, \action_{k})$ is adjusted by $O (\state_{0}, \action_{k})$ and resemble the sequential case, these simulation returns $\hat{V} (\state_{k})$ are not available and the variance of the estimate is relatively high compared to the sequential algorithm. Keep in mind this similarity between WU-UCT and UCT. In the following, we analyze the excess regret caused by the \emph{inaccurate} $Q_{\tau} (\state_{0}, \action_{k})$.

Following Eq.~(\ref{eq: WU-UCT proof partial result 2}), we have 
    \begin{align*}
        T_{k} (t) & \leq l + \sum_{\tau = K + 1}^{t} \mathbbm{1} \Bigg [\overline{Q}_{\tau - 1} (\state_{0}, \action_{k^{*}}) + e_{\tau, \overline{N}_{\tau - 1} (\state_{0}, \action_{k^{*}})} \leq \overline{Q}_{\tau - 1} (\state_{0}, \action_{k}) + e_{\tau, \overline{N}_{\tau - 1} (\state_{0}, \action_{k})} \Bigg ] \\
        & \overset{(a)}{\leq} l + \sum_{\tau = K + 1}^{t} \mathbbm{1} \Bigg [Q_{\tau - 1} (\state_{0}, \action_{k^{*}}) + e_{\tau, \overline{N}_{\tau - 1} (\state_{0}, \action_{k^{*}})} \leq Q_{\tau - 1} (\state_{0}, \action_{k}) + e_{\tau, \overline{N}_{\tau - 1} (\state_{0}, \action_{k})} \Bigg ] \\
        & \overset{(b)}{\leq} l + \sum_{\tau = 1}^{t} \max_{n \in [1, \tau - 1]} \max_{n' \in [1, \tau - 1]} \mathbbm{1} \Bigg [ Q_{k^{*}, n} + e_{\tau, n + O_{\tau} (\state_{0}, \action_{k^{*}})} \leq Q_{k, n'} + e_{\tau, n' + O_{\tau}\! (\state_{0}, \action_{k})} \Bigg ] \\
        & \overset{(c)}{\leq} l + \sum_{\tau = 1}^{t} \max_{n \in [1, \tau - 1]} \max_{n' \in [1, \tau - 1]} \mathbbm{1} \Bigg [ Q_{k^{*}, n} + e_{\tau, n + O_{\tau} (\state_{0}, \action_{k^{*}})} \leq Q_{k, n'} + e_{\tau, n'} \Bigg ].
    \end{align*}
\noindent where $(a)$ uses the fact that WU-UCT do not adjust the value (i.e. $\widetilde{Q} (\state, \action) = 0$), which results in $\overline{Q}_{\tau - 1} (\state_{0}, \action_{k^{*}}) = Q_{\tau - 1} (\state_{0}, \action_{k^{*}})$; $(b)$ largely follows Eq.~(\ref{eq: bound for sub-optimal arm visit count}), and $(c)$ is based on the fact that $e_{\tau, n_{1}} > e_{\tau, n_{2}} \; (n_{1} < n_{2})$.

The main difference between the above upper bound and the corresponding upper bound of UCT (Eq.~(\ref{eq: bound for sub-optimal arm visit count})) is the potential lag in $Q$, i.e., it has $O_{\tau} (\state_{0}, \action_{k})$ less observed value estimates $\hat{V} (\state_{k^{*}})$ compared to that expected by the confidence interval $c$. Similar to Eqs.~(\ref{eq: UCT bound condition 1})-(\ref{eq: UCT bound condition 3}), the probability of the event in the indicator function $\mathbbm{1} [ \cdot ]$ can be bounded by the sum of the probability of the three following events:
    \begin{align*}
        Q_{k^{*}, n} & \leq \mu^{*} - e_{\tau, n + O_{\tau} (\state_{0}, \action_{k^{*}})}, \numberthis \label{eq: UCT bound condition 2-1} \\
        Q_{k, n'} & \geq \mu_{k} + e_{\tau, n'}, \numberthis \label{eq: UCT bound condition 2-2} \\
        \mu^{*} & < \mu_{k} + 2 e_{\tau, n'}. \numberthis \label{eq: UCT bound condition 2-3}
    \end{align*}
Therefore, we only need to analysis the extra regret caused by $O_{\tau} (\state_{0}, \action_{k^{*}})$ in Eq.~(\ref{eq: UCT bound condition 2-1}).
Specifically, Figure~\ref{fig: WU-UCT asymptotic proof} illustrate the affection on the regret caused by the existence of \emph{incomplete simulations} (i.e., ongoing simulations whose return is currently unavailable) of the optimal arm. The remainder of the proof uses the following definition.

\begin{definition}[Simulation interval $\tau_{\mathrm{sim}}$]
\label{definition: simulation interval}
Between the period of a simulation task $(\state, m)$ being assigned to a worker in the simulation step and being returned in the wait step (i.e. the simulation completes), there are at most $\maxsimudelay \!-\! 1$ and at least $\minsimudelay \!-\! 1$ other returned simulation result $(\state', V, \hat{m})$ where $\hat{m} \!=\! m$. 
\end{definition}

In the case of WU-UCT, since the algorithm contains only one global search tree, $\maxsimudelay$ and $\minsimudelay$ measure the maximum and minimum rollout steps taken from a simulation task being assigned and being returned.

As demonstrated by Figure~\ref{fig: WU-UCT asymptotic proof}, the main cost of the \emph{on-going simulations} on the optimal arm is that it make the value estimate $\overline{Q}$ less accurate, and there will be an underestimation on the upper confidence bound since we shrinked the exploration term in \ref{eq: modified UCT} over-optimistically. 
Concretely, we formalize the condition of the loss and the cost/effect of it. To have $O_{\tau} (\state_{0}, \action_{k^{*}})$ increased at time step $\tau$, the precondition should be that the arm $k^{*}$ is pulled at that time step (i.e., general condition A). In addition to that, we have to make sure that the observed/returned task/simulation at time $\tau$ is not for arm $k^{*}$ (i.e., condition A) since if that is the case, $O_{\tau} (\state_{0}, \action_{k^{*}})$ would not change before and after time $\tau$, and thus no additional cost will be added. Since condition A is hard to directly quantify, we instead rely on a looser condition that has guaranteed larger probability of it. Specifically, condition B (Figure~\ref{fig: WU-UCT asymptotic proof}) is a quantifiable constraint that satisfies the above statement. Condition B is based on the fact that only the tasks initiated between time $\tau - \maxsimudelay$ and $\tau - \minsimudelay$ is possible to terminate at time $\tau$, where we define $\maxsimudelay$ as the maximum simulation interval and $\minsimudelay$ as the minimum simulation delay. To justify $\Pr (\text{condition A}) \leq \Pr (\text{condition B})$, we can verify that the converse of condition B (i.e., arm $i^{*}$ is pulled in all time steps between $\tau - \maxsimudelay$ and $\tau - \minsimudelay$. Without loss of generality, in the following we assume the simulation interval is always equal to $\simudelay$.

Therefore, at an abstract level, the additional expected cumulative regret incurred by WU-UCT compared to the (sequential) UCT can be written as:
    \begin{align*}
        & \Pr \left ( \text{General condition A} \right ) \cdot \Pr \left ( \text{Condition B} \right ) \cdot \Pr \left ( \text{Cost A} \right ). \numberthis \label{eq: WU-UCT abstract loss 1}
    \end{align*}
We upper bound equation (\ref{eq: WU-UCT abstract loss 1}) by
    \begin{align*}
        & \Pr \left ( \text{Condition B} \right ) \cdot \Pr \left ( \text{Cost A} \right ), \numberthis \label{eq: WU-UCT abstract loss 2}
    \end{align*}
We now consider each of the probabilities.

\noindent \textbf{Condition B} $\;$ As hinted by the description of condition B in Figure~\ref{fig: WU-UCT asymptotic proof}, the probability of condition B is upper bounded by 
    \begin{align*}
        & \left ( \maxsimudelay - \minsimudelay + 1 \right ) \cdot \max_{\tau_{0} \in [ \minsimudelay, \maxsimudelay ]} \max_{n \in [1, \tau - 1]} \max_{n' \in [1, \tau - 1]} \Big \{ \Pr ( Q_{k^{*}, n} \leq \mu^{*} - e_{\tau - \tau_{0}, n} ) \\ 
        & \qquad \qquad \qquad \qquad \qquad \qquad \qquad \qquad \qquad \qquad \quad + \Pr ( Q_{k, n'} \geq \mu_{k} + e_{\tau - \tau_{0}, n'}) \Big \} \\
        \leq & 2 \cdot \frac{\maxsimudelay - \minsimudelay + 1}{\tau - \maxsimudelay} \overset{(a)}{=} \frac{2}{\tau - \simudelay},
    \end{align*}
\noindent where $(a)$ uses our assumption that $\maxsimudelay = \minsimudelay = \simudelay$.\footnote{If this assumption does not hold, it will only add a constant term (independent to the number of rollout steps) in the final regret, which will not affect our main result.}
Note that in this case we do not need to consider the case where $O_{\tau} (\state_{0}, \action_{k}) > 0$ since they are bounded by the \emph{cost A} term in previous time steps and would be redundant to consider again here. Specifically, the excess regret caused by the on-going simulation at time step $\tau - \simudelay$ has been upper bounded by the \emph{cost A} term in their respective rollout step that they are initialized.

\noindent \textbf{Cost A} $\;$ The cost here refers to the additional expected regret incurred by using the adjusted confidence interval 
    \begin{align*}
        \sqrt{\frac{2 \ln \tau}{N_{\tau} (\state_{0}, \action_{k^{*}}) + O_{\tau} (\state_{0}, \action_{k^{*}})}} \quad (O_{\tau} (\state_{0}, \action_{k^{*}}) > 0) 
    \end{align*}
instead of the optimistic one (in the sequential case) $\sqrt{\frac{2 \ln \tau}{N_{\tau} (\state_{0}, \action_{k^{*}})}}$.
Formally, cost A can be bounded by 
    \begin{align}
        & \max_{O \in [1, M - 1]} \sum_{t = \tau}^{\tau + \simudelay} \max_{n \in [1, \tau - 1]} \max_{n' \in [1, \tau - 1]} \Big \{ \Pr ( Q_{k^{*}, n + O} \leq \mu^{*} - e_{t, n + O + 1} ) \nonumber \\
        & \qquad \qquad \qquad \qquad \qquad \qquad \qquad \; \; + \Pr ( Q_{k, n' + O} \geq \mu_{k} + e_{t, n' + O + 1}) \nonumber \\
        & \qquad \qquad \qquad \qquad \qquad \qquad \qquad \; \; - \Pr ( Q_{k^{*}, n + O} \leq \mu^{*} - e_{t, n + O} ) \nonumber \\
        & \qquad \qquad \qquad \qquad \qquad \qquad \qquad \; \; - \Pr ( V_{k, s' + O} \geq \mu_{i} + e_{t, n' + O}) \Big \} \nonumber \\
        & \leq \max_{O \in [1, M - 1]} \sum_{t = \tau}^{\tau + \simudelay} 2 \left [ t^{- \frac{O}{O + 1}} - t^{-1} \right ] \nonumber \\
        & \leq 2 \left ( \frac{\simudelay}{\sqrt{\tau}} - \frac{\simudelay}{\tau} \right ). \label{eq: WU-UCT partial result 0_1}
    \end{align}
Finally, we plug in the upper bounds of the conditions and costs into Eq.~(\ref{eq: WU-UCT abstract loss 2}), which gives
    \begin{align*}
        \frac{4}{\tau - \simudelay} \left ( \frac{\simudelay}{\sqrt{\tau}} - \frac{\simudelay}{\tau} \right ).
    \end{align*}
Finally, we upper bound the total cost incurred on $\expectation \left [ T_{i} (t) \right ]$ by
    \begin{align*}
        \sum_{\tau = \lceil \frac{8 \ln t}{\Delta_{i}^{2}} \rceil}^{t} \frac{4}{\tau - \simudelay} \left ( \frac{\simudelay}{\sqrt{\tau}} - \frac{\simudelay}{\tau} \right )
        \numberthis \label{eq: WU-UCT partial result 1}
    \end{align*}
Since $\simudelay$ is not dependent on $t$, there exists $t^{*} \in \mathbb{Z}^{+}$ such that whenever $t > t^{*}$, we have $\simudelay < \lceil \frac{8 \ln t}{\Delta_{k}^{2}} \rceil$ / 2. Therefore, Eq.~(\ref{eq: WU-UCT partial result 1}) is upper bounded by
    \begin{align*}
        4 \simudelay \sum_{\tau = \lceil \frac{8 \ln t}{\Delta_{k}^{2}} \rceil / 2}^{t} \left [ \frac{1}{\tau \sqrt{\tau}} - \frac{1}{\tau^{2}} \right ] \leq 2 \simudelay \left [ 2 \sqrt{\frac{\Delta_{k}^{2}}{4 \ln t}} - \frac{\Delta_{k}^{2}}{4 \ln t} \right ]. \numberthis \label{eq: WU-UCT partial result 2}
    \end{align*}
Note that Eq.~(\ref{eq: WU-UCT proof partial result 2}) is the regret bound of $\expectation [ T_{k} (t) ]$, plugging in Eq.~ (\ref{eq: re-decomposed regret}) finishes the proof, that is, the cumulative regret of WU-UCT on the MAB case is upper bounded by
    \begin{align*}
        R_{\mathrm{UCT}} (n) + 4 \simudelay \sum_{k: \mu_{k} < \mu^{*}} 2 \Delta_{k} \sqrt{\frac{\Delta_{k}^{2}}{4 \ln n}}.
    \end{align*}
Since in WU-UCT, $\tau_{\mathrm{sim}} = M$, the above quantity is equal to
    \begin{align*}
        R_{\mathrm{UCT}} (n) + 4 M \sum_{k: \mu_{k} < \mu^{*}} 2 \Delta_{k} \sqrt{\frac{\Delta_{k}^{2}}{4 \ln n}}.
    \end{align*}
\end{proof}

\section{Additional Details for BU-UCT}
\label{Additional Details for BU-UCT}

This section provides additional details of the BU-UCT algorithm, including an algorithm table (Appendix~\ref{Algorithm Table for BU-UCT}) and introduction of all its hyperparameters (Appendix~\ref{Hyperparameters of BU-UCT}).

\subsection{Algorithm Table for BU-UCT}
\label{Algorithm Table for BU-UCT}

The algorithm table of BU-UCT is provided in Algorithm~\ref{alg: BU-UCT}.

\subsection{Hyperparameters of BU-UCT}
\label{Hyperparameters of BU-UCT}

The following provides a list of all hyperparameters in BU-UCT. We briefly discuss the recommended values for each hyperparameter.

$\bullet \;$ $m_{\mathrm{max}}$. $m_{\mathrm{max}} \in (0, 1)$ is the hyperparameter that controls the degree we penalize $\overline{O}$. Specifically, if an edge $(\state, \action)$ has $\overline{O} (\state, \action) \geq m_{\mathrm{max}} \cdot M$ ($M$ is the number of workers), action $\action$ will not be selected by the tree policy when we are currently at node $\state$. In our experiments, we choose $m_{\mathrm{max}} \!=\! 0.8$.

$\bullet \;$ Maximum tree depth/width. These hyperparameters should depend on the complexity of the tasks as well as the total computation budget. In our experiments, both the maximum tree depth and the maximum tree width are set to $100$ and $20$, respectively.

$\bullet \;$ Number of expansion/simulation workers. Expansion and simulation workers perform expansion and simulation tasks, respectively. In our experiments, we use $1$ expansion worker and $16$ simulation workers.

$\bullet \;$ The tree policy balancing factor $c$. $c$ balance the exploration term (the second term) and the exploitation term (the first term) in the tree policy (Eq.~(\ref{eq: UCT tree policy})). In our experiments, it is selected as the standard deviation of the cumulative reward received by each node. For example, for node $\state$, $c$ is computed by the standard deviation of all cumulative reward received by $\state$ (i.e., all $V (\state)$).

\section{Alternative Surrogate Statistics}
\label{Alternative Surrogate Statistics}

\begin{figure}[t]
    \centering
    \includegraphics[width=\columnwidth]{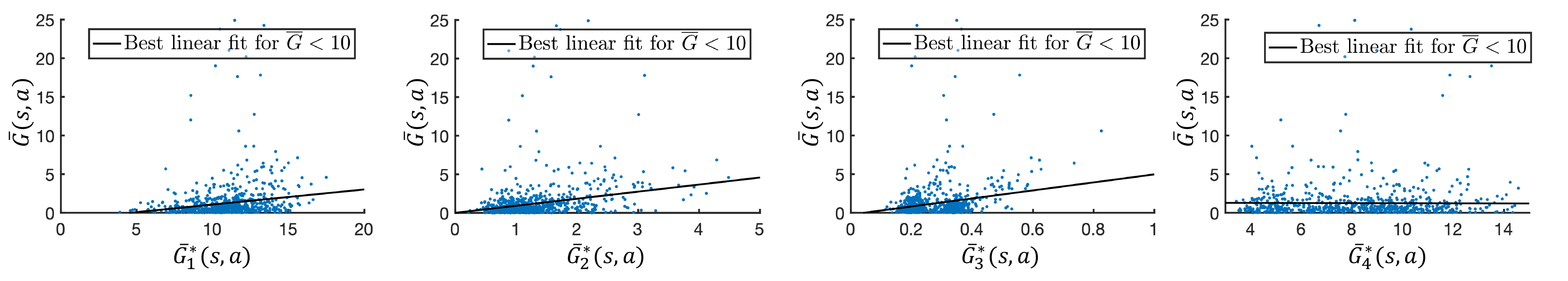}
    
    \vspace{-0.5em}
    \caption{Relation between four additional surrogate gaps (i.e., $\overline{G}^{*}_{1} (\state, \action)$, $\overline{G}^{*}_{2} (\state, \action)$, $\overline{G}^{*}_{3} (\state, \action)$, and $\overline{G}^{*}_{4} (\state, \action)$) and the action value gap (i.e., $\overline{G} (\state, \action)$).}
    \label{fig: additional surrogate statistics}
\end{figure}

Figure~\ref{fig: additional surrogate statistics} presents four surrogate gaps (i.e., $\overline{G}^{*}_{1} (\state, \action)$, $\overline{G}^{*}_{2} (\state, \action)$, $\overline{G}^{*}_{3} (\state, \action)$, and $\overline{G}^{*}_{4} (\state, \action)$) that also exhibit positive correlation with the action value gap $\overline{G} (\state, \action)$. $\overline{G}^{*}_{1} (\state, \action)$, $\overline{G}^{*}_{2} (\state, \action)$, and $\overline{G}^{*}_{3} (\state, \action)$ all exhibit positive correlation with the action value gap (although their fitness scores are worse than $\overline{G}^{*} (\state, \action)$ introduced in the main text); $\overline{G}^{*}_{4} (\state, \action)$ is a very related statistics of $\overline{G}^{*} (\state, \action)$ but it is not correlated with $\overline{G} (\state, \action)$. Note that the surrogate gaps presented here are not exhaustive, and better statistics with stronger correlation with the action value gap could exists. The goal of presenting these additional gaps is to help inspire future work for designing more principled parallel MCTS algorithms. In the following, we introduce the three surrogate gaps in detail.

\noindent $\bullet \;$ $\overline{G}^{*}_{1} (\state, \action)$: $\overline{G}^{*}_{1} (\state, \action)$ is the standard deviation of the $n$ (the number of rollouts) simulation returns related to the node $\state'$ (the next state following $(\state, \action)$):
    \begin{align*}
        \overline{G}^{*}_{1} (\state, \action) := \mathrm{Std} \big [ \big \{ V_{i} (\state') \big \}_{i = 1}^{n} \big ],
    \end{align*}
\noindent where $\mathrm{Std} [ A ]$ denotes the standard deviation of all values in the set $A$.
    
\noindent $\bullet \;$ $\overline{G}^{*}_{2} (\state, \action)$: Define $\overline{Q}_{i} (\state, \action)$ as the modified action value related to the edge $(\state, \action)$ at the $i$th rollout step. $\overline{G}^{*}_{2} (\state, \action)$ denotes the standard deviation of the $n$ modified action values $\{ \overline{Q}_{i} (\state, \action) \}_{i = 1}^{n}$:
    \begin{align*}
        \overline{G}^{*}_{2} (\state, \action) := \mathrm{Std} \big [ \big \{ \overline{Q}_{i} (\state, \action) \big \}_{i = 1}^{n} \big ].
    \end{align*}

\noindent $\bullet \;$ $\overline{G}^{*}_{3} (\state, \action)$: Similar to $\overline{G}^{*}_{1} (\state, \action)$, $\overline{G}^{*}_{3} (\state, \action)$ is the coefficient of variance (i.e., standard deviation divided by mean) of the $n$ simulation returns $\{ V_{i} (\state') \}_{i = 1}^{n}$:
    \begin{align*}
        \overline{G}^{*}_{3} (\state, \action) := \frac{\mathrm{Std} \big [ \big \{ V_{i} (\state') \big \}_{i = 1}^{n} \big ]}{\mathrm{Aveg} \big [ \big \{ V_{i} (\state') \big \}_{i = 1}^{n} \big ]},
    \end{align*}
\noindent where $\mathrm{Aveg} [ A ]$ denotes the average of all values in the set $A$.

\noindent $\bullet \;$ $\overline{G}^{*}_{4} (\state, \action)$: $\overline{G}^{*}_{4} (\state, \action)$ is defined as follows
    \begin{align*}
        \overline{G}^{*}_{4} (\state, \action) := \overline{O} (\state, \action) = \frac{1}{n} \sum_{i = 1}^{n} O_{i} (\state', \action').
    \end{align*}
Despite its subtle difference with $\overline{G}^{*} (\state, \action)$ (recall that $\overline{G}^{*} (\state, \action) \!:=\! \max_{\action' \in \actions} \overline{O} (\state', \action')$, where $\state'$ is the next state following $(\state, \action)$), $\overline{G}^{*}_{4} (\state, \action)$ is barely correlated with the action value gap $\overline{G} (\state, \action)$ while $\overline{G}^{*} (\state, \action)$ has strong (positive) correlation with it. Therefore, we conclude being extensively simulated by multiple workers does not necessarily results in high action value gap, which makes sense as the optimal child should be extensively exploited; instead, it is the maximum $\overline{O}$ among its child nodes that strongly correlates with $\overline{G}$, which might suggests that how well the tree policy can properly balance exploration and exploitation is of great importance. This point is further elaborated in the following.

\begin{figure}[t]
    \centering
    \includegraphics[width=\columnwidth]{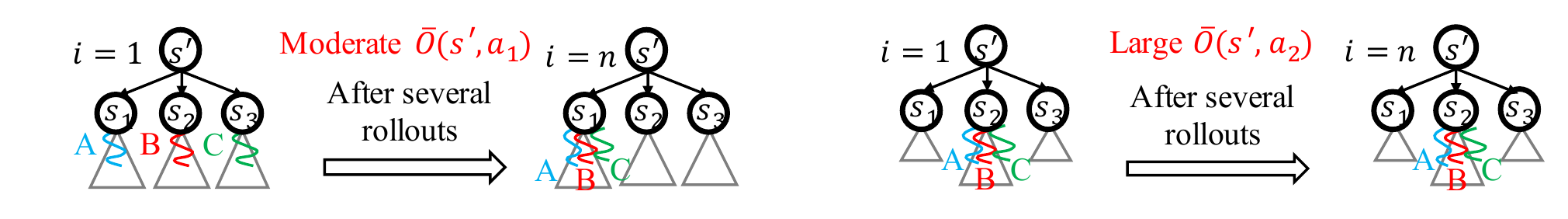}
    
    \vspace{-0.5em}
    \caption{A key implication of the surrogate gap $\overline{G}^{*}$ introduced in the main text: high $\overline{G}^{*}$ potentially indicates overly exploitation of some (suboptimal) child nodes.}
    \label{fig: surrogate gap implication}
\end{figure}

Consider the two example tree search processes shown in Figure~\ref{fig: surrogate gap implication}. On the left side, if the child nodes of $\state'$ (i.e., $\state_{1}$, $\state_{2}$, and $\state_{3}$) are visited equally often at earlier stages (i.e., when $\state'$ has been visit for only a few times) and only start exploiting the nodes with higher action value (i.e., $\overline{Q}$) after certain number of rollouts, the statistics $\overline{O} (\state', \action_{i}) \; (i \!=\! 1, 2, 3)$ will be relatively small since $\overline{O}$ is averaged across different rollout steps (i.e., different $i$ for $O_{i}$). Hence, in this case, the surrogate gap $\overline{G}^{*} (\state, \action)$ ($(\state, \action)$ is the edge that lead to $\state'$) will be relatively small, which suggests that the action value gap is also small. This match our intuition since properly explore all child nodes before exploiting the best one is beneficial and should lead to good performance. In contrast, consider a second case shown on the right side of Figure~\ref{fig: additional surrogate statistics}. In this case, exploitation happens even at the very beginning stage. The agent keeps assigning simulation tasks to query (offspring nodes of) $\state_{2}$. In this case, $\overline{O} (\state', \action_{2})$ is high according to its definition. This leads to high surrogate gap $\overline{G}^{*} (\state, \action)$, which suggests that the performance on this node is less desirable compared to the previous example. In fact, exploit certain nodes aggressively at earlier stages will cause other nodes under-explored, which makes the agent unable to recognize which child node is the most rewarding. If the agent does not happen to select the optimal node to exploit, it will fail to find an optimal action. Although implicit, the second modification (i.e., modification \#2) proposed by BU-UCT try to solve this problem by penalizing over-exploitation (through lowering $\overline{N}$) in earlier stages.

Another advantage of BU-UCT's second modification is to encourage the algorithm to search deeper and wider. Again use Figure~\ref{fig: additional surrogate statistics}(left) as an example. For BU-UCT, if $\overline{N} (\state', \state_{1}) \!=\! m$, then $m$ distinct offspring nodes of $\state_{1}$ have been assigned simulation tasks. However, for its base algorithm WU-UCT, since $\state_{1}$ might be queried by multiple workers, it will have less than $m$ distinct offspring nodes of $\state_{1}$ being assigned simulation tasks when $\overline{N} (\state', \state_{1}) \!=\! m$. This allows BU-UCT to explore deeper and potentially provide more accurate value estimate of $\state_{1}$ compared to WU-UCT. In an extreme case, if WU-UCT assigns $m$ workers to simulate $\state_{1}$, it can only obtain the simulation return at node $\state_{1}$, which makes its action value $\overline{Q} (\state', \action_{1})$ less accurate.

\section{Additional Details for Experiments}
\label{Details of the Experiments}

This section provides additional experiment results and implementation details of the Atari experiments. First, Appendix~\ref{Speedup Test for BU-UCT} provides results of BU-UCT's speedup test on 15 Atari games. Next, Appendix~\ref{Experiment Details of the Atari Games} describes additional implementation details of the Atari experiments. Finally, Appendix~\ref{Details for the Action Value Gap vs. Performance Experiments} provides details regarding the demonstrative experiment in Figure~\ref{fig: theoretical results demonstration}(b) (i.e., average action value gap vs. episode reward).

\subsection{Speedup Test for BU-UCT}
\label{Speedup Test for BU-UCT}

The speedup of BU-UCT with 16 workers compared to its sequential counterpart (i.e., 1 worker) is shown in Table~\ref{table: speedup}. Across 15 Atari games, BU-UCT achieves on average 14.33 times speedup using 16 workers, which suggests that BU-UCT can better retain the performance of UCT compared to the baselines while achieving desired speedup.

\begin{table*}[t]
\caption{Speedup achieved by BU-UCT using 16 workers on 15 Atari games. Elapsed time represents the average wall clock time to run a single tree search step (i.e., build a search tree with 128 rollouts). Speedup is calculated by dividing the (average) elapsed time using 1 worker by the (average) elapsed time using 16 workers.}
\label{table: speedup}

\centering
{\fontsize{9}{9}\selectfont

\begin{tabular}{cccc}
\toprule
Environment & \makecell[c]{Elapsed time/s \\ (1 worker)} & \makecell[c]{Elapsed time/s \\ (16 workers)} & \makecell[c]{Speedup \\ (16 vs. 1 worker(s))} \\
\midrule

Alien & 54.19 & 3.81 & 14.20 \\
Boxing & 54.71 & 3.55 & 15.37 \\
Breakout & 44.27 & 3.56 & 12.42 \\
Centipede & 50.18 & 3.34 & 15.01 \\
Freeway & 56.98 & 3.75 & 15.16 \\
Gravitar & 39.44 & 2.90 & 13.55 \\
MsPacman & 44.18 & 3.18 & 13.88 \\
NameThisGame & 43.36 & 3.06 & 14.13 \\
RoadRunner & 45.36 & 3.03 & 14.92 \\
Robotank & 54.08 & 3.80 & 14.19 \\
Qbert & 43.25 & 3.06 & 14.10 \\
SpaceInvaders & 45.11 & 3.14 & 14.36 \\
Tennis & 54.35 & 3.72 & 14.58 \\
TimePilot & 41.81 & 2.91 & 14.34 \\
Zaxxon & 46.09 & 3.09 & 14.88 \\

\bottomrule

\end{tabular}
}
\end{table*}

\subsection{Experiment Details of the Atari Games}
\label{Experiment Details of the Atari Games}

\noindent \textbf{MCTS simulation} Each simulation worker is equipped with a pre-trained policy network (that predicts $\pi (\action | \state)$) and a pre-trained value network (that estimate $V (\state)$). Both networks are pre-trained by the Proximal Policy Optimization (PPO) \citep{schulman2017proximal} algorithm. Table~\ref{table: atari PPO} summarizes the performance on the 15 Atari games using only the PPO policy. For a simulation started from state $\state_{0}$, we use the PPO policy network to interact with the environment for 100 steps, which forms a trajectory $\state_{0}, \action_{0}, \reward_{0}, \state_{1}, \dots, \state_{99}, \action_{99}, \reward_{99}, \state_{100}$. If the environment does not terminate, the full simulation return is computed by the intermediate rewards plus the value of $\state_{100}$, i.e., the simulation return is $\rewards_{\mathrm{sim}} \!:=\! \sum_{i = 0}^{99} \gamma^{i} \reward_{i} \!+\! \gamma^{100} V (\state_{100})$. To reduce the variance of Monte Carlo sampling, we average it with the value $V (\state_{0})$. The final simulation return is $\rewards \!:=\! 0.5 \rewards_{\mathrm{sim}} \!+\! 0.5 V (\state_{0})$.

\begin{table*}
\caption{Performance of the PPO policy on 15 Atari games.}
\label{table: atari PPO}

\centering
{\fontsize{9}{9}\selectfont

\begin{tabular}{cc}
\toprule
Environment& PPO policy\\
\midrule

Alien & 850\\
Boxing & 7\\
Breakout & 191\\
Centipede & 1701\\
Freeway & 32\\
Gravitar & 600\\
MsPacman & 1860\\
NameThisGame & 6354\\
RoadRunner & 26600\\
Robotank & 13\\
Qbert & 12725\\
SpaceInvaders & 1015\\
Tennis & -10\\
TimePilot & 4400\\
Zaxxon & 3504\\

\bottomrule

\end{tabular}
}
\end{table*}

\noindent \textbf{Hyperparameters and experiment details for BU-UCT} For all parallel MCTS algorithms, we choose the maximum tree depth/width as 100/20, respectively. The discount factor $\gamma$ is set to 0.99 (note that the reported score is not discounted). Additional details regarding hyperparameters are shown in Appendix~\ref{Hyperparameters of BU-UCT}. Experiments are deployed on machines with 88 CPU cores and 8 NVIDIA\textregistered~P40 GPUs. To minimize speed fluctuation caused by difference in the machines' workload, we ensure that the total number of processes is smaller than the total number of CPU cores.

\noindent \textbf{Hyperparameters and experiment details for baseline algorithms} For WU-UCT \citep{liu2020watch}, we reuse their code provided on GitHub. We also reuse the implementation of the baseline algorithms (i.e., VL-UCT, LeafP, and RootP) provided by \citet{liu2020watch}. All algorithms are implemented in Python, especially utilizing its ``multiprocessing'' module. All hyperparametes of LeafP and RootP have been covered in the previous paragraph. For VL-UCT, we report the better performance among the following two hyperparameter setups: $r_{VL} \!=\! 1.0$ and $r_{VL} \!=\! 5.0$ (see Appendix~\ref{Existing Parallel MCTS Algorithms}).

\subsection{Details for the Action Value Gap vs. Performance Experiments}
\label{Details for the Action Value Gap vs. Performance Experiments}

This section describes the experiment setup of the scatter plot between average action value gap and episode reward (i.e., Figure~\ref{fig: theoretical results demonstration}(b)).

\noindent \textbf{General setup} $\;$ Each node in the scatter plots represents a full run in the corresponding Atari game. That is, at each time step, we use MCTS to plan for the best action to execute, until the game terminates. Note that for each time step we need to construct a search tree and perform rollouts on it.

\noindent \textbf{Average action value gap} $\;$ The reported average action value is averaged across (i) search trees constructed at different time steps of the game, and (ii) (for a single search tree) the action value gap $\overline{G} (\state, \action)$ with respect to different edges $(\state, \action)$. Note that to minimize noise, we use a weighted average over action value gap $\overline{G} (\state, \action)$ for different edges. The weight is the complete visit count of that node (i.e., $N (\state, \action)$).

\noindent \textbf{Episode reward} $\;$ We adopt the most common performance measure used in Atari --- the episode reward. It sums up the reward obtained at all time steps without discount.

\noindent \textbf{Hyperparameters} $\;$ All experiments are run with a set of randomly selected hyperparameters. For all algorithms, we randomly select the number of workers from the range $[4, 32]$. All experiments perform in total 512 rollouts. A random default policy was used to reduce the time consumption. Maximum depth/width of the search tree is 100/20. For hard virtual loss (see Appendix~\ref{Existing Parallel MCTS Algorithms}), $r_{VL}$ is selected from the range $[0, 10]$; for soft virtual loss, $r_{VL}$ is selected from the range $[0, 20]$ and $n_{VL}$ is selected from the range $[1, 5]$.

\begin{figure}[t]
\begin{minipage}[t]{\textwidth}

\begin{algorithm}[H]
\caption{BU-UCT}
\label{alg: BU-UCT}
{\fontsize{9}{9} \selectfont
\begin{algorithmic}

\STATE{\textbf{Input:} environment emulator $\mathcal{E}$, root tree node $s_{root}$, maximum simulation step $T_{max}$, maximum simulation depth $d_{max}$, number of expansion workers $N_{exp}$, and number of simulation workers $N_{sim}$}

\STATE{\textbf{Initialize:} expansion worker pool $\mathcal{W}_{exp}$, simulation worker pool $\mathcal{W}_{sim}$, game-state buffer $\mathcal{B}$, $t \leftarrow 0$, and $t_{complete} \leftarrow 0$}

\WHILE{$t_{complete} < T_{max}$}

\STATE\texttt{\# Selection}

\STATE{Traverse the tree top down from root node $s_{root}$ with the tree policy shown below (i.e., Eq.~(\ref{eq: UCT tree policy})) until (i) its depth greater than $d_{max}$, (ii) it is a leaf node, or (iii) it is a node that has not been fully expanded and \emph{random()} $<$ 0.5. Specifically, when we are at node $\state_{t}$, we select the following action $\action_{t}$ using the tree policy:
    { \setlength{\abovedisplayskip}{-0.4em}
      \setlength{\belowdisplayskip}{-0.0em}
    \begin{align*}
        a_{t} \!=\! \operatornamewithlimits{argmax}\limits_{a \in \mathcal{A}} \! \bigg \{ \overline{Q} (s_{t}, a) \!+\! c \sqrt{\frac{2 \ln \sum_{a'} \! \overline{N} (s_{t},\! a')}{\overline{N} (s_{t}, a)}} \bigg \} \; \mathtt{(tree~policy)},
    \end{align*}}
    
\noindent where $\overline{Q} (\state, \action) \!:=\! Q (\state, \action) \!-\! \infty \!\cdot\! \mathbbm{1} \big [ \, \overline{O} (\state, \action) \!\geq\! m_{\mathrm{max}} \!\cdot\! M \big ]$ (i.e., Eq.~(\ref{eq: BU-UCT modified action value})) and $\overline{N} (\state, \action) \!:=\! N (\state, \action) \!+\! O (\state, \action)$.
}

\vspace{0.5em}

\STATE\texttt{\# Assign expension and simulation tasks}

\IF{expansion is required}

\STATE{$\Bar{s} \leftarrow \text{shallow copy of the current node}$}

\STATE{Assign expansion task $(t, \Bar{s})$ to pool $\mathcal{W}_{exp}$ \textbf{// $t$ is the task index}}

\ELSE

\STATE{assign simulation task $(t, s)$ to pool $\mathcal{W}_{sim}$ if episode not terminated}

\STATE{Call \textbf{incomplete\_update}$(s)$; if episode terminated, call \textbf{complete\_update}$(t, s, 0.0)$}

\ENDIF

\vspace{0.5em}

\STATE\texttt{\# Fetch expansion tasks}

\STATE{\textbf{if} $\mathcal{W}_{exp}$ fully occupied \textbf{then}}

\STATE{\hspace{\algorithmicindent} Wait for a expansion task with return: (task index $\tau$, current state $s$, expended action $a$, reward $r$,} 
\STATE{\hspace{\algorithmicindent} expended state $\state'$ (the next state following $(\state, \action)$), terminal signal $d$, task index $\tau$); expand the tree} 
\STATE{\hspace{\algorithmicindent} according to $\tau$, $s$, $a$, $s'$, $r$, and $d$; assign simulation task $(\tau, s)$ to pool $\mathcal{W}_{sim}$}
\STATE{\hspace{\algorithmicindent} Call \textbf{incomplete\_update}$(t, s)$}

\IF{$\state'$ is the first expended state of $\state$}

\STATE{$\state_{p} \leftarrow \mathcal{P}\mathcal{R} (\state)$; $\action_{p} \leftarrow \mathrm{the~action~that~transits~}\state_{p}\mathrm{~to~}\state$ \textbf{// $\mathcal{PR}(s)$ denotes the parent node of $s$}}

\STATE{$N (\state_{p}, \action_{p}) \leftarrow 1$ \textbf{// Note that all previous simulation returns updated to $Q (\state_{p}, \action_{p})$ are initiated from $\state$, hence according to Section~\ref{Towards Optimal Parallel Tree Search} we treat them as a single simulation return.}}

\ENDIF

\STATE{\textbf{else} \textbf{continue}}

\vspace{0.5em}

\STATE\texttt{\# Fetch simulation tasks and backpropagate}

\STATE{\textbf{if} $\mathcal{W}_{sim}$ fully occupied \textbf{then}}

\STATE{\hspace{\algorithmicindent} Wait for a simulation task with return: (task index $\tau$, node $s$, cumulative reward $\Bar{r}$)}

\STATE{\hspace{\algorithmicindent} Call \textbf{complete\_update}$(\tau, s, \Bar{r})$; $t_{complete} \leftarrow t_{complete} + 1$}

\STATE{\textbf{else} \textbf{continue}}

\vspace{0.5em}

\STATE{$t \leftarrow t + 1$}

\ENDWHILE

\end{algorithmic}
}

\end{algorithm}

\end{minipage}

\begin{minipage}[t]{1.0\textwidth}
\begin{algorithm}[H]
\caption{incomplete\_update}
\label{alg:incomplete_update}
{\fontsize{9}{9} \selectfont
\begin{algorithmic}

\STATE{\textbf{input:} node $s$}

\WHILE{$n \neq \text{null}$}

\STATE{$\action \leftarrow \mathtt{the~previous~action~selected~at~this~node}$}

\STATE{$O (\state, \action) \leftarrow O (\state, \action) + 1$ \textbf{// Update incomplete visit count}}

\STATE{$\overline{O} (\state, \action) \leftarrow \frac{\overline{N} (\state, \action) - 1}{\overline{N} (\state, \action)} \overline{O} (\state, \action) + \frac{1}{\overline{N} (\state, \action)} O (\state, \action)$ \textbf{// Update average incomplete visit count}}

\STATE{$s \leftarrow \mathcal{PR}(s)$ \textbf{// $\mathcal{PR}(s)$ denotes the parent node of $s$}}

\ENDWHILE

\end{algorithmic}
}

\end{algorithm}
\end{minipage}
\hfill
\begin{minipage}[t]{1.0\textwidth}
\begin{algorithm}[H]
\caption{complete\_update}
\label{alg:complete_update}
{\fontsize{9}{9} \selectfont
\begin{algorithmic}

\STATE{\textbf{input:} task index $t$, node $s$, reward $\Bar{r}$}

\WHILE{$n \neq \text{null}$}

\STATE{Retrieve the selected action $\action$ and the corresponding reward $r$ according to task index $t$}

\STATE{$N (\state, \action) \leftarrow N (\state, \action) + 1$; $O (\state, \action) \leftarrow O (\state, \action) - 1$ \textbf{// Update complete and incomplete visit count}}

\STATE{$\Bar{r} \leftarrow r + \gamma \Bar{r}$ \textbf{// Calculate cumulative reward}}

\STATE{$Q (\state, \action) \leftarrow \frac{N (\state, \action) - 1}{N (\state, \action)} Q (\state, \action) + \frac{1}{N (\state, \action)} \Bar{r}$ \textbf{// Update action value}}

\STATE{$s \leftarrow \mathcal{PR}(s)$ \textbf{// $\mathcal{PR}(s)$ denotes the parent node of $s$}}

\ENDWHILE

\end{algorithmic}
}

\end{algorithm}

\end{minipage}
\end{figure}